\documentclass[11pt,letterpaper]{article}
\pdfoutput=1 
\usepackage[height=8in]{geometry}
\usepackage{amsmath, amssymb, amsthm, comment, mathrsfs}
\usepackage{graphicx, color}
\usepackage{array, multirow}
\usepackage[font=small,labelfont=bf]{caption} 
\usepackage{lipsum}
\usepackage{amsfonts}
\usepackage{tabularx}
\usepackage{booktabs}
\usepackage{tcolorbox}
\usepackage{graphicx}
\usepackage{epstopdf}
\usepackage{algorithmic}
\usepackage{mathtools}
\usepackage{tikz}
\usepackage{todonotes}
\usepackage[toc,title,page]{appendix}
\usepackage[activate={true,nocompatibility},final,tracking=true,kerning=true,spacing=true,factor=1100,stretch=10,shrink=10]{microtype}

\ifpdf
\DeclareGraphicsExtensions{.eps,.pdf,.png,.jpg}
\else
\DeclareGraphicsExtensions{.eps}
\fi

\AtBeginDocument{%
\setlength{\oddsidemargin}{\dimexpr(\paperwidth-\textwidth)/2-1in}%
\setlength{\evensidemargin}{\oddsidemargin}%
\setlength{\topmargin}{%
\dimexpr(\paperheight-\textheight)/2-\headheight-\headsep-1in}%
}

\usepackage[shortlabels]{enumitem}

\usepackage[sort]{natbib}

\setcitestyle{numbers,square,comma}
\usepackage{hyperref}
\hypersetup{colorlinks=true,
urlcolor=blue,
linkcolor=blue,
citecolor=blue,
bookmarksdepth=paragraph}
\usepackage[nameinlink]{cleveref}
\crefname{equation}{}{}
\crefname{section}{section}{sections}
\crefname{figure}{figure}{figures}
\crefname{table}{table}{tables}
\crefname{example}{example}{examples}
\crefname{proposition}{proposition}{propositions}
\Crefname{section}{Section}{Sections}
\Crefname{figure}{Figure}{Figures}
\Crefname{table}{Table}{Tables}
\Crefname{definition}{Definition}{Definitions}
\Crefname{theorem}{Theorem}{Theorems}
\Crefname{remark}{Remark}{Remarks}
\Crefname{example}{Example}{Examples}
\Crefname{proposition}{Proposition}{Propositions}
\numberwithin{equation}{section}

\newtheorem{definition}{Definition}[section]
\newtheorem{theorem}[definition]{Theorem}
\newtheorem{lemma}[definition]{Lemma}

\newtheorem{proposition}[definition]{Proposition}

\newtheorem{remark}[definition]{Remark}

\newtheorem{exmp}[definition]{Example}

\newcommand{\R}{\mathbb{R}}
\newcommand{\Realization}{\mathsf{R}}
\newcommand{\risk}{\mathcal{R}}

\newcommand{\N}{\mathbb{N}}

\DeclareMathOperator*{\argmin}{arg\,min}

\title{Stable Learning Using Spiking Neural Networks Equipped With Affine Encoders and Decoders}

\author{%
 A. Martina Neuman%
  \thanks{University of Vienna, Faculty of Mathematics, Kolingasse 14-16,
    1090 Wien,
    e-mail: \texttt{anh.martina.neuman@univie.ac.at}
  }
  \and
 Dominik Dold
  \thanks{University of Vienna, Faculty of Mathematics, Kolingasse 14-16,
    1090 Wien,
    e-mail: \texttt{dominik.dold@univie.ac.at}
  }
  \and
  Philipp Christian Petersen%
  \thanks{
    University of Vienna,
    Faculty of Mathematics and Research Network Data Science @ Uni Vienna, Kolingasse 14-16,
    1090 Wien,
    e-mail: \texttt{philipp.petersen@univie.ac.at}
  }
}

\begin{document}

\maketitle


\begin{abstract}
We study the learning problem associated with spiking neural networks. 
Specifically, we focus on spiking neural networks composed of \textit{simple spiking neurons} having only positive synaptic weights, equipped with an affine encoder and decoder; we refer to these as \textit{affine spiking neural networks}. 
These neural networks are shown to depend continuously on their parameters, which facilitates classical covering number-based generalization statements and supports stable gradient-based training. 
We demonstrate that the positivity of the weights enables a wide range of expressivity results, including rate-optimal approximation of smooth functions and dimension-independent approximation of Barron regular functions.
In particular, we show in theory and simulations that affine spiking neural networks are capable of approximating shallow ReLU neural networks.
Furthermore, we apply these affine spiking neural networks to standard machine learning benchmarks and reach competitive results.
Finally, we observe that from a generalization perspective, contrary to feedforward neural networks or previous results for general spiking neural networks, the depth has little to no adverse effect on the generalization capabilities.
\end{abstract}

\section{Introduction}

Deep learning \cite{lecun2015deep, bengio2017deep} is a technology that has revolutionized many areas of modern life. 
At its core, the term describes the gradient-based training of deep neural networks. 
Since its breakthrough in image classification in 2012 \cite{krizhevsky2012imagenet}, deep learning is essentially the only viable technology for this application. 
Moreover, it is the basis of multiple recent breakthroughs in science \cite{jumper2021highly} and even mathematical research \cite{davies2021advancing}. 
Recently, deep learning has received wide public attention through the advent of generative AI in the form of large language models such as ChatGPT \cite{openai2023gpt}.
However, it is well-documented that deep learning in modern applications often demands significant computational resources, with hardware requirements scaling at an unsustainable rate \cite{thompson2021deep}.
In constrained settings, this limits the practicality of deploying deep learning methods.
In addition, these extensive computations come with an immense environmental cost \cite{luccioni2024light,luccioni2024power}. 
Consequently, to address the growing need for more powerful computational tools, controlling computational costs becomes crucial.
Neuromorphic computing \cite{schuman2017survey} offers one promising solution to this problem. 
This computational paradigm leverages \textit{spiking neural networks} (SNNs) \cite{maass1997networks}, which are more closely aligned with biological neural networks, and hold the potential to be significantly more energy-efficient than traditional deep neural networks \cite{frenkel2023bottom,lunghi2024investigation,yin2021accurate,goltz2021fast}. 
 
A neuron in a spiking neural network is referred to as a \text{spiking neuron}. 
Different from their artificial counterparts commonly used in deep learning, these spiking neurons operate in an inherently temporal manner \cite{gerstner2014neuronal}.
The key state variable of spiking neurons is their membrane potential, which resembles the potential difference across the membrane of a biological neuron.
The outputs of a spiking neuron consist of all-or-nothing events, namely stereotypical electrical pulses called \textit{spikes}. 
The connection between two neurons, known as a synapse, converts an incoming spike into either an \textit{excitatory} or \textit{inhibitory} change in the membrane potential of the receiving neuron.
When excitatory changes in potential are rapid and significant, typically modeled as exceeding a predefined \textit{threshold}, the receiving neuron in turn emits a spike \cite{hodgkin1952quantitative}.
Beyond the basic configuration, a wide range of spiking neuron models have been developed, each designed to meet the demands of specific applications \cite{izhikevich2004model,gerstner2009good}.
For instance, models such as Hodgkin-Huxley \cite{hodgkin1952quantitative} and FitzHugh-Nagumo \cite{fitzhugh1961impulses,nagumo1962active} provide high biological realism but at the cost of significant computational complexity.
On the other hand, in the neuromorphic and neuro-inspired AI communities, where a major focus lies in leveraging spiking neural networks for machine learning applications, threshold-based models such as the \textit{leaky integrate-and-fire} model (LIF) \cite{lapicque1907recherches} and its variations have become the spiking neuron models of choice \cite{frenkel2023bottom,eshraghian2023training}.

The energy efficiency of spiking neurons, and spiking neural networks by extension, is rooted in the characteristic all-or-nothing spiking mechanism, which promotes high sparsity in synaptic interactions.
Information can then be encoded either in precise \textit{spike times} \cite{thorpe1996speed}, spike time sequences \cite{xie2024neuronal,dold2022neuro}, or temporally averaged quantities, such as \textit{spike rates} over sequences of spikes \cite{gerstner2014neuronal}.
Despite their considerable application potential, developing training methods for SNNs that effectively capture temporal sparsity while achieving competitive performance on machine learning tasks remains a challenging and ongoing area of research \cite{zenke2021visualizing,davidson2021comparison,lunghi2024investigation}.
However, recent advancements, such as surrogate gradients \cite{neftci2019surrogate}, have made some progress in narrowing the performance gap with artificial neural networks (ANNs).
A key difficulty encountered in training SNNs is the fact that, for many neuron models such as the current-based LIF model, the output of an SNN does not continuously depend on its parameters.
Thus, resulting discontinuous changes of spike times, such as sudden disappearance (or re-emergence) of spikes, lead to instabilities during training with gradient-based methods \cite{eshraghian2023training}.
A spike disappearance occurs, for example, when increased inhibition drives a neuron's membrane potential below the threshold (Figure~\ref{fig:intro}A).
This is particularly problematic in deep SNNs, where neurons in deeper layers may become inactive, resulting in dead neurons and vanishing gradients during training \cite{rossbroich2022fluctuation,eshraghian2023training}, which is usually countered by adding a term to the loss function that aims at reviving dead neurons when too many die out \cite{mostafa2017supervised,goltz2021fast}.
Moreover, these discontinuities prevent the derivation of standard covering number-based learning bounds, complicating a comparison of SNNs with ANNs.

In this work, we examine the simple \textit{spike-response model} (SRM) \cite{maass1997complexity} with single-spike encoding, which has recently obtained increased interest in theoretical studies \cite{singh2023expressivity,stanojevic2023exact} and deep learning applications \cite{stanojevic2024high}.
In the general family of SRMs, which encompasses the family of LIF neurons, each incoming spike triggers a \textit{response}, and the membrane potential is represented as the sum of these spike-responses \cite{gerstner2014neuronal}. 
In the used simple SRM, these responses are chosen to be linear, where each slope reflects an excitatory or inhibitory effect and is determined by a \textit{synaptic weight}.
We discover that simple SRM SNNs exhibit discontinuities with respect to their neural network parameters, including synaptic weights, with inhibition (modeled by synaptic negative weights) being a primary cause (Figure~\ref{fig:intro}B).
Such discontinuities crucially hinder the derivation of standard covering number-based learning bounds, complicating a theoretical comparison of SNNs and ANNs.
Additionally, they contribute to the challenges associated with neuronal inactivity and training instability, as gradient-based optimization methods cannot accommodate the abrupt changes in the neural network's output.
To address this limitation, we propose a modified type of the simple SRM SNN, which we refer to as \textit{affine SNN}. 
It is characterized by two key properties:

\begin{enumerate}
    \item only excitatory responses, or more specifically, only positive synaptic weights, are permitted;
    \item in line with \cite{singh2023expressivity}, we add an affine encoding and a decoding layers\footnote{In our context, the adjective ``affine'' in the term ``affine SNN'' refers to the affine encoding and decoding layers. Since our model uses only positive synaptic weights, we adopt the term for brevity. However, we may refer to, say, ``affine SNNs with real-valued synaptic weights'' in our experiments when describing the same model extended to allow negative synaptic weights.}.
\end{enumerate}
The restriction to excitatory synapses carries significant analytical and practical consequences. 
It is designed to sustain neuronal spike activity, and thereby guarantee well-defined and tractable spike times. 
\begin{figure} 
    \centering
    \includegraphics[width=.95\linewidth]{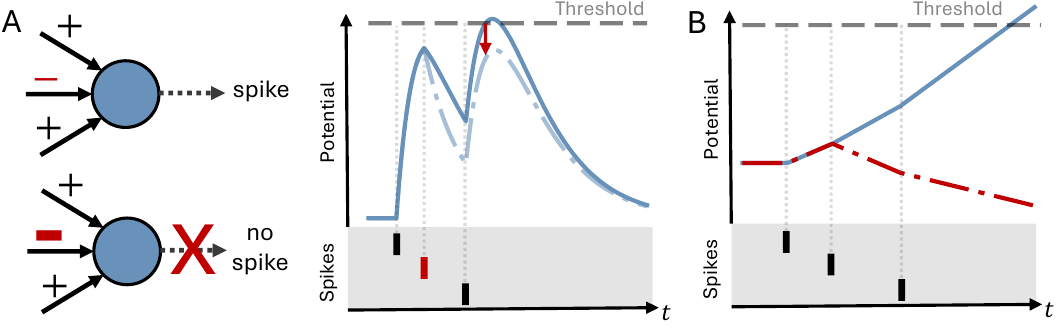}
    \caption{\textbf{A} Example of the response of a current-based LIF neuron to three input spikes. When increasing the inhibitory input (red), the potential gets pulled below the threshold (dash-dotted line) and the neuron stops spiking. \textbf{B} Potential of a simple spike-response neuron given three input spikes. In contrast to the case with only excitatory inputs, i.e., positive weights (blue line), the neuron becomes silent if the slope of the response is negative (red, dash-dotted line).}
    \label{fig:intro}
\end{figure}
An outstanding property of affine SNNs is that their outputs are \emph{Lipschitz continuous} with respect to both neural network inputs and \textit{parameters}, which we show in Theorem \ref{thm:LipThm}. 
This is a property absent in more general SNNs that incorporate inhibitory responses.
We demonstrate that, despite the restriction on their weights, affine SNNs are a novel computational paradigm possessing beneficial properties of feedforward neural networks \cite{cybenko1989approximation, he2020relu, barron}.
We collect the results in Section \ref{sec:expressivity}, where we find the following:
\begin{enumerate}
\item Affine SNNs are \emph{universal approximators}, i.e., every continuous function on a compact domain can be arbitrarily well approximated by an affine SNN (Theorem~\ref{thm:universality}).
\item Affine SNNs can replicate approximation results achieved by \emph{linear finite elements} (Theorem~\ref{thm:reapproximationofFEMSpaces}).
In particular, they approximate, at optimal approximation rates, Sobolev-regular smooth functions (Theorem~\ref{thm:approximationWsinfty}). 
\item Affine SNNs can produce dimension-independent approximation rates for Barron-regular functions (Theorem~\ref{thm:COD}).
\end{enumerate}
Our proposed class of affine SNNs occupies an intermediate position between shallow and deep ReLU networks in the approximation-generalization trade-off; such characterization is also supported by our empirical results.
Notably, we find that, in order to keep the generalization gap under control, we need a number of training samples that depends only linearly on the number of parameters (Theorem \ref{thm:generalizationGapTheorem}). 
In particular, unlike feedforward neural networks or previous VC dimension-based analyses of SNNs \cite{schmitt1999vc, maass1997complexity}, our generalization bounds scale at most \textit{logarithmically with the depth of the network graph}, highlighting a striking advantage in efficiency and scalability (Theorem \ref{thm:generalizationGapTheorem}).
Overall, we observe the following key property of affine SNNs: The capacity cost for learning, i.e., the complexity of the hypothesis set, is bounded with at most logarithmic dependence on the depth of the underlying affine SNNs. 
Hence, affine SNNs can solve problems that shallow feedforward neural networks cannot solve at practically no higher capacity cost. 

We complement these theoretical results with experimental simulations in Section~\ref{sec:simulations}, enabled by the continuous nature of affine SNNs which renders them \textit{particularly well-suited for classical gradient-based training}.
In particular, we demonstrate that affine SNNs feature superior generalization properties in a simple regression task than shallow and deep ReLU neural networks (Figure~\ref{fig:training}A), and that they generalize closer to shallow than deep ReLU neural networks in a classification task (Figure~\ref{fig:training}B). We further show that affine SNNs reach competitive performance levels on standard machine learning benchmarks such as MNIST ($96.75^{+0.02}_{-0.08}\ \%$ median test accuracy, with upper and lower index being the distance to the third and first quartile) and Fashion MNIST ($87.81^{+0.04}_{-0.60}\ \%$ median test accuracy).


\subsection{Related work}

\paragraph{Feedforward neural networks.}

In the last few years, a decent understanding of the learning theory of deep neural networks has been established. For comprehensive overviews, we refer to \cite{berner2021modern, anthony1999neural, Petersen2024Deep}.
Learning theory is typically split into two aspects: first, the expressivity of architecture, i.e., how well a certain type of neural networks can represent a set of functions of interest, and second, generalization bounds, which describe the mismatch between the performance of a trained model on the training set and unseen data points.

The approximation theory of feedforward neural networks is comparatively very well understood.
First of all, universality properties have been shown for various architectures \cite{cybenko1989approximation, hornik1990universal, kidger2020universal}.
Moreover, for specific function classes, approximation rates can be derived.
For example, focusing only on feedforward neural networks with the ReLU activation function, it was established that neural networks could reproduce approximation by linear and higher order finite elements \cite{he2020relu, opschoor2020deep}, achieve optimal approximation of smooth functions \cite{yarotsky2017error, shen2022optimal}, and approximate high-dimensional functions without curse of dimensionality \cite{barron, parhi2022near, caragea2023neural, lerma2024dimension}.

Regarding learning guarantees, classical statistical learning theory facilitates generalization through, for example, VC dimension, covering number or pseudo-dimension bounds \cite{shalev2014understanding, berner2021modern, anthony1999neural}.
Applied specifically to ReLU neural networks, such generalization bounds are, for example, derived in \cite{berner2020analysis, schmidt2020nonparametric}.

It has to be mentioned that in the context of modern machine learning applications using overparameterized architectures, classical statistical learning theory-based arguments are potentially not the best possible, and different tools are required \cite{zhang2021understanding, belkin2019reconciling}. 

\paragraph{Spiking neural networks.}

The research on SNNs is vast and has been developed over multiple decades.
For a more comprehensive overview, we refer to the survey articles \cite{maass1997networks, paugam2006spiking, gruning2014spiking,eshraghian2023training}.
Generally, it has been shown that SNNs can represent specific functions, such as coincidence detectors, with significantly fewer parameters compared to feedforward neural networks \cite{abeles1982role}, making them an attractive alternative in resource-constrained environments.
Moreover, they have been shown to enable inference with very low times-to-solution using the time-to-first-spike paradigm \cite{goltz2021fast,stanojevic2024high} as well as highly energy-efficient deep learning solutions \cite{goltz2021fast,yin2021accurate,lunghi2024investigation}.

In the family of spike response models with temporally encoded inputs, the expressivity of SNNs has been studied extensively, e.g., \cite{maass1996lower,maass1997networks,maass2015spike,singh2023expressivity,stanojevic2023exact}. 
This includes the universal approximation property \cite{mostafa2017supervised,kheradpisheh2020temporal,comsa2020temporal,goltz2021fast,klos2023smooth} as well as other, more quantitative approximation rates \cite{singh2023expressivity}. 
Our results distinguish themselves from previous work due to our requirement of positive weights. Even though at first glance, this offers significantly less flexibility regarding an SNN's parametrization, it opens up a new angle for studying SNNs, with remarkable results. 
Concerning learning rates, it was shown in \cite{maass1997complexity, maass1999complexity, schmitt1999vc} that classical statistical learning theory bounds can be derived.
These are in terms of the VC dimension or \textit{pseudo-dimension}. 
In contrast, our results use covering number estimates and yield stronger generalization guarantees since the upper bounds only depend logarithmically on the depth of the SNNs.

\section{Notions of spiking neural networks}

In this section, we introduce the central concepts of this paper.
Specifically, we will describe the SNN architectures and the concept of \emph{spike times}. 
SNNs are based on the so-called network graphs introduced below.

\begin{definition} A (finite) directed, unweighted graph $G=(V,E)$ satisfying the following properties:
\begin{enumerate}
\item $G$ has no directed cycles,
\item $G$ has no isolated nodes,
\end{enumerate}
is called a directed acyclic graph or a network graph.
We denote the set of all nodes with no incoming edges as $V_{\rm in}$, referred to as the input nodes, and the set of nodes with no outgoing edges as $V_{\rm out}$, referred to as the output nodes. 
We refer to the length of the longest directed path in $G$ as the (graph) depth of $G$.
\end{definition}

\begin{figure}
\begin{minipage}{0.5\textwidth}
    \centering
    \caption{A network graph $G$ with two input nodes in burgundy and one output node in violet}
    \includegraphics[width=0.8\linewidth]{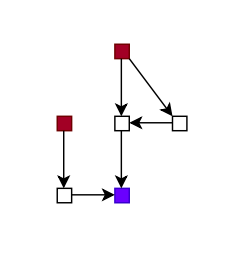}
    \label{fig:networkgraph}
\end{minipage}  \hfill
\begin{minipage}{0.5\textwidth}
    \centering
    \includegraphics[width=0.6\linewidth]{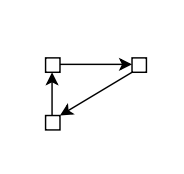}
    \caption{An example of a forbidden directed cycle}
    \label{fig:nonnetworkgraph}
\end{minipage}
\end{figure}

Figures~\ref{fig:networkgraph},~\ref{fig:nonnetworkgraph} provide an example and a non-example of a network graph, respectively.

Based on a network graph, various definitions of SNNs are possible.
In Subsection~\ref{sec:outSetup}, we specify our model of SNNs, and in Subsection~\ref{sec:OtherSetups}, we draw comparisons between our setups and others documented in the literature.
Following this, in Subsection~\ref{sec:affineSNN}, we introduce affine SNNs. 
Finally, in Subsection~\ref{sec:CalculusOfSNNs}, we present a useful operation applicable to affine SNNs, facilitating the construction of more complex neural networks. 

In this paper, we denote the cardinality of a set $S$ by $\#S$.
Drawing inspiration from graph theory, neural networks, and biology, we will also use the terms ``node", ``vertex" and ``neuron" interchangeably. 

\subsection{Spiking neural network model} \label{sec:outSetup}

We now present the definition of a general SNN as an architecture, followed by a description of its dynamics.
Afterward, we introduce a special type of \emph{positive SNNs}.

\begin{definition} \label{def:SNN}
Let $G=(V,E)$ be a network graph with a subset $V_{\rm in}\subset V$ of input neurons, a subset $V_{\rm out}\subset V$ of output neurons, and a set $E\subset V\times V$ of synapses.
Each synapse $(u,v)\in E$ is a directed edge, associated with the following attributes
\begin{enumerate}
\item a response function $\varepsilon_{(u,v)}\colon \R \to\R$,
\item a synaptic delay $d_{(u,v)}\geq 0$,
\item a synaptic weight $\mathsf{w}_{(u,v)}\geq 0$.
\end{enumerate}
Moreover, for every $v\in V\setminus V_{\rm in}$, there exists $(u,v)\in E$ such that $\mathsf{w}_{(u,v)}>0$. 

Lastly, let $\mathsf{W} \coloneqq (\mathsf{w}_{(u,v)})_{(u,v)\in E}$, $D \coloneqq (d_{(u,v)})_{(u,v)\in E}$, and $\mathcal{E} \coloneqq (\varepsilon_{(u,v)})_{(u,v)\in E}$ be the tuple of synaptic weights, synaptic delays, and response functions, respectively.
Then a spiking neural network (SNN) is a tuple $\Phi = (G,\mathsf{W}, D, \mathcal{E})$.
\end{definition}

In the sequel, we will focus on SNNs for which the response function is the same for all edges. 
Concretely, inspired by \cite{singh2023expressivity}, we adopt a unified response function modeled after the ReLU activation function.

\begin{definition} \label{def:spiking}
Let $G$ be a network graph, and let $(u,v)$ be one of its synapses.
We define the response function $\varepsilon_{(u,v)}:\R\to\R$ associated with $(u,v)$ as follows
\begin{align} \label{eqdef:response}
	\varepsilon_{(u,v)}(t) \coloneqq \varrho(t)
\end{align}
where $\varrho(t) = \max \{t, 0\}$ denotes the ReLU activation function.
\end{definition}

Information is transmitted through an SNN via spike times, triggered when the membrane potential, or simply the \textit{potential}, reaches a critical threshold.
We state the mathematical model below. 

\begin{definition} \label{def:firing_time} 
Let $\Phi=(G,\mathsf{W}, D, \mathcal{E})$ be an SNN, and let $G=(V,E)$ be a network graph. 
Let $\varepsilon_{(u,v)}\in\mathcal{E}$ be defined as in \eqref{eqdef:response}.
Let $t_u \in \R$ for $u \in  V_{\rm in}$. 
Then, for $v\in V\setminus V_{\rm in}$, we define the potential at $v$ as $P_v\colon \R\to\R$, where for $t \in \R$
\begin{align} \label{eqdef:accumulation}
	P_v(t) \coloneqq \sum_{(u,v)\in E} \mathsf{w}_{(u,v)}\varepsilon_{(u,v)}(t - t_u - d_{(u,v)}) = \sum_{(u,v)\in E} \mathsf{w}_{(u,v)}\varrho(t - t_u - d_{(u,v)}).
\end{align}
Here in \eqref{eqdef:accumulation}, $t_u = \min\{t \in \R\colon P_u(t) = 1\}$ is the spike time at $u$ if $u\in V \setminus V_{\rm in}$. 
Subsequently, the spike time $t_v$ at $v$ is given by $t_v = \min\{t \in \R\colon P_v(t) = 1\}$.
\end{definition}

A few remarks are in order. 
First, the provided definition may appear circular, as the spike time at a noninput neuron is determined by its potential, which, in turn, hinges on the spike times of other presynaptic neurons.
Second, it does not inherently guarantee that $\min\{t \in \R\colon P_u(t) = 1\}$ is nonempty.

The following lemma demonstrates the well-definedness of $P_v$ and $t_v$ for all $v \in V\setminus V_{\rm in}$, when the responses $\varepsilon_{(u,v)}$ are of the form \eqref{eqdef:response}.
A proof is given in Appendix~\ref{appx:well-definedFiringTimes}.

\begin{lemma} \label{lem:well-definedFiringTimes}
Let $\Phi=(G,\mathsf{W}, D, \mathcal{E})$ be an SNN where $G=(V,E)$ is a network graph. 
Let $\varepsilon_{(u,v)}\in\mathcal{E}$ be defined as in \eqref{eqdef:response}. 
Let $t_u \in \R$ for $u \in  V_{\rm in}$. 
Then $P_v$ and $t_v$ are well-defined for all $v \in V\setminus V_{\rm in}$.
\end{lemma}

We conclude this subsection by formalizing the selected SNNs for this work, referred to as positive SNNs.

\begin{definition} \label{def:posSNN}
Let $\Phi=(G,\mathsf{W}, D, \mathcal{E})$ be an SNN with $G=(V,E)$ being a network graph.
Then $\Phi$ is a positive SNN if for all $(u,v)\in E$, $\mathsf{w}_{(u,v)}>0$, and $\varepsilon_{(u,v)}$ is given by \eqref{eqdef:response}. 
Moreover, if $\Phi$ is a positive SNN, we streamline its tuple notation as $\Phi = (G,\mathsf{W},D)$.
\end{definition}

Since Lemma~\ref{lem:well-definedFiringTimes} asserts in particular that $t_v$ exists for $v\in V_{\rm out}$ once $t_u$ is assigned for all $u\in V_{\rm in}$ in a positive SNN $\Phi$, we can define a map taking all input spike times to the output spike times. This function is called the \emph{realization of $\Phi$}.

\begin{definition} \label{def:realSNN} 
Let $\Phi=(G,\mathsf{W},D)$ be a positive SNN. 
Let ${\rm d}_{\rm in}$, ${\rm d}_{\rm out}$ denote the cardinality of $V_{\rm in}$, $V_{\rm out}$, respectively. 
Then the realization of $\Phi$, $\Realization(\Phi):\R^{{\rm d}_{\rm in}}\to \R^{{\rm d}_{\rm out}}$, is a function whose inputs are $(t_v)_{v\in V_{\rm in}} \in \R^{{\rm d}_{\rm in}}$ and whose outputs are $(t_v)_{v\in V_{\rm out}}\in\R^{{\rm d}_{\rm out}}$, where $t_v$ denotes the spike time at neuron $v$. 
\end{definition}

It is to be understood from the definition above and throughout this paper that we assume a consistent enumeration of the input and output neurons. 

\subsection{Model discussion}\label{sec:OtherSetups}


In Definition~\ref{def:posSNN} of positive SNNs $\Phi = (G,\mathsf{W},D)$, the synaptic weights are taken to be positive, and the response functions assume the form \eqref{eqdef:response}.
This leads to an exclusively monotone increase of the potential that always crosses a positive threshold from below at a unique, and analytically calculable, time. 
This property is central to our analysis, which allows us to prove Lipschitz continuity (Section~\ref{sec:LipschitzCont}) with respect to the neural network parameters and derive generalization bounds based on covering numbers (Section~\ref{sec:expressivity}).
In line with the points raised in the introduction, we offer two examples in Appendix~\ref{appx:discontinuity} demonstrating that the inclusion of negative synaptic weights will lead to the lack of continuity with respect to both network parameters as well as input spike times.
In addition, the choice of a linear response function aligns with a wider class of spike-response models \cite{maas1997noisy} that feature linear synaptic response between neurons $u$ and $v$ given by
\begin{align*} 
    \tilde{\varepsilon}_{(u,v)}(t) 
    \coloneqq 
	\begin{cases}
	    0 
	    &\text{ if } t\not\in [0, \delta]\\
	    \varrho(t) 
	    &\text{ if } t\in [0, \delta]
	\end{cases},   
\end{align*}
for $t \in \R$ and $\delta\in (0,\infty]$. 
On the one hand, when $\delta<\infty$, this reduces to a generalized LIF model \cite{gerstner1995time}. 
On the other hand, setting $\delta=\infty$, we recover \eqref{eqdef:response}, which has also been used in \cite{singh2023expressivity,stanojevic2023exact,stanojevic2024high}. 
Due to the linear rise, similar to other types of integrate-and-fire neurons \cite{mostafa2017supervised,stanojevic2023exact,stanojevic2024high}, information about previous input spikes is never forgotten until the neuron spikes itself.
Furthermore, information is encoded purely in the starting time of the linear rise.



\subsection{Spiking neural networks with affine encoders and decoders} \label{sec:affineSNN}

To contextualize and motivate the introduction of affine SNNs, we inspect the implications of Definitions~\ref{def:spiking} and~\ref{def:realSNN}.
Let $\Phi = (G,\mathsf{W},D)$ be a positive SNN associated with the network graph $G=(V,E)$.  
Consider two tuples of input spike times $(t_u)_{u\in V_{\rm in}}$, $(\tilde{t}_u)_{u\in V_{\rm in}}$, such that $\tilde{t}_u\geq t_u$, for all $u\in V_{\rm in}$.
For $v\in V_{\rm out}$, let $t_v$, $\tilde{t}_v$ denote the respective corresponding spike times at $v$.
Then it follows directly from \eqref{eqdef:accumulation} that $t_v\geq \min\{t_u\colon u\in V_{\rm in}\}$ and that $\tilde{t}_v\geq t_v$. 
As a consequence, Definition~\ref{def:realSNN} implies monotonicity of the function $\Realization(\Phi)$.
Naturally, this constitutes a strong limitation on the functions expressible by positive SNNs.
As a remedy, we use affine encoders and decoders to amend the neural network construction.
Note that the use of general encoders and decoders for SNNs has already been introduced in \cite{singh2023expressivity}.

\begin{definition} \label{def:affineSNN}
Let ${\rm d}_{\rm in}, {\rm d}_{\rm out}, {\rm d}_0, {\rm d}_1 \in \N$. 
A spiking neural network equipped with an affine encoder and decoder, or an affine spiking neural network (affine SNN), is a triple $\Psi = (A_{\rm in}, \Phi, A_{\rm out})$. 
Here, $\Phi=(G,\mathsf{W}, D)$ is a positive SNN, with $\#V_{\rm in} = {\rm d}_{\rm in}$, $\#V_{\rm out} = {\rm d}_{\rm out}$, and $A_{\rm in}\colon \R^{{\rm d}_0} \to \R^{{\rm d}_{\rm in}}$, $A_{\rm out}\colon \R^{{\rm d}_{\rm out}} \to \R^{{\rm d}_1}$ are two affine maps, called an encoder and a decoder, respectively, such that for $x \in \R^{{\rm d}_0}$ and $z \in \R^{{\rm d}_{\rm out}}$
\begin{align*}
	A_{\rm in}(x) = W_{\rm in}x + b_{\rm in} \quad\text{ and }\quad A_{\rm out}(z) = W_{\rm out}z + b_{\rm out},
\end{align*}
where $W_{\rm in}\in \R^{{\rm d}_{\rm in}\times {\rm d}_0}$, $W_{\rm out}\in \R^{{\rm d}_1\times {\rm d}_{\rm out}}$, $b_{\rm in}\in \R^{{\rm d}_{\rm in}}$, $b_{\rm out}\in \R^{{\rm d}_1}$. 

The realization of an affine SNN $\Psi = (A_{\rm in}, \Phi, A_{\rm out})$ is given by $\Realization(\Psi)\colon \R^{{\rm d}_0} \to \R^{{\rm d}_1}$, where $\Realization(\Psi) = A_{\rm out}\circ \Realization(\Phi) \circ A_{\rm in}$.
\end{definition}

Next, we address the quantification of the size of these neural networks. 
With our adoption of directed acyclic graphs as network structures, there is no \textit{canonical} concept of layers\footnote{However, see a graph layering algorithm presented in the proof of Lemma~\ref{lem:graphsplit} in Appendix~\ref{appx:LipPropPhi}.}. 
Consequently, we evaluate the neural network's size by examining attributes such as the number of synaptic weights and delays alongside the conventional sizing metrics associated with the affine encoder and decoder maps.

\begin{definition} \label{def:Size} 
Let $\Psi = (A_{\rm in}, \Phi, A_{\rm out})$ be an affine SNN.
The size of $\Psi$, denoted ${\rm Size}(\Psi)$, is defined by the total number of nonzero scalar entries in the tuple $(\mathsf{W}, D, W_{\rm in}, W_{\rm out}, b_{\rm in}, b_{\rm out})$, i.e., 
\begin{align*}
	{\rm Size}(\Psi) \coloneqq \big\| (\mathsf{W}, D, W_{\rm in}, W_{\rm out}, b_{\rm in}, b_{\rm out}) \big\|_{\ell^0}.
\end{align*}
\end{definition}

In our study, it will sometimes be important to guarantee that an affine SNN does not have arbitrarily large outputs.
This requirement is standard in the analysis of learning properties of feedforward neural networks \cite[Setting 2.5]{berner2020analysis}, \cite[Equation (4)]{schmidt2020nonparametric}. 
To enforce this, we introduce a \textit{clipped realization} of an affine SNN below.

\begin{definition} \label{def:clipping}
Let $\Psi$ be an affine SNN. 
Let $I\subset\mathbb{R}$ be a compact interval.
The $I$-clipped realization of $\Psi$ is given by $\Realization_{I}(\Psi)\colon \R^{{\rm d}_0} \to I^{{\rm d}_1}$, where $\Realization_{I}(\Psi) = \mathrm{clip}_{I} \circ \Realization(\Psi)$. 
Here, for $x\in \R^{{\rm d}_1}$ and $i = 1, \dots, {\rm d}_1$
\begin{align*}
	(\mathrm{clip}_I(x))_i = \argmin\{ |x_i - z| \colon z \in I\}
\end{align*}
where $x_i$ denotes the $i$-th coordinate of $x$.
\end{definition}

For ease of reference, we present Table~\ref{tab:symbols}, which summarizes the symbols associated with SNNs and affine SNNs that will be consistently used throughout the remainder of this paper. 

\begin{table}[htb]
\centering
    \begin{tabular}{@{}llll@{}}
    \toprule
        Symbols & & & Default meaning \\
		  \midrule
		  $\Psi$ & & & Affine SNN \\
		  $\Phi$ & & & SNN\\
		  $G$ & & & Network graph\\
        $\mathsf{W}$ & & & Synaptic weight tuple\\
        $D$ & & & Synaptic delay tuple\\
        $V_{\rm in}$ & & & Input nodes of $G$ \\
        $V_{\rm out}$ & & & Output nodes of $G$ \\
        ${\rm d}_{\rm in}$ & & & Cardinality of $V_{\rm in}$ and input dimension of $\Realization(\Phi)$\\
        ${\rm d}_{\rm out}$ & & & Cardinality of $V_{\rm out}$ and output dimension of $\Realization(\Phi)$\\
        $A_{\rm in}$ & & & Affine encoder of $\Psi$\\
        $A_{\rm out}$ & & & Affine decoder of $\Psi$\\
        $W_{\rm in}$ & & & Matrix associated with $A_{\rm in}$\\
        $W_{\rm out}$ & & & Matrix associated with $A_{\rm out}$\\
        $b_{\rm in}$ & & & Shift associated with $A_{\rm in}$\\
        $b_{\rm out}$ & & & Shift associated with $A_{\rm out}$\\
        ${\rm d}_0$ & & & Input dimension of $\Realization(\Psi)$\\
        ${\rm d}_1$ & & & Output dimension of $\Realization(\Psi)$\\
        ${\rm Size}(\Psi)$ & & & Size of $\Psi$\\
    \bottomrule
    \end{tabular}
    \caption{Commonly used symbols and their meanings} 
\label{tab:symbols}
\end{table}

\subsection{Addition of affine spiking neural networks}\label{sec:CalculusOfSNNs}

We introduce 
addition for affine SNNs, 
a commutative operation 
on pairs of SNNs with matching 
input and output dimensions. 
In later sections, we will employ affine SNN addition to reproduce approximation results based on the superposition of simple functions.

Let $\Psi =(A_{\rm in}, \Phi, A_{\rm out})$ be an affine SNN, where $\Phi=(G,\mathsf{W}, D)$ is a positive SNN.
In what follows, we refer to $A_{\rm in}(\Psi)$ and $A_{\rm out}(\Psi)$ as the encoder and decoder of $\Psi$, respectively. 
Additionally, we use $G(\Phi)$, $\mathsf{W}(\Phi)$, and $D(\Phi)$ to denote the network graph, the synaptic weight and synaptic delay matrices associated with $\Phi$, respectively. 

\begin{definition} \label{def:Addition} 
Let 
\begin{align*}
	\Psi=(A_{\rm in}, \Phi, A_{\rm out}) \quad\text{ and }\quad \Psi'=(A'_{\rm in}, \Phi', A'_{\rm out}),
\end{align*}
be two affine SNNs, associated with positive SNNs, 
\begin{align*}
	\Phi=(G,\mathsf{W}, D) \quad\text{ and }\quad \Phi'=(G',\mathsf{W}', D'),
\end{align*}
respectively.
Let for $x \in \R^{{\rm d}_{0}}$, $y \in \R^{{\rm d}_{\rm out}}$, and $z \in \R^{{\rm d}_{\rm out}'}$, 
\begin{alignat*}{2}
	A_{\rm in}(x) &= W_{\rm in}x + b_{\rm in} \quad\text{ and }\quad A_{\rm out}(y) &&= W_{\rm out}y + b_{\rm out},\\
	A'_{\rm in}(x) &= W'_{\rm in}x + b'_{\rm in} \quad\text{ and }\quad A'_{\rm out}(z) &&= W'_{\rm out}z + b'_{\rm out},  
\end{alignat*}
where $W_{\rm in}\in \R^{{\rm d}_{\rm in}\times {\rm d}_0}$, $W_{\rm out}\in \R^{{\rm d}_1\times {\rm d}_{\rm out}}$, $b_{\rm in}\in \R^{{\rm d}_{\rm in}}$, $b_{\rm out}\in \R^{{\rm d}_1}$, and $W'_{\rm in}\in \R^{{\rm d}'_{\rm in}\times {\rm d}_0}$, $W'_{\rm out}\in \R^{{\rm d}_1\times {\rm d}'_{\rm out}}$, $b'_{\rm in}\in \R^{{\rm d}'_{\rm in}}$, $b'_{\rm out}\in \R^{{\rm d}_1}$. 
Then the addition of $\Psi$, $\Psi'$, denoted $\Psi\oplus \Psi'$, is an affine SNN, associated with a positive SNN $\Phi\oplus \Phi'$ such that
\begin{align*}
	G\big(\Phi\oplus \Phi'\big) &\coloneqq G\cup G' \\
	\mathsf{W}\big(\Phi\oplus \Phi'\big) &\coloneqq \mathsf{W}\cup \mathsf{W}'\\
	D\big(\Phi\oplus \Phi'\big) &\coloneqq D\cup D',
\end{align*}
where input (output) nodes of $G$ are listed before input (output) nodes of $G'$. 
Moreover, 
\begin{align*}
	A_{\rm in}\big(\Psi\oplus \Psi'\big) &\colon \R^{{\rm d}_0} \to \R^{{\rm d}_{\rm in} + {\rm d}'_{\rm in}}\\
	A_{\rm out}\big(\Psi\oplus \Psi'\big) &\colon \R^{{\rm d}_{\rm out} + {\rm d}'_{\rm out}} \to \R^{{\rm d}_1},
\end{align*}
are affine maps, such that, for $x\in\R^{{\rm d}_0 }$,

\begin{align*}
	A_{\rm in}\big(\Psi\oplus \Psi'\big)(x)
	\coloneqq 
	\begin{pmatrix} W_{\rm in} \\ W'_{\rm in}\end{pmatrix}x + \begin{pmatrix} b_{\rm in} \\ b'_{\rm in}\end{pmatrix},
\end{align*}
and for $x\in\R^{{\rm d}_{\rm out} + {\rm d}'_{\rm out}}$,
\begin{align*}
	A_{\rm out}\big(\Psi\oplus \Psi'\big)(x)
	\coloneqq 
	\begin{pmatrix} W_{\rm out} & W'_{\rm out}\end{pmatrix}x + (b_{\rm out} + b'_{\rm out}).
\end{align*}
\end{definition}

The following result on the size of the addition of affine SNNs 
follows immediately from the construction.

\begin{lemma} \label{lem:addition}
Let $\Psi, \Psi'$ be two affine SNNs. 
Then
\begin{align*}
	{\rm Size}\big(\Psi \oplus \Psi'\big) = {\rm Size}(\Psi) + {\rm Size}(\Psi').
\end{align*}
\end{lemma}

\section{Lipschitz continuity of affine spiking neural networks} \label{sec:LipschitzCont}

Let $\Psi=(A_{\rm in}, \Phi, A_{\rm out})$ be an affine SNN, where $\Phi = (G,\mathsf{W},D)$ is a positive SNN. 
In this section, we delineate two types of continuity the realization $\Realization(\Psi)\colon \R^{{\rm d}_0}\to \R^{{\rm d}_1}$ of $\Psi$ exhibits: 
\begin{enumerate}
\item continuity concerning the \textit{neural network input} $x\in\R^{{\rm d}_0}$,
\item continuity concerning the \textit{neural network parameters} $(\mathsf{W},D, W_{\rm in}, b_{\rm in}, W_{\rm out}, b_{\rm out})$.
\end{enumerate}
We explore these types of continuity in the listed order.
In addressing the first type of continuity with respect to the neural network input, we offer the following theorem, the proof of which is given in Appendix~\ref{sec:LipaffineSNNpf}.

\begin{theorem} \label{thm:LipaffineSNN} 
Let $\Psi$ be an affine SNN.
Then, for $x,\tilde{x} \in \R^{{\rm d}_0}$, 
\begin{align*}
	\|\Realization(\Psi)(x)-\Realization(\Psi)(\tilde{x})\|_{\ell^{\infty}} \leq ({\rm d}_0{\rm d}_{\rm out})^{\frac{1}{2}}\|W_{\rm in}\|_F\|W_{\rm out}\|_F\|x-\tilde{x}\|_{\ell^{\infty}},
\end{align*}
where $\|\cdot\|_F$ denotes the Frobenius norm of a matrix. 
\end{theorem}

To assess the continuity of $\Psi$ in terms of the neural network parameters, we begin with the continuity of $\Phi$ concerning $(\mathsf{W},D)$. 
In this context, we fix a network graph $G$ and examine the discrepancies between two positive SNNs constructed on $G$, $\Phi=(G,\mathsf{W},D)$ and $\widetilde{\Phi}=(G,\widetilde{\mathsf{W}},\widetilde{D})$.
The following proposition serves as a foundation for the ensuing discussions.
A proof is provided in Appendix~\ref{appx:LipPropPhi}.

\begin{proposition} \label{prop:LipPropPhi} 
Let $\Phi=(G,\mathsf{W},D)$, $\widetilde{\Phi}=(G,\widetilde{\mathsf{W}},\widetilde{D})$ be two positive SNNs.
Suppose there exists $\mathsf{b}>0$ such that for every $\mathsf{w}_{(u,v)}\in\mathsf{W}$, $\tilde{\mathsf{w}}_{(u,v)}\in\widetilde{\mathsf{W}}$, 
\begin{align} \label{eq:strictlypositive}
	\min\big\{\mathsf{w}_{(u,v)}, \tilde{\mathsf{w}}_{(u,v)}\big\} \geq\mathsf{b}.
\end{align}
Let $\mathsf{L}$ be the graph depth of $G$.
Then for every $t\in \R^{{\rm d}_{\rm in}}$, 
\begin{align} \label{eq:LipPropconc}
    \|\Realization(\Phi)(t) - \Realization(\widetilde{\Phi})(t)\|_{\ell^{\infty}} 
    \leq \mathsf{L} \cdot \Big(1+\frac{1}{\mathsf{b}^2}\Big)\|\Phi - \widetilde{\Phi} \|_{\ell^{\infty}},
\end{align}
where
\begin{align} \label{eqdef:PhiLinftydist}
	\|\Phi - \widetilde{\Phi}\|_{\ell^{\infty}} 
	\coloneqq \max\big\{\|\mathsf{W} - \widetilde{\mathsf{W}}\|_{\ell^{\infty}}, \|D - \widetilde{D}\|_{\ell^{\infty}}\big\}.
\end{align}
\end{proposition}

\begin{remark}\label{rem:TightnessOfLipBound}
The constant of the Lipschitz estimate in \eqref{eq:LipPropconc} is almost tight in general.
To see this, consider the graph $G$ with $N$ nodes $(v_i)_{i=1}^N$ and edges only between $v_i$ and $v_{i+1}$ for $i = 1, \dots, N-1$. 
Let $\Phi$, $\widetilde{\Phi}$ be two positive SNNs with respective weights $\mathsf{w}_{(v_i,v_{i+1})} = 2\mathsf{b}$, $\tilde{\mathsf{w}}_{(v_i,v_{i+1})} = \mathsf{b}$ and delays $d_{(v_i,v_{i+1})} = 0$, $\tilde{d}_{(v_i,v_{i+1})} = \mathsf{b}$ for $i = 1, \dots, N-1$.
Then, it is immediate for every $t\in\R$ that
\[
    \Realization(\Phi)(t) = t + (N-1) \frac{1}{2\mathsf{b}} \quad\text{ and }\quad \Realization(\widetilde{\Phi})(t) = t + (N-1) \left(\frac{1}{\mathsf{b}} + \mathsf{b}\right).
\]
Subsequently, 
\[
	|\Realization(\Phi)(t) - \Realization(\widetilde{\Phi})(t)| = (N-1)
	\left(\frac{1}{2\mathsf{b}} +  \mathsf{b} \right) = (N-1)\left(\frac{1}{2\mathsf{b}^2} +  1\right) \|\Phi - \widetilde{\Phi} \|_{\ell^{\infty}}.
\]
\end{remark}

\begin{remark}
The previous Remark \ref{rem:TightnessOfLipBound} demonstrates the importance of strictly positive synaptic weights indicated in \eqref{eq:strictlypositive}.
Indeed, if this condition is not met, it is generally not possible to constrain the difference in two neural network outputs by a multiple of the difference in their neural network parameters. This can be observed by choosing arbitrarily small $\mathsf{b}$ in Remark \ref{rem:TightnessOfLipBound}.
\end{remark}

\begin{remark}
The global Lipschitz continuity of positive SNNs should be contrasted with the local Lipschitz continuity of feedforward neural networks \cite[
Proposition 4]{petersen2021topological}.
For feedforward neural networks, the difference in the realizations of the neural networks can only be controlled by the $\mathsf{L}$-th power of the $\ell^\infty$ difference of the neural network weights, where $\mathsf{L}$ is the graph depth. 
\end{remark}

Continuing, we present a theorem that expands upon Proposition~\ref{prop:LipPropPhi}, clarifying the Lipschitz continuity of an affine SNN in relation to its neural network parameters.
A proof is included in Appendix~\ref{appx:LipThms}.

\begin{theorem} \label{thm:LipThm}
Let $\Psi = (A_{\rm in},\Phi,A_{\rm out})$, $\widetilde{\Psi} = (\tilde{A}_{\rm in},\widetilde{\Phi},\tilde{A}_{\rm out})$ be two affine SNNs, where $\Phi=(G,\mathsf{W},D)$, $\widetilde{\Phi}=(G,\widetilde{\mathsf{W}},\widetilde{D})$ are two positive SNNs.
Suppose there exist $\mathsf{b}, \mathsf{B}\in (0,\infty)$ such that for every $\mathsf{w}_{(u,v)}\in\mathsf{W}$, $\tilde{\mathsf{w}}_{(u,v)}\in\widetilde{\mathsf{W}}$, 
\begin{align*} 
	\min\big\{\mathsf{w}_{(u,v)}, \tilde{\mathsf{w}}_{(u,v)}\big\}\geq\mathsf{b},
\end{align*}
and for every $d_{(u,v)}\in D$, $\tilde{d}_{(u,v)}\in\widetilde{D}$
\begin{align*}
	\max\big\{d_{(u,v)}, \tilde{d}_{(u,v)}\big\} \leq \mathsf{B}.
\end{align*}
Let 
\begin{align*}
	W^{\star}_{\rm out} &\coloneqq \max\big\{\|W_{\rm out}\|_F, \|\widetilde{W}_{\rm out}\|_F\big\}\\
	W^{\star}_{\rm in} &\coloneqq \max\big\{\|W_{\rm in}\|_F, \|\widetilde{W}_{\rm in}\|_F\big\}\\
	b^{\star}_{\rm in} &\coloneqq \max\big\{\|b_{\rm in}\|_{\ell^{\infty}}, \|\tilde{b}_{\rm in}\|_{\ell^{\infty}}\big\}.
\end{align*}
Then for every $x\in\R^{{\rm d}_0}$, 
\begin{align}\label{eq:LipThmconc}
	\|\Realization(\Psi)(x) - \Realization(\widetilde{\Psi})(x)\|_{\ell^{\infty}} \leq B_1 + B_2
\end{align}
where, with $\mathsf{L}$ denoting the graph depth of $G$,
\begin{align*}
	B_1 \coloneqq {\rm d}_{\rm out}^{\frac12} W^{\star}_{\rm out} \cdot \bigg( {\rm d}_0^{\frac12} \|W_{\rm in} - \widetilde{W}_{\rm in}\|_F\|x\|_{\ell^{\infty}}  + \mathsf{L} \cdot \Big(1+\frac{1}{\mathsf{b}^2}\Big) \|\Phi - \widetilde{\Phi}\|_{\ell^{\infty}} + \|b_{\rm in} - \tilde{b}_{\rm in}\|_{\ell^{\infty}}\bigg),
\end{align*}
and
\begin{align*}
	B_2 \coloneqq {\rm d}_{\rm out}^{\frac12}\|W_{\rm out} - \widetilde{W}_{\rm out}\|_F\Big({\rm d}_0^{\frac12} W^{\star}_{\rm in}\|x\|_{\ell^{\infty}} + b^{\star}_{\rm in} + \mathsf{L} \cdot \Big(\frac{1}{\mathsf{b}} + \mathsf{B}\Big) \Big) + \|b_{\rm out} - \tilde{b}_{\rm out}\|_{\ell^{\infty}}.
\end{align*}
\end{theorem}

\section{Generalization bounds for affine spiking neural networks} \label{sec:generalization}

In this section, we derive generalization bounds for the problem of learning functions from finitely many samples using affine SNNs.
We start by defining a learning problem.

Let $\Omega$ be a compact domain in a Euclidean space. 
We assume that there is an \textit{unknown} probability distribution $\mathcal{D}$ on $\Omega \times [0,1]$. 
In a learning problem, our objective is to select a member from an appropriate hypothesis set $\mathcal{H}$ of functions mapping $\Omega$ to $[0,1]$ that fits $\mathcal{D}$ best.
Concretely, we want to find $g\in\mathcal{H}$ that minimizes the \emph{risk} $\risk$ defined as
\begin{align} \label{eq:squarerisk}
	\risk(g) \coloneqq \mathbb{E}_{(x,y)\sim \mathcal{D}} \, |g(x) - y|^2,
\end{align}
over $\mathcal{H}$. 
Since we do not know $\mathcal{D}$, this optimization problem cannot be solved directly. 
Instead, we assume that we are given a sample of size $m \in \N$ of observations drawn i.i.d. from $\mathcal{D}$, i.e., $S = (x_i,y_i)_{i=1}^m \sim \mathcal{D}^m$. 
Based on this sample, we define the \emph{empirical risk} $\widehat{\mathcal{R}}_S$ to be
\begin{align} \label{eq:squaremprisk}
	\widehat{\risk}_S(g) \coloneqq \frac{1}{m}\sum_{i=1}^m  |g(x_i) - y_i|^2.
\end{align}
A function $g_S \in  \argmin_{h \in \mathcal{H}} \widehat{\risk}_S(g)$ is called an \emph{empirical risk minimizer}. 
Such a function serves as a potential approximate minimizer for the optimization problem of the risk. Hence, $g_S$ approximates a solution to the learning problem if it can be shown that the risk and the empirical risk do not differ too much.
In this section, we focus on bounding the risk \eqref{eq:squarerisk} by the empirical risk \eqref{eq:squaremprisk} for a hypothesis class $\mathcal{H}$ consisting of clipped realizations of affine SNNs, defined in \eqref{eq:hypclass} below, up to a small additive term.
The main result, Theorem~\ref{thm:generalizationGapTheorem}, is provided at the end.


For a parameterized hypothesis class, a well-known strategy in statistical learning theory involves leveraging the Lipschitz property of the parameterization to assess the so-called covering numbers, which are then used to upper bound the \textit{generalization error} 
\begin{align*}
	\sup_{g\in\mathcal{H}} |\risk(g) - \widehat{\risk}_S(g)|.
\end{align*}
This approach has been employed in the context of feedforward neural networks \cite{berner2020analysis, grohs2023proof, schmidt2020nonparametric}. 
For the sake of concreteness, we recall the covering number as follows.


\begin{definition}\label{def:coveringNumber}
Let $\Omega$ be a relatively compact subset of a metric space $(\mathcal{X}, \mathfrak{d})$.
For $\varepsilon >0$, we call 
\begin{align*}
	\mathcal{N}(\Omega, \varepsilon,\mathcal{X}) \coloneqq \min\bigg\{m \in \N\colon \exists \, (x_i)_{i=1}^m \in \mathcal{X}^m \text{ such that } \bigcup_{i=1}^m \overline{B(x_i,\varepsilon)} \supset \Omega\bigg\},
\end{align*}
the $\varepsilon$-covering number of $\Omega$ in $\mathcal{X}$, where $\overline{B(x,\varepsilon)} \coloneqq \{ z \in \mathcal{X}\colon \mathfrak{d}(z,x) \leq \varepsilon \}$ is the closed ball with radius $\varepsilon$ centered at $x$.
\end{definition}

Further, suppose that $f$ is an $C_{\rm Lip}$-Lipschitz continuous map for some $C_{\rm Lip}> 0$ from a metric space $\mathcal{X}$ to $L^\infty(\R^{{\rm d}_0})$. 
Then it is straightforward to verify that 
\begin{align}\label{eq:mappingPropertyCoveringNumbers}
	\mathcal{N}(f(\Omega), L\varepsilon, L^\infty(\R^{{\rm d}_0})) \leq \mathcal{N}(\Omega,\varepsilon,\mathcal{X}).
\end{align}

Let $\mathsf{b} \in (0,1]$, $\mathsf{B}\in [1,\infty)$.
Let $G= (V,E)$ be a \textit{fixed} network graph with $\# V_{\rm in} = {\rm d}_{\rm in}$ and $\# V_{\rm out} = {\rm d}_{\rm out}$. 
Consider the class of affine SNNs $\Psi = (A_{\rm in},\Phi,A_{\rm out})$ whose positive SNN $\Phi = (G,\mathsf{W},D)$ is built on $G$.
Using this, we define the following \textit{parameterized} 
class of these affine SNNs
\begin{alignat}{2} \label{eq:affineSNNparametrization}
	\nonumber \mathcal{P}_{\rm SNN}^{\star}(G,\vec{{\rm d}}; \mathsf{b}, \mathsf{B}) &\coloneqq \bigg\{\Psi = (A_{\rm in},\Phi,A_{\rm out})\colon && W_{\rm in} \in [-\mathsf{B},\mathsf{B}]^{{\rm d}_{\rm in} \times {\rm d}_0}, b_{\rm in}\in [-\mathsf{B},\mathsf{B}]^{{\rm d}_{\rm in}},\\     
	& && \mathsf{W} \in [\mathsf{b},\mathsf{B}]^{\#E}, D\in [0,\mathsf{B}]^{\#E}, \\
	\nonumber & && W_{\rm out} \in [-\mathsf{B},\mathsf{B}]^{{\rm d}_1 \times {\rm d}_{\rm out}}, b_{\rm out}\in [-\mathsf{B},\mathsf{B}]^{{\rm d}_1} \bigg\},
\end{alignat}
where $\vec{{\rm d}} \coloneqq ({\rm d}_0, {\rm d}_1, {\rm d}_{\rm in},{\rm d}_{\rm out})$. 
We equip $\mathcal{P}_{\rm SNN}^{\star}(G,\vec{{\rm d}}; \mathsf{b}, \mathsf{B})$ with the metric 
\begin{align*}
	\mathfrak{d}(\Psi, \widetilde{\Psi}) \coloneqq \max\{\|W_{\rm in} - \widetilde{W}_{\rm in}\|_{\ell^\infty}, \|b_{\rm in} - \tilde{b}_{\rm in}\|_{\ell^{\infty}}, \|\Phi - \widetilde{\Phi}\|_{\ell^\infty}, \|W_{\rm out} - \widetilde{W}_{\rm out}\|_{\ell^\infty}, \|b_{\rm out} - \tilde{b}_{\rm out}\|_{\ell^{\infty}}\},
\end{align*}
where we recall from \eqref{eqdef:PhiLinftydist} that for $\Phi = (G,\mathsf{W},D)$, $\widetilde{\Phi} = (G,\widetilde{\mathsf{W}},\widetilde{D})$, 
\begin{align*}
	\|\Phi - \widetilde{\Phi}\|_{\ell^{\infty}}= \max\{\|\mathsf{W} - \widetilde{\mathsf{W}}\|_{\ell^{\infty}}, \|D - \widetilde{D}\|_{\ell^{\infty}}\}.
\end{align*}
Then it is readily seen that $\mathcal{P}_{\rm SNN}^{\star}(G,\vec{{\rm d}}; \mathsf{b}, \mathsf{B})$ is isometrically isomorphic to a compact subset of $[-\mathsf{B}, \mathsf{B}]^{\mathsf{M}}$ where $\mathsf{M} \coloneqq {\rm d}_{\rm in}{\rm d}_0 + 2\#E + {\rm d}_1{\rm d}_{\rm out}$.

Let ${\rm d}_1 =1$.
Our hypothesis set of choice is as follows,
\begin{align} \label{eq:hypclass}
	\mathcal{H} = \Realization_{[0,1]}(\mathcal{P}_{\rm SNN}^{\star}(G,\vec{{\rm d}}; \mathsf{b}, \mathsf{B})).
\end{align}
Here, $\Realization_{[0,1]} (\mathcal{P}_{\rm SNN}^{\star}(G,\vec{{\rm d}}; \mathsf{b}, \mathsf{B}))$ denotes the image set of the map $\Realization_{[0,1]}$, which is the $[0,1]$-clipped realization map (see Definition~\ref{def:clipping}),
\begin{align*} 
	\Realization_{[0,1]}\colon \quad \mathcal{P}_{\rm SNN}^{\star}(G,\vec{{\rm d}}; \mathsf{b}, \mathsf{B}) \to  L^\infty(\R^{{\rm d}_0}).
\end{align*}
Next, we compute the covering number 
\begin{align*}
	\mathcal{N}(\Realization_{[0,1]}( \mathcal{P}_{\rm SNN}^{\star}(G,\vec{{\rm d}}; \mathsf{b}, \mathsf{B})), \varepsilon, L^\infty(\R^{{\rm d}_0})),
\end{align*}
a critical quantity that will appear in Theorem~\ref{thm:generalizationGapTheorem}.
By leveraging Theorem~\ref{thm:LipThm}
in Section~\ref{sec:LipschitzCont}, it is straightforward to see that $\Realization_{[0,1]}$
is a Lipschitz continuous map on $\mathcal{H}$ with the Lipschitz constant $C_{\rm Lip}$, where
\begin{align} \label{eqdef:Lstar}
    C_{\rm Lip}\leq C^{\star}_{\rm Lip}\coloneqq {\rm d}_{\rm out} \cdot
    \Big( 2{\rm d}_{\rm in}^{\frac{1}{2}}\, {\rm d}_0 \mathsf{B} + \mathsf{B}\mathsf{L}\cdot \Big(1+\frac{1}{\mathsf{b}^2}\Big) + 2\mathsf{B} + \mathsf{L}\cdot \Big(\frac{1}{\mathsf{b}} + \mathsf{B}\Big)\Big) + 1.
\end{align}
Therefore, we conclude, using \eqref{eq:mappingPropertyCoveringNumbers} and the well-known fact that the $\varepsilon$-covering number of $[-\mathsf{B},\mathsf{B}]^{\mathsf{M}}$ is $\lceil 2\mathsf{B}/ \varepsilon \rceil^{\mathsf{M}}$, that
\begin{align} \label{eq:coveringNumberEstimate}
	\log \Big(\mathcal{N}(\Realization_{[0,1]}( \mathcal{P}_{\rm SNN}^{\star}(G,\vec{{\rm d}}; \mathsf{b}, \mathsf{B})), \varepsilon, L^\infty(\R^{{\rm d}_0}))\Big)
	\leq \mathsf{M} \log \Big(\Big \lceil \frac{2 \mathsf{B} C^{\star}_{\rm Lip} }{\varepsilon} \Big\rceil\Big).
\end{align}

\begin{remark}
Note that the logarithm of the covering number derived in \eqref{eq:coveringNumberEstimate} depends linearly on the total number of parameters and only logarithmically on the depth of the underlying graph. 
This is a remarkable property since, for feedforward neural networks, the logarithm of the covering numbers depends linearly on the product of the number of layers and the number of parameters, see e.g.
\cite[Remark~1]{schmidt2020nonparametric}, \cite[Proposition~2.8]{berner2020analysis}. 

However, it should be emphasized that this improved Lipschitz control of affine SNNs comes at the expense of representational efficiency, as significantly more parameters are needed to match the expressive capacity of deep FNNs.
\end{remark}

We are now prepared to present the two main results of this section. 

\begin{theorem} \label{thm:generalizationGapTheorem}
Let ${\rm d}_1=1$, ${\rm d}_0,{\rm d}_{\rm in}, {\rm d}_{\rm out} \in \N$, $\mathsf{b}\in (0,1]$, and $\mathsf{B} \in [1,\infty)$. 
Let $G$ be a network graph with ${\rm d}_{\rm in}$ input neurons and ${\rm d}_{\rm out}$ output neurons. 
Let $\mathcal{D}$ be a distribution on $[0,1]^{{\rm d}_0} \times [0,1]$, and let $S \sim \mathcal{D}^m$ be a sample. 
Then it holds for all $\Psi \in \mathcal{P}_{\rm SNN}^{\star}(G,\vec{{\rm d}}; \mathsf{b}, \mathsf{B})$ that with probability $1-\delta$
\begin{align*}
	\risk(\Realization_{[0,1]}(\Psi)) \leq \widehat{\mathcal{R}}_S(\Realization_{[0,1]}(\Psi))
	+ \sqrt{\frac{2 (\mathsf{M} \log (m \lceil 16 \mathsf{B} C^{\star}_{\rm Lip} \rceil) +  \log(2/\delta))}{m}},
\end{align*}
for all 
\begin{align} \label{eq:assumptionOnM}
	m \geq 2 (\mathsf{M} \log (m \lceil 16 \mathsf{B} C^{\star}_{\rm Lip} \rceil) + \log (2/\delta)).
\end{align}
\end{theorem}

A proof of Theorem~\ref{thm:generalizationGapTheorem} is given in Appendix~\ref{appx:generalization}.

A sharper learning bound can be established for the generalization error of an empirical risk near-minimizer in the noiseless setting, where the sample $S=(x_i,f_0(x_i))_{i=1}^m$ is generated deterministically from a target function $f_0$.

\begin{theorem} \label{thm:generalizationGapTheoremRealizable}
Let ${\rm d}_1=1$, ${\rm d}_0,{\rm d}_{\rm in}, {\rm d}_{\rm out} \in \N$, $\mathsf{b}\in (0,1]$, and $\mathsf{B} \in [1,\infty)$. 
Let $G$ be a network graph with ${\rm d}_{\rm in}$ input neurons and ${\rm d}_{\rm out}$ output neurons. 
Let $f_0 \colon [0,1]^{{\rm d}_0} \to [0,1]$.
Let $\mathcal{D}$ be a distribution on $[0,1]^{{\rm d}_0}$, and let $S = (x_i, f_0(x_i))_{i=1}^m$ with $x_i \sim \mathcal{D}$ i.i.d. 
Then, for $\Psi_m \in \mathcal{P}_{\rm SNN}^{\star}(G,\vec{{\rm d}}; \mathsf{b}, \mathsf{B})$ satisfying
\begin{align}\label{eq:almostERM}
    \widehat{\mathcal{R}}_S(\Realization_{[0,1]} (\Psi_m)) 
    \leq \inf_{\Psi \in \mathcal{P}_{\rm SNN}^{\star}(G,\vec{{\rm d}}; \mathsf{b}, \mathsf{B})} \widehat{\mathcal{R}}_S(\Realization_{[0,1]}(\Psi)) + \frac{\sqrt{2}-1}{m},
\end{align}
it holds for all $\varepsilon\in (0,1]$ that,
\begin{multline} \label{eq:fastRateLearningForERM}
    \risk(\Realization_{[0,1]}(\Psi_m)) \\
    \leq  4 \bigg( \inf_{\Psi \in \mathcal{P}_{\rm SNN}^{\star}(G,\vec{{\rm d}}; \mathsf{b}, \mathsf{B})} \risk(\Realization_{[0,1]}(\Psi))   + \frac{18 \mathsf{M} \log \Big(\Big \lceil \frac{2 \mathsf{B} C^{\star}_{\rm Lip} }{\varepsilon} \Big\rceil\Big)  + 73}{m} +  32 \varepsilon  \bigg) 
    + \sqrt{\frac{\log(1/\delta)}{m}}
\end{multline}
with probability $1-\delta$.
\end{theorem}

A proof of Theorem~\ref{thm:generalizationGapTheoremRealizable} is given in Appendix~\ref{app:proofThmRealizable}.

In concluding this section, we briefly remark that there exists a substantial discourse on Lipschitz-continuous neural networks and their generalization properties, see e.g. \cite{bethune2022pay, virmaux2018lipschitz, gouk2021regularisation}.

\section{Expressivity of affine SNNs}\label{sec:expressivity}

It seems plausible that relying exclusively on excitatory responses (see Definition~\ref{def:spiking}) will restrict the range of functions that can be approximated by affine SNNs. 
Nonetheless, we will see in this section that many approximation results of feedforward neural networks can still be reproduced by affine SNNs and even improved. 

The core principle behind these expressivity results is the technical lemma below, whose proof is given in Appendix~\ref{appx:minapproxpf}. 

\begin{lemma} \label{lem:minapprox} 
Let ${\rm d}_0\in\N$ such that ${\rm d}_0\geq 2$, and let $\varepsilon>0$. 
Then there exists an affine SNN $\Psi_{\varepsilon}^{\min } = (A_{\rm in},\Phi^{\rm min}_{\varepsilon},A_{\rm out})$, with $\Realization(\Psi^{\min}_{\varepsilon})\colon \R^{{\rm d}_0}\to \R$, such that for all $x_1, \dots, x_{{\rm d}_0} \in \R$,
\begin{align} \label{minapproxconc}
	|\Realization(\Psi_{\varepsilon}^{\min })(x_1,\dots, x_{{\rm d}_0}) - \min\{x_1,\dots, x_{{\rm d}_0}\}|\leq \varepsilon.
\end{align}
Moreover, it holds that ${\rm Size}(\Psi_{\varepsilon}^{\min }) = 2{\rm d}_0+1$, the network graph has depth $1$, all the weights in $\Psi_{\varepsilon}^{\min }$ are bounded above in absolute value by $\max\{1, {1}/{\varepsilon}\}$, and all the weights of $\Phi^{\rm min}_{\varepsilon}$ are bounded below by ${1}/{\varepsilon}$.
\end{lemma}

\begin{remark}
Let us point out how remarkable the approximation of the minimum of Lemma \ref{lem:minapprox} is in view of approximation by shallow ReLU neural networks.
It can be seen from the proof of the lemma that the promised SNN 
has effectively one non-input neuron. 
It has been shown in \cite[Theorem 4.3]{safran2024many} that no ReLU neural network, regardless of its depth, can approximate the $d$-dimensional min operator with fewer than $d$ neurons per hidden layer. 
Moreover, shallow feedforward neural networks, i.e., those with fewer than three layers cannot efficiently approximate the min operator with $\mathcal{O}(d)$ neurons. 
\end{remark}

In the remainder of this section as well as the next, we continue to take ${\rm d}_1 = 1$.

\subsection{Universality} \label{sec:universality}

We start with the following essential lemma, which is a direct consequence of Lemma~\ref{lem:minapprox}.
Its proof is given in Appendix~\ref{appx:ReLUapproxpf}.

\begin{lemma} \label{lem:ReLUapprox}
Let ${\rm d}_0\in\N$, and let $\varepsilon>0$. 
Let $a\in\R^{{\rm d}_0}$, $b,c,d\in\R$.
Then there exists an affine SNN $\Psi_{a,b,c,d,\varepsilon}$, with $\Realization(\Psi_{a,b,c,d,\varepsilon})\colon \R^{{\rm d}_0}\to \R$, such that for all $x \in \R^{{\rm d}_0}$,
\begin{align} \label{eq:approximationOfReLU}
	|\Realization(\Psi_{a,b,c,d, \varepsilon})(x) - c\max\{a^\top x + b,0\} - d| \leq |c|\varepsilon.
\end{align}
Moreover, it holds that ${\rm Size}(\Psi_{a,b,c,d, \varepsilon})\leq {\rm d}_0+ 5$, 
all the weights in $\Psi_{a,b,c,d,\varepsilon}$ are bounded above in absolute value by $\max\{{1}/{\varepsilon}, \|a\|_{\ell^{\infty}}, |b|, |c|, |d|\}$, 
and all the synaptic weights are bounded below by $1/\varepsilon$.
\end{lemma}

To highlight the significance of Lemma~\ref{lem:ReLUapprox}, consider setting ${\rm d}_0=1$, $a=c=1$, $b=d=0$ in \eqref{eq:approximationOfReLU}.
This results in an affine SNN with constant size, whose realization is capable of approximating the ReLU function with an error of $\varepsilon$.
An immediate consequence is that affine SNNs exhibit the same level of expressiveness as shallow ReLU feedforward neural networks.
Specifically, it can be further observed that, for each $M\in\N$, the space
\begin{align} \label{eq:thisStatementCanBeUsedForBarronToo}
	H_M \coloneqq \bigg\{ x\mapsto \sum_{i=1}^M c_i \max\{a_i^\top x + b_i,0\} + d_i\colon  a_i \in \R^{{\rm d}_0}, b_i, c_i, d_i \in \R \bigg\}
\end{align}
is contained in the closure of the set of realizations of affine SNNs.
Since the finite sums of so-called \textit{ridge functions} \eqref{eq:thisStatementCanBeUsedForBarronToo} are known to be universal approximators, see e.g. \cite{leshno1993multilayer}, it implies that affine SNNs also share this property. 
We detail this in the theorem below, whose proof is given in Appendix~\ref{appx:universalitypf}.

\begin{theorem} \label{thm:universality}
Let ${\rm d}_0\in\N$, and let $\varepsilon>0$.
Let $\Omega \subset \R^{{\rm d}_0}$ be a compact domain. 
Then for every $f \in \mathcal{C}(\Omega)$, there exists an affine SNN $\Psi^f$ such that $\sup_{x\in \Omega} |f(x) - \Realization(\Psi^f)(x)| \leq \varepsilon$.
\end{theorem}


\subsection{Emulation of finite element spaces} \label{sec:FEM}

Lemma~\ref{lem:minapprox} confirms that affine SNNs are capable of approximating the minimum of multiple inputs to a given accuracy. This suggests the potential for effective approximation of \textit{linear finite element spaces}, which comprise continuous piecewise affine functions. 
A similar capacity has already been documented in \cite{he2020relu} for feedforward neural networks. 
We aim to translate these results to affine SNNs below.

For a compact domain $\Omega\subset\R^{{\rm d}_0}$, a \emph{finite element space} is based on a simplicial triangulation of $\Omega$ using simplices that we define below.

\begin{definition}
Let ${\rm d}_0\in\N$, $n\in\N_0$ such that $n\le {\rm d}_0$.
We call
$x_0,\dots,x_n\in\R^{{\rm d}_0}$ affinely independent points if and only if either $n=0$ or
$n\ge 1$, and the vectors $x_1-x_0,\dots,x_n-x_0$ are linearly independent.

An $n$-simplex is the convex hull of a set of $n+1$ affinely independent points $x_0,\dots,x_n$, denoted as $\mathrm{co}(x_0,\dots,x_n)$.
\end{definition}

A simplicial triangulation of $\Omega$ is a partition of $\Omega$ into simplices. 

\begin{definition}\label{def:mesh}
Let ${\rm d}_0\in\N$, and let $\Omega\subset\R^{{\rm d}_0}$ be compact.
Let $\mathcal{N}\subset\Omega$ be a finite set and let $\mathcal{T}$ be a finite set of ${\rm d}_0$-simplices such that for each $\tau\in\mathcal{T}$, the set $N(\tau) \coloneqq \mathcal{N} \cap \tau$ has cardinality ${\rm d}_0+1$ and $\tau = \mathrm{co}(N(\tau))$.
We call $\mathcal{T}$ a regular triangulation of $\Omega$, if and only if
\begin{enumerate}
\item $\bigcup_{\tau\in\mathcal{T}}\tau = \Omega$,
\item\label{def:meshint}
for all $\tau$, $\tau'\in\mathcal{T}$ it holds that $\tau\cap\tau'=\mathrm{co}(N(\tau)\cap N(\tau'))$.
\end{enumerate}
We call $\eta\in \mathcal{N}$ a node. 
Finally, we call 
\begin{align*}
	h_{\min}(\mathcal{T}) &\coloneqq \min_{\tau\in\mathcal{T}} \min_{\eta_1\neq \eta_2\in N(\tau)} |\eta_1 - \eta_2|,\\
	h_{\max}(\mathcal{T}) &\coloneqq \max_{\tau \in \mathcal{T}}\max_{\eta_1,\eta_2 \in N(\tau)} |\eta_1 - \eta_2|,
\end{align*}
the min mesh-size and the max mesh-size of $\mathcal{T}$, respectively.
\end{definition}

For each $\eta \in \mathcal{N}$, we let $T(\eta) \coloneqq \{\tau \in \mathcal{T} \colon \eta \in \tau\}$ be the set of \emph{simplices that contain $\eta$}, and let $G(\eta) \coloneqq \bigcup_{\tau\in T(\eta)}\tau$. 
Given a triangulation $\mathcal{T}$, the associated \emph{linear finite element space} is defined to be
\begin{align*}
	V_{\mathcal{T}} \coloneqq \{f \in \mathcal{C}(\Omega)\colon f \text{ is affine on all } \tau \in  \mathcal{T}\}.
\end{align*}
Let $\phi_{\eta}$ be the unique function in $V_{\mathcal{T}}$ such that for all $\eta' \in \mathcal{N}$,
\begin{align} \label{eq:delta}
	\phi_{\eta}(\eta') = \delta_{\eta, \eta'}.
\end{align}
It follows that $(\phi_{\eta})_{\eta \in \mathcal{N}}$ is a basis for $V_{\mathcal{T}}$.
In the case that $G(\eta)$ is convex, there exists a simplified formula for $\phi_{\eta}$. 
Specifically, it is shown in \cite[Lemma 3.1]{he2020relu} that
\begin{align}\label{eq:HatsAsMinima}
	\phi_\eta( x ) = \max\{0, \min_{\tau \in T(\eta)}g_\tau(x)\} = \min_{\tau \in T(\eta)} g_\tau(x) - \min\{ 0, \min_{\tau \in T(\eta)} g_\tau(x)\},
\end{align}
for all $x \in \Omega$, where $g_\tau$ is the unique globally affine function such that $g_\tau = \phi_{\eta}$ on $\tau$. 
Thus, Lemma~\ref{lem:minapprox}, coupled with \eqref{eq:HatsAsMinima}, suggests the existence of moderately sized affine SNNs capable of approximating the functions $\phi_\eta$ with arbitrary precision, and subsequently, the elements in $V_{\mathcal{T}}$.
This argument is formalized in Theorem~\ref{thm:reapproximationofFEMSpaces} below, with its proof given in Appendix~\ref{appx:FEMpf}.
In what follows, a \textit{convex regular triangulation} $\mathcal{T}$ refers to a regular triangulation $\mathcal{T}$ for which $G(\eta)$ is convex for each node $\eta$.

\begin{theorem} \label{thm:reapproximationofFEMSpaces}
Let ${\rm d}_0\in\N$, and let $\Omega \subset \R^{{\rm d}_0}$ be compact.
Let $\mathcal{T}$ be a convex regular triangulation of $\Omega$ with node set $\mathcal{N}$. 
Let $\varepsilon>0$. 
Then for all $f \in V_{\mathcal{T}}$ there exists an affine SNN $\Psi_\varepsilon^{f}$ such that for all $x \in \Omega$,
\begin{align*}
	|\Realization(\Psi^f_{\varepsilon})(x) - f(x)|\leq \sum_{\eta \in \mathcal{N}} |f(\eta)| \varepsilon.
\end{align*}
Moreover, 
\begin{align*}
	{\rm Size}(\Psi^{f}_{\varepsilon}) \leq \sum_{\eta\in\mathcal{N}} \big(\#T(\eta)({\rm d}_0 + 2) +6\big), 
\end{align*}
all the weights in $\Psi_\varepsilon^f$ are bounded above in absolute value by 
\begin{align*}
	\max\Big\{1,\|(f(\eta))_{\eta\in\mathcal{N}}\|_{\ell^{\infty}}/h_{\min}(\mathcal{T}), C{\rm d}_0\|(f(\eta))_{\eta\in\mathcal{N}}\|_{\ell^{\infty}}/h_{\min}(\mathcal{T}),3/\varepsilon\Big\}
\end{align*}
for some $C = C(\Omega)>0$, and all the synaptic weights are bounded below by $\min\{1, {3}/{\varepsilon}\}$. 
\end{theorem}

We conclude this subsection by highlighting another notable approximation result of affine SNNs for smooth functions.
To articulate it, we need a geometric condition to hold for the compact domain $\Omega$.

\begin{definition}
We say that a compact domain $\Omega\subset \R^{{\rm d}_0}$ is admissible if for every $N\in \N$, there exist a triangulation $\mathcal{T}_N$ and universal constants $C_1, C_2 , c_2>0$ such that $\#\mathcal{T}_N \leq C_1 N$ and that
\begin{align} \label{eq:comparability}
	c_2 N^{-1/{\rm d}_0} \leq h_{\min}(\mathcal{T}_N) \leq h_{\max}(\mathcal{T}_N) \leq C_2 N^{-1/{\rm d}_0}.
\end{align}
\end{definition}

The anticipated approximation result for smooth functions is as follows.
We provide a proof in Appendix~\ref{appx:approximationWsinfinity}.

\begin{theorem}\label{thm:approximationWsinfty}
Let ${\rm d}_0 \in \N$, and let $\Omega\subset \R^{{\rm d}_0}$ be an admissible domain.
Let $s\in\{1,2\}$.
Then for every $f\in W^{s, \infty}(\Omega)$ and every $N\in \N$, there exist an affine SNN $\Psi^f_N$ and a constant $C_1=C_1({\rm d}_0)$ such that 
\begin{align} \label{eq:Sobolevthmconc}
    \|\Realization(\Psi^f_N) - f\|_{L^{\infty}(\Omega)} \leq C_1 N^{-s/{\rm d}_0} \|f \|_{W^{s,\infty}(\Omega)}.
\end{align}
Moreover, $\mathrm{Size}(\Psi^f_N) \leq C_2 N$, and all weights of $\Psi^f_N$ are bounded above in absolute value by 
\begin{align*}
	\max \Big\{ C_3N^{1/{\rm d}_0} \|f\|_{L^{\infty}(\Omega)}, 2 N^{s/{\rm d}_0 + 1} \Big\}
\end{align*}
and all synaptic weights are bounded below by $2N^{s/{\rm d}_0 +1}$, for some constants $C_2 = C_2(\Omega)$, $C_3=C_3(\Omega)$.

\end{theorem}

\subsection{Curse of dimensionality} \label{sec:COD}

It is well-known that sums of ReLU feedforward neural networks can overcome the curse of dimensionality when approximating specific functions of bounded variation \cite{barron, wojtowytsch2022representation, siegel2022high, caragea2023neural}, called \textit{Barron functions}.
We recall the definition of these spaces below. 

\begin{definition}
Let ${\rm d}_0 \in \N$. 
Let $K >0$. 
The Barron class $\Gamma_K$ with a constant $K$ is defined to be the set of functions $f\in L_{\rm loc}^1(\R^{{\rm d}_0})$ for which there exists a measurable function $\hat{f}$ such that, for all $x\in\R^{{\rm d}_0}$
\begin{align*}
	f(x) = \int_{\R^{{\rm d}_0}} e^{i x \xi} \hat{f}(\xi) \,{\rm d}\xi, \quad 
\text{ 
and } \quad 
	    \int_{\R^{{\rm d}_0}} |\xi| |\hat{f}(\xi)| \, {\rm d}\xi \leq K.
\end{align*}
\end{definition}

The central point of this subsection is to showcase that affine SNNs can approximate Barron functions with a dimension-independent rate. 
The formal result is as follows.

\begin{theorem} \label{thm:COD}
Let ${\rm d}_0\in \N$. 
There is a universal constant $\nu > 0$ such that the following holds.
For every $K > 0$, every $f \in \Gamma_K$, and every $M \in \N$, there exists an affine SNN $\Psi^f_M$ such that
\begin{align}\label{eq:ShowThisForBarron}
	\sup_{x \in \overline{B(0,1)}} |\Psi^f_M(x) - f(x)| \leq  \frac{\nu\,{\rm d}_0^{\frac{1}{2}} K}{\sqrt{M}}.
\end{align}
Furthermore, there exist $C,c>0$ such that ${\rm Size}(\Psi^f_M) \leq C{\rm d}_0 M$, all weights in $\Psi^f_M$ can be bounded above by $C\cdot (M^{3/2}/\sqrt{K} + \sqrt{K})$, and all synaptic weights are bounded below by $cM^{3/2}/\sqrt{K}$.
\end{theorem}

A proof of Theorem~\ref{thm:COD} is given in Appendix~\ref{appx:CODpf}.
It is a product of Lemmas~\ref{lem:addition}, \ref{lem:ReLUapprox}, and the known result that Barron functions can be uniformly approximated on compact domains without the curse of dimensionality by sums of ridge functions \cite{barron1992neural, caragea2023neural}.

\section{Full error analysis}\label{sec:fullError}


We recall from Section~\ref{sec:generalization} the following parameterized hypothesis class 
\begin{align*}
	\mathcal{H} = \Realization_{[0,1]} (\mathcal{P}_{\rm SNN}^{\star}(G,\vec{{\rm d}};\mathsf{b},\mathsf{B})).
\end{align*}
The class encompasses clipped realizations of affine SNNs associated with a fixed network graph $G=(V,E)$.
As defined in \eqref{eq:affineSNNparametrization}, \eqref{eq:hypclass}, $\vec{{\rm d}} = ({\rm d}_0, {\rm d}_1, {\rm d}_{\rm in},{\rm d}_{\rm out}) = ({\rm d}_0, 1, {\rm d}_{\rm in},{\rm d}_{\rm out})$, $\mathsf{B}$ denotes an upper bound for all the weight and delay parameters, and $\mathsf{b}$ denotes a lower bound for the synaptic weights of the associated positive SNNs.
Moreover, the class $\Realization_{[0,1]}(\mathcal{P}_{\rm SNN}^{\star}(G,\vec{{\rm d}};\mathsf{b},\mathsf{B}))$
was identified to be isometrically isomorphic to a compact domain in $[-\mathsf{B},\mathsf{B}]^{\mathsf{M}}$, where the total Euclidean dimension $\mathsf{M} = {\rm d}_{\rm in}{\rm d}_0 + 2\#E + {\rm d}_1{\rm d}_{\rm out} = {\rm d}_{\rm in}{\rm d}_0 + 2\#E + {\rm d}_{\rm out}$.

We first present the main learning theorem of this section, which will thereafter be applied to the approximation results of the previous section.

\begin{theorem} \label{thm:fullErrorBound}
Let ${\rm d}_0 \in \N$.
Let $\kappa_{\mathsf{B}}, \kappa_{\mathsf{M}}>0$.
Let $f_0:[0,1]^{{\rm d}_0}\to [0,1]$, and suppose for every $\varepsilon\in (0,1)$, there exists a hypothesis class of clipped realizations of affine SNNs,
\begin{align} \label{eq:Hepsilon}
	\mathcal{H}(\varepsilon) \coloneqq 
	\Realization_{[0,1]}\big( \mathcal{P}_{\rm SNN}^{\star} \big(G(\varepsilon), \vec{{\rm d}}(\varepsilon); \mathsf{b}(\varepsilon), \mathsf{B}(\varepsilon) \big)\big),
\end{align}
such that
\begin{align} \label{eq:uniformapprox}
    \inf\big\{\|f_0-g\|_{L^{\infty}([0,1]^{{\rm d}_0})}\colon g \in \mathcal{H}(\varepsilon) \big\} \leq \varepsilon.
\end{align}
Here in \eqref{eq:Hepsilon}, $\vec{{\rm d}}(\varepsilon)\coloneqq({\rm d}_0, 1,{\rm d}_{\rm in}(\varepsilon),{\rm d}_{\rm out}(\varepsilon))$, and $G(\varepsilon)$ is a network graph with ${\rm d}_{\rm in}(\varepsilon)$ input nodes and ${\rm d}_{\rm out}(\varepsilon)$ output nodes, such that the total dimension satisfies
\begin{align} \label{secondeq}
    \mathsf{M}(\varepsilon) \leq  \varepsilon^{-\kappa_{\mathsf{M}}}.
\end{align}
Moreover, $\mathsf{b}(\varepsilon)$, $\mathsf{B}(\varepsilon)\in\R$ are constrained by
\begin{align} \label{firsteq}
	\varepsilon^{-\kappa_{\mathsf{B}}} \geq  \mathsf{B}(\varepsilon) \geq 1\geq \mathsf{b}(\varepsilon) \geq \varepsilon^{\kappa_{\mathsf{B}}},
\end{align}
Let $\mathcal{D}$ be a distribution on $[0,1]^{{\rm d}_0}$. 
Let $S = (x_i, f_0(x_i))_{i=1}^m$ be a sample where $x_i \sim \mathcal{D}$ are i.i.d. 
Let $g_m\in \mathcal{H}(m^{{-1}/(\kappa_{\mathsf{M}} + 2)})$ that is also an empirical risk minimizer based on the sample $S$.
Then there exists a universal constant $c>0$, such that the event
\begin{align} \label{eq:risksimp}
    \risk(g_m) &\leq c\max\big\{1,\kappa_{\mathsf{B}}/ \kappa_{\mathsf{M}}\big\}m^{-2/(\kappa_{\mathsf{M}}+2)} \log(m) + \sqrt{\frac{\log(1/\delta)}{m}}
\end{align}
holds with a probability at least $1-\delta$. 
\end{theorem}

A proof of Theorem~\ref{thm:fullErrorBound} is given in Appendix~\ref{appx:fullErrorBound}.
It follows from the proof that the conclusion remains valid if the bounds $\varepsilon^{-\kappa_{\mathsf{B}}}$, $\varepsilon^{\kappa_{\mathsf{B}}}$ in \eqref{firsteq} or $\varepsilon^{-\kappa_{\mathsf{M}}}$ in \eqref{secondeq} are only assumed to hold up to a multiplicative constant, with the constant then affecting the final estimate.

Let us describe Theorem~\ref{thm:fullErrorBound} in words. 
First, the condition \eqref{eq:uniformapprox} necessitates that the family of hypothesis classes $(\mathcal{H}(\varepsilon))_{\varepsilon\in (0,1)}$ possesses the ability to achieve uniform $\varepsilon$-approximation of $f$. 
Second, when \eqref{eq:uniformapprox} is met alongside the growth conditions on the parameters \eqref{secondeq}, \eqref{firsteq}, the theorem asserts that learning can be accomplished with high probability.
We want to emphasize that \eqref{eq:uniformapprox} holds when $f$ belongs to a Sobolev class, as per Theorem~\ref{thm:approximationWsinfty}, or a Barron class, as per Theorem~\ref{thm:COD}.
Therefore, below, we will provide the appropriate values for ${\rm d}_{\rm in}(\varepsilon), {\rm d}_{\rm out}(\varepsilon)$, $\kappa_{\mathsf{B}}$, $\kappa_{\mathsf{M}}$ for these two results and present the overall learning error estimate.

\begin{itemize}
\item \textbf{Sobolev functions.} 
Let $s \in \{1,2\}$ and $\Omega \subset \R^{{\rm d}_0}$ be a compact domain. 
Let $f \in W^{s, \infty}(\Omega)$ such that $\|f\|_{W^{s, \infty}(\Omega)} \leq 1$.
Let $G(\varepsilon)$ be the network graph of the SNN of Theorem \ref{thm:approximationWsinfty} with ${\rm d}_{\rm in}(\varepsilon)= {\rm d}_{\rm out}(\varepsilon) = N = \lceil \varepsilon^{-{{\rm d}_0}/s}\rceil$.
Let $\kappa_{\mathsf{B}} = {\rm d}_0/s +1$, $\kappa_{\mathsf{M}} = {{\rm d}_0}/s$, and $\varepsilon=m^{-1/(\kappa_{\mathsf{M}}+2)}$.
Then Theorem~\ref{thm:fullErrorBound} yields that the risk of the empirical risk minimizer over affine SNNs is asymptotically bounded by a constant multiple of
\begin{align*}
    m^{-2/({\rm d}_0/s + 2)} \log(m) + \sqrt{\frac{\log(1/\delta)}{m}},
\end{align*}
with probability $1-\delta$.
\item \textbf{Barron functions.} 
Let $K > 0$. 
Let $f \in \Gamma_K$.
Let $G(\varepsilon)$ the network graph of the SNNs of  Theorem~\ref{thm:COD} with ${\rm d}_{\rm in}(\varepsilon)= {\rm d}_{\rm out}(\varepsilon) = \lceil \nu^2 {\rm d}_0 K^2 \varepsilon^{-2}\rceil$.
Let $\kappa_{\mathsf{B}} = 3$, $\kappa_{\mathsf{M}} = 2$, and $\varepsilon=m^{-1/(\kappa_{\mathsf{M}}+2)}$.
Then Theorem~\ref{thm:fullErrorBound} yields that the risk of the empirical risk minimizer over affine SNNs is asymptotically bounded by a constant multiple of  
\begin{align} \label{eq:BarronOverallEstimate}
    m^{-1/2} {\log(m)} + \sqrt{\frac{\log(1/\delta)}{m}},
\end{align}
with probability $1-\delta$.
It is worth noting that there is no dimension dependence in the exponent of $m$ in \eqref{eq:BarronOverallEstimate}.
Therefore, the overall learning error bound can be seen to have overcome the curse of dimensionality.
\end{itemize}


\section{Simulation results}\label{sec:simulations}

We complement the mathematical results with simulations of affine SNNs in pyTorch for a series of machine learning tasks, unifying the derived theory with practical applications.
For all experiments, the neural network architecture consists of an affine encoder, $L$ layers of simple spiking neurons, and an affine decoder.
Encoder and decoder were realized using standard linear layers in pyTorch, with weights and biases lower and upper bounded by $-B$ and $B$, respectively.
In all experiments, we set $B = 10$.
For the simple neurons, a custom module was implemented that solves for output spike times analytically, similar to related spiking neuron models \cite{mostafa2017supervised,comsa2020temporal,goltz2021fast},
\begin{equation}
    t_v = \frac{1 + \sum_{u \in \mathcal{C}_v} \mathsf{w}_{(u,v)} (t_u + d_{(u,v)})}{\sum_{u \in \mathcal{C}_v} \mathsf{w}_{(u,v)}} \,,
\end{equation}
where $\mathcal{C}_v$ is a set containing the indices of all causal input spike times, i.e., 
\begin{equation*}
    \mathcal{C}_v \coloneqq \{u: (u,v)\in E \text{ and } t_u + d_{(u,v)} \leq t_v \}.
\end{equation*}
For simplicity, we set all synaptic delays $d_{(u,v)}$ to 0 in simulations unless stated otherwise.
For the spiking layers, weights $\mathsf{w}_{(u,v)}$ are lower and upper bounded by $b$ and $B$. 
In pyTorch, we realize this by applying the clamp function to the weights in each layer's forward function.
Since the encoder, output spike times, and decoder are all differentiable, the standard backward function of pyTorch was used to train affine SNNs.
This allows the usage of exact gradients to train the SNN, as also done for related single-spike neural models \cite{mostafa2017supervised,comsa2020temporal,goltz2021fast}.
Alternatively, approximate methods based on SpikeProp \cite{bohte2000spikeprop}, such as surrogate gradients \cite{neftci2019surrogate}, can be used.

First, we illustrate the results of Lemmas~\ref{lem:minapprox} and \ref{lem:ReLUapprox}.
In Figures~\ref{fig:numerics}A and B, we show that an affine SNN can be set up to approximate the min and max operator with an error upper bounded by $\varepsilon$.
For the min operator, we use the identity function as encoder and decoder, and we set the weights to $\mathsf{w}_{(u,v)} = \mathrm{max}(1, \frac{1}{\varepsilon})$.
To evaluate the neural network, we randomly generate $1000$ inputs from a uniform distribution (centered around 0), with ${\rm d}_0 = 784$, ${\rm d}_\mathrm{in} = {\rm d}_1 = {\rm d}_\mathrm{out} = 1$.
For the max operator, we use the identity function multiplied by $-1$ as encoder and decoder, and we set the weights to $\mathsf{w}_{(u,v)} = \frac{1}{\varepsilon}$.
To evaluate it, we again generate $1000$ random inputs from a uniform distribution (centered around 0), with ${\rm d}_0 = {\rm d}_\mathrm{in} = {\rm d}_1 = {\rm d}_\mathrm{out} = 1$. 
\begin{figure} 
    \centering
    \includegraphics[width=\linewidth]{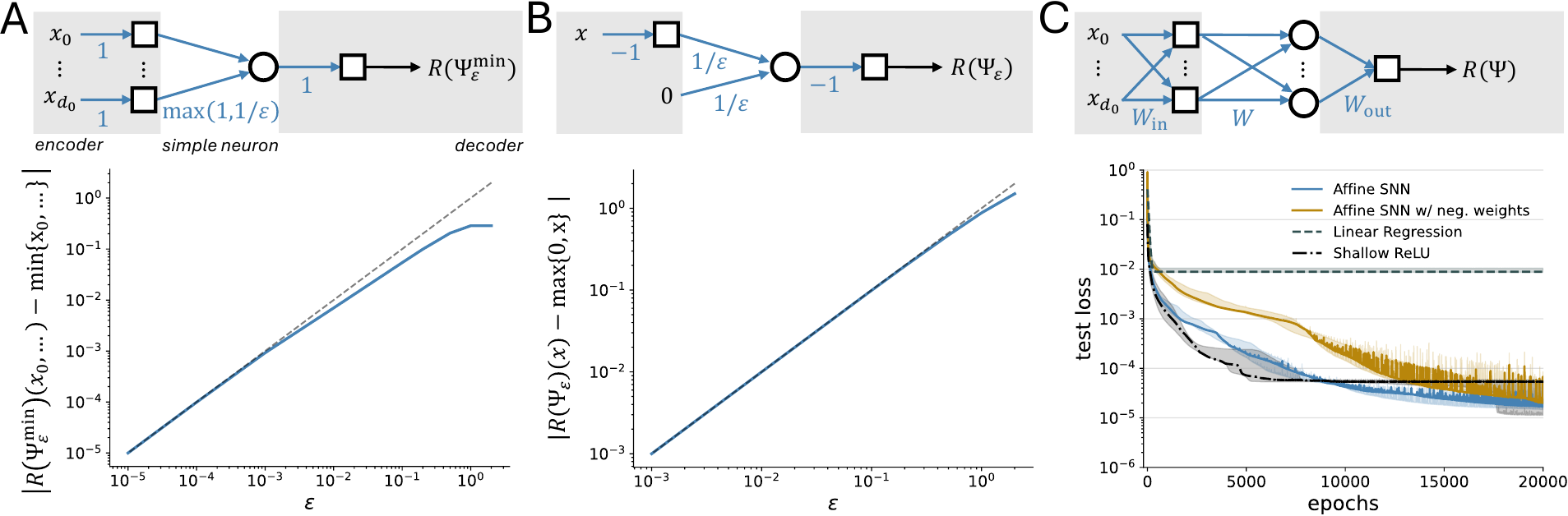}
    \caption{Affine SNN approximating \textbf{A} the min operator and \textbf{B} the max operator with an error (blue line) bounded by $\varepsilon$ (dashed diagonal line). \textbf{C} An affine SNN is trained to reproduce a shallow ReLU neural network. As references, we also trained a linear neural network, a shallow ReLU neural network, an affine SNN in our model, and an affine SNN with both positive and negative synaptic weights. 
    We show the median (line) and first and third quartiles (shaded area) over five different random seeds.}
    \label{fig:numerics}
\end{figure}

We further show that an affine SNN is capable of approximating shallow ReLU neural networks. 
To do this, we initialized a random ReLU neural network with ${\rm d}_0 = 40$ and $M = 20$ (see Lemma~\ref{lem:ReLUapprox}).
The weights were initialized by sampling from a rescaled uniform distribution, i.e. $W_{ij} \sim \mathcal{U}\big(- \frac{1}{\sqrt{m}}, \frac{1}{\sqrt{m}}\big)$ for $W_{ij} \in \mathbb{R}^{n \times m}$, which is the default initialization in pyTorch.
From this ReLU neural network, we randomly sampled $10^4$ training and $10^3$ test examples.
We trained an affine SNN (${\rm d}_0 = 40$, ${\rm d}_\mathrm{in} = 40$, ${\rm d}_1 = 20$, ${\rm d}_\mathrm{out} = 1$, and $b=0.01$) using a mean squared error loss, a learning rate of $10^{-3}$, and L2 regularization with coefficient $10^{-5}$. 
The whole experiment was repeated for five different random seeds.
In Figure~\ref{fig:numerics}C, the median test loss is shown.
In addition, we show the test loss achieved by linear regression, an affine SNN with negative and positive weights, as well as a ReLU neural network with the same architecture as the neural network that was used to generate the training and test data.
The affine SNN performs equally well as the ReLU neural network, with both learning significantly faster and more stably than an affine SNN with both positive and negative weights.

\begin{figure}[b!] 
    \centering
    \includegraphics[width=.9\linewidth]{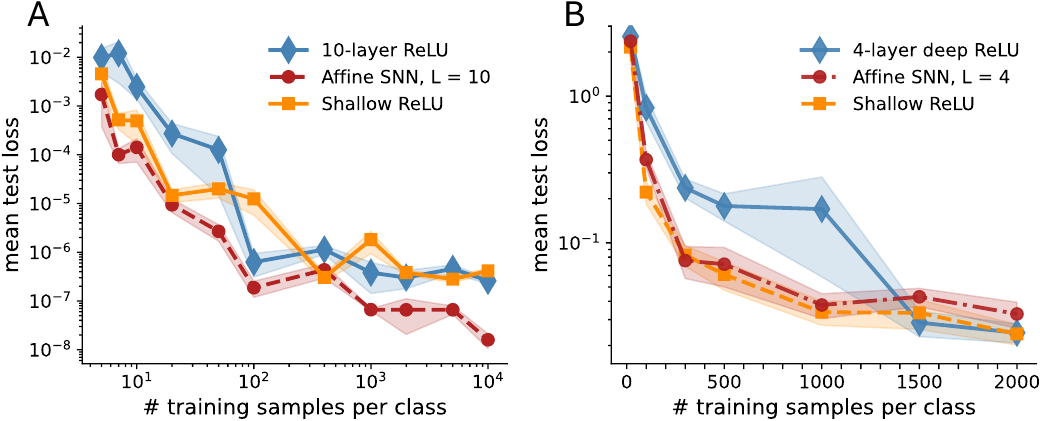}
    \caption{Comparison of affine SNNs and ReLU networks for different tasks. \textbf{A} Mean test loss (line) and standard error (shaded) for the regression task. 
    The test loss is reported for the epoch with the lowest training loss (for a maximum of 2000 epochs).
    Simulations were repeated for five random seeds.
    \textbf{B} Same as A, but for a classification task (Yin-Yang \cite{kriener2022yin}) and a maximum of 10000 epochs.}
    \label{fig:training}
\end{figure}
To demonstrate the generalization properties of affine SNNs, we trained three different network architectures on a simple regression task (fitting a quadratic function $f(x) = x / 4 + x^2$): a 10-layer deep ReLU network with weights bounded by $-B$ and $B$ and $100$ neurons per hidden layer, an affine SNN (${\rm d}_\mathrm{in} = 10, {\rm d}_0 = 1, {\rm d}_1 = 1, {\rm d}_\mathrm{out} = 100$, $L = 10$, $b = 0.1$) with $100$ neurons per hidden layer and trainable delays bounded by $[0, B]$, and a shallow ReLU network with $30000$ hidden neurons -- chosen so that all three networks have approximately the same number of trainable parameters.
To cover different severity levels of overfitting, we trained all networks on different numbers of training samples per class, from just 5 per class up to 10000. 
Testing was done on a separate set of 2000 samples, and we report the test loss obtained from the epoch with the lowest training loss, i.e., the one where the model fitted the training data best.
In all cases, no weight decay or other regularization methods were used.
Furthermore, we used a training rate of $10^{-3}$ (Adam optimizer), and a mean squared error loss function.
In Figure~\ref{fig:training}A, we show the mean test loss and standard error for all three models after five runs with different random seeds.
For a low number of training samples, both the affine SNN and shallow ReLU network outperform the deep ReLU network.
However, for larger training datasets, the affine SNN starts outperforming both the shallow and deep ReLU networks.

We further show results for a classification task using the Yin-Yang dataset \cite{kriener2022yin}, again for the same three models, but with different parameters: a 4-layer deep ReLU network with $100$ neurons per hidden layer, an affine SNN (${\rm d}_\mathrm{in} = 40, {\rm d}_0 = 4, {\rm d}_1 = 3, {\rm d}_\mathrm{out} = 100$, $L = 4$, $b = 0.1$) with $100$ neurons per hidden layer and trainable delays bounded by $[0, B]$, and a shallow ReLU network with $4000$ hidden neurons.
In this experiment, we used a learning rate of $10^{-2}$ (Adam optimizer, batch size of $400$) and a cross-entropy loss function.
The test loss of the affine SNN behaves similarly to the equivalent shallow ReLU neural network -- although for very small and very large numbers of training samples, all networks perform similarly, with the deep ReLU neural network reaching the lowest test loss.

Finally, we trained an affine SNN (${\rm d}_0 = {\rm d}_\mathrm{in} = 784$, ${\rm d}_1 = 200$, and ${\rm d}_\mathrm{out} = 10$) on the MNIST and Fashion MNIST task using a cross-entropy loss function, reaching similar performance levels as reported for deep feedfoward neural networks in the literature: $96.75^{+0.02}_{-0.08}\ \%$ median test accuracy for MNIST, with upper and lower index being the distance to the third and first quartile, and $87.81^{+0.04}_{-0.60}\ \%$ for Fashion MNIST.
The experiments were repeated five times with different random seeds.

The simulation code is available online\footnote{\url{https://github.com/dodo47/affineSNN}, commit abaa$4$c$0$}.

\section{Discussion}

In this study, we identify the fact that the weights in neural networks of simple spike-response neurons are allowed to be negative or arbitrarily close to zero as a reason why the parameterization of such SNNs through their weights is not continuous.
This includes a larger family of spike models, such as the simple SRMs and variations of the current-based LIF models.
We propose a solution to this shortcoming via a modified type of SNNs called affine SNNs, where the spiking neurons have exclusively positive weights, and show that it possesses remarkable approximation properties. 

For mathematical tractability and to shift the focus of our analysis to the temporal aspect of SNNs, the simple SRM has been used. 
This model is, at the time of this study, not found in neuromorphic platforms \cite{frenkel2023bottom}, which at first glance limits the adaptability of our results to real-world applications.
However, it is the limit case of the current-based LIF neuron model for large time constants and negligible leak, thus connecting to one of the currently most used spiking neuron models, both in simulations and neuromorphic implementations.
This directly provides a neuromorphic implementation by modeling the neuron using a capacitance with a very large time constant, which features the required linear charging profile in the interval between the initial potential and the threshold potential.
To ensure that spike times remain within a given experimental time window, appropriate regularization terms can be added during training.
Thus, we are confident that our model can be reproduced in neuromorphic devices, either by approximating the used SRM through IF neurons, or by designing explicit circuits.

In general, our approach highlights the need to identify SNNs that depend continuously on their parameters to ensure stable training.
Notably, recent work demonstrates that the Quadratic Integrate-and-Fire (QIF) model \cite{klos2023smooth}, which features a self-amplification mechanism, i.e. a quadratic rise that emulates the action potentials of biological neurons, satisfies this condition.

Lastly, our work complements recent results on mapping the parameters of ReLU neural networks to simple SRM SNNs \cite{stanojevic2023exact,stanojevic2024high}, proving that such SNNs not only share many properties of feedforward ANNs, but also possess superior properties.

In conclusion, we believe that the construction presented in this work is a first step toward identifying spiking neuron models and SNN architectures that admit
\begin{enumerate}
\item continuous dependence on parameters,
\item no worse approximation performance than deep feedforward neural networks for relevant function classes,
\item better performance than deep feedforward neural networks in some tasks,
\item superior generalization performance (as evidenced by smaller covering numbers) over deep feedforward neural networks.
\end{enumerate}
However, we also note that there are still many challenges ahead, especially in deriving such results for neuron models that are closer to biology or engineering applications (e.g., manufactured in neuromorphic devices), including features such as negative weights, synaptic response functions that are temporally bounded, alternative encodings (e.g. bursts \cite{payeur2021burst} and dendritic spikes \cite{larkum2022guide}) and adaptation mechanisms such as spike-frequency adaptation and short-term plasticity. 

\section*{Acknowledgements}

A.M.N. and P.C.P. were 
supported by the Austrian Science Fund (FWF) Project P-37010.
D.D. was supported by the Horizon Europe's Marie Skłodowska-Curie Actions (MSCA) Project 101103062 (BASE).
The authors would like to thank M. Singh, A. Fono, and G. Kutyniok for enlightening discussions on the subject.

\bibliographystyle{plain}
\bibliography{spikingNN}

\begin{thebibliography}{10}

\bibitem{abeles1982role}
Moshe Abeles.
\newblock Role of the cortical neuron: integrator or coincidence detector?
\newblock {\em Israel journal of medical sciences}, 18(1):83--92, 1982.

\bibitem{anthony1999neural}
Martin Anthony and Peter~L Bartlett.
\newblock {\em Neural network learning: Theoretical foundations}, volume~9.
\newblock cambridge university press Cambridge, 1999.

\bibitem{barron1992neural}
Andrew~R Barron.
\newblock Neural net approximation.
\newblock In {\em Proc. 7th Yale workshop on adaptive and learning systems}, volume~1, pages 69--72, 1992.

\bibitem{barron}
Andrew~R. Barron.
\newblock Universal approximation bounds for superpositions of a sigmoidal function.
\newblock {\em IEEE Trans. Inform. Theory}, 39(3):930--945, 1993.

\bibitem{belkin2019reconciling}
Mikhail Belkin, Daniel Hsu, Siyuan Ma, and Soumik Mandal.
\newblock Reconciling modern machine-learning practice and the classical bias--variance trade-off.
\newblock {\em Proceedings of the National Academy of Sciences}, 116(32):15849--15854, 2019.

\bibitem{bengio2017deep}
Yoshua Bengio, Ian Goodfellow, and Aaron Courville.
\newblock {\em Deep learning}, volume~1.
\newblock MIT press Cambridge, MA, USA, 2017.

\bibitem{berner2020analysis}
Julius Berner, Philipp Grohs, and Arnulf Jentzen.
\newblock Analysis of the generalization error: Empirical risk minimization over deep artificial neural networks overcomes the curse of dimensionality in the numerical approximation of black--scholes partial differential equations.
\newblock {\em SIAM Journal on Mathematics of Data Science}, 2(3):631--657, 2020.

\bibitem{berner2021modern}
Julius Berner, Philipp Grohs, Gitta Kutyniok, and Philipp Petersen.
\newblock The modern mathematics of deep learning.
\newblock {\em arXiv preprint arXiv:2105.04026}, pages 86--114, 2021.

\bibitem{bertoluzza2012primer}
Silvia Bertoluzza, Ricardo~H Nochetto, Alfio Quarteroni, Kunibert~G Siebert, Andreas Veeser, Ricardo~H Nochetto, and Andreas Veeser.
\newblock Primer of adaptive finite element methods.
\newblock {\em Multiscale and Adaptivity: Modeling, Numerics and Applications: CIME Summer School, Cetraro, Italy 2009, Editors: Giovanni Naldi, Giovanni Russo}, pages 125--225, 2012.

\bibitem{bethune2022pay}
Louis B{\'e}thune, Thibaut Boissin, Mathieu Serrurier, Franck Mamalet, Corentin Friedrich, and Alberto Gonzalez~Sanz.
\newblock Pay attention to your loss: understanding misconceptions about lipschitz neural networks.
\newblock {\em Advances in Neural Information Processing Systems}, 35:20077--20091, 2022.

\bibitem{bohte2000spikeprop}
Sander~M Bohte, Joost~N Kok, and Johannes~A La~Poutr{\'e}.
\newblock Spikeprop: backpropagation for networks of spiking neurons.
\newblock In {\em ESANN}, volume~48, pages 419--424. Bruges, 2000.

\bibitem{caragea2023neural}
Andrei Caragea, Philipp Petersen, and Felix Voigtlaender.
\newblock Neural network approximation and estimation of classifiers with classification boundary in a {B}arron class.
\newblock {\em The Annals of Applied Probability}, 33(4):3039--3079, 2023.

\bibitem{comsa2020temporal}
Iulia~M Comsa, Krzysztof Potempa, Luca Versari, Thomas Fischbacher, Andrea Gesmundo, and Jyrki Alakuijala.
\newblock Temporal coding in spiking neural networks with alpha synaptic function.
\newblock In {\em ICASSP 2020-2020 IEEE International Conference on Acoustics, Speech and Signal Processing (ICASSP)}, pages 8529--8533. IEEE, 2020.

\bibitem{cormen2022introduction}
Thomas~H Cormen, Charles~E Leiserson, Ronald~L Rivest, and Clifford Stein.
\newblock {\em Introduction to algorithms}.
\newblock MIT press, 2022.

\bibitem{cybenko1989approximation}
George Cybenko.
\newblock Approximation by superpositions of a sigmoidal function.
\newblock {\em Mathematics of control, signals and systems}, 2(4):303--314, 1989.

\bibitem{davidson2021comparison}
Simon Davidson and Steve~B Furber.
\newblock Comparison of artificial and spiking neural networks on digital hardware.
\newblock {\em Frontiers in Neuroscience}, 15:651141, 2021.

\bibitem{davies2021advancing}
Alex Davies, Petar Veli{\v{c}}kovi{\'c}, Lars Buesing, Sam Blackwell, Daniel Zheng, Nenad Toma{\v{s}}ev, Richard Tanburn, Peter Battaglia, Charles Blundell, Andr{\'a}s Juh{\'a}sz, et~al.
\newblock Advancing mathematics by guiding human intuition with ai.
\newblock {\em Nature}, 600(7887):70--74, 2021.

\bibitem{dold2022neuro}
Dominik Dold, Josep Soler~Garrido, Victor Caceres~Chian, Marcel Hildebrandt, and Thomas Runkler.
\newblock Neuro-symbolic computing with spiking neural networks.
\newblock In {\em Proceedings of the International Conference on Neuromorphic Systems 2022}, pages 1--4, 2022.

\bibitem{wojtowytsch2022representation}
Weinan E and Stephan Wojtowytsch.
\newblock Representation formulas and pointwise properties for {B}arron functions.
\newblock {\em Calculus of Variations and Partial Differential Equations}, 61(2):1--37, 2022.

\bibitem{eshraghian2023training}
Jason~K Eshraghian, Max Ward, Emre~O Neftci, Xinxin Wang, Gregor Lenz, Girish Dwivedi, Mohammed Bennamoun, Doo~Seok Jeong, and Wei~D Lu.
\newblock Training spiking neural networks using lessons from deep learning.
\newblock {\em Proceedings of the IEEE}, 2023.

\bibitem{fitzhugh1961impulses}
Richard FitzHugh.
\newblock Impulses and physiological states in theoretical models of nerve membrane.
\newblock {\em Biophysical journal}, 1(6):445--466, 1961.

\bibitem{frenkel2023bottom}
Charlotte Frenkel, David Bol, and Giacomo Indiveri.
\newblock Bottom-up and top-down approaches for the design of neuromorphic processing systems: tradeoffs and synergies between natural and artificial intelligence.
\newblock {\em Proceedings of the IEEE}, 111(6):623--652, 2023.

\bibitem{gerstner1995time}
Wulfram Gerstner.
\newblock Time structure of the activity in neural network models.
\newblock {\em Physical review E}, 51(1):738, 1995.

\bibitem{gerstner2014neuronal}
Wulfram Gerstner, Werner~M Kistler, Richard Naud, and Liam Paninski.
\newblock {\em Neuronal dynamics: From single neurons to networks and models of cognition}.
\newblock Cambridge University Press, 2014.

\bibitem{gerstner2009good}
Wulfram Gerstner and Richard Naud.
\newblock How good are neuron models?
\newblock {\em Science}, 326(5951):379--380, 2009.

\bibitem{goltz2021fast}
Julian G{\"o}ltz, Laura Kriener, Andreas Baumbach, Sebastian Billaudelle, Oliver Breitwieser, Benjamin Cramer, Dominik Dold, Akos~Ferenc Kungl, Walter Senn, Johannes Schemmel, et~al.
\newblock Fast and energy-efficient neuromorphic deep learning with first-spike times.
\newblock {\em Nature machine intelligence}, 3(9):823--835, 2021.

\bibitem{gouk2021regularisation}
Henry Gouk, Eibe Frank, Bernhard Pfahringer, and Michael~J Cree.
\newblock Regularisation of neural networks by enforcing lipschitz continuity.
\newblock {\em Machine Learning}, 110:393--416, 2021.

\bibitem{grohs2023proof}
Philipp Grohs and Felix Voigtlaender.
\newblock Proof of the theory-to-practice gap in deep learning via sampling complexity bounds for neural network approximation spaces.
\newblock {\em Foundations of Computational Mathematics}, pages 1--59, 2023.

\bibitem{gruning2014spiking}
Andr{\'e} Gr{\"u}ning and Sander~M Bohte.
\newblock Spiking neural networks: Principles and challenges.
\newblock In {\em ESANN}. Bruges, 2014.

\bibitem{he2020relu}
Juncai He, Lin Li, Jinchao Xu, and Chunyue Zheng.
\newblock Relu deep neural networks and linear finite elements.
\newblock {\em Journal of Computational Mathematics}, 38(3):502--527, 2020.

\bibitem{hodgkin1952quantitative}
Alan~L Hodgkin and Andrew~F Huxley.
\newblock A quantitative description of membrane current and its application to conduction and excitation in nerve.
\newblock {\em The Journal of physiology}, 117(4):500, 1952.

\bibitem{hornik1990universal}
Kurt Hornik, Maxwell Stinchcombe, and Halbert White.
\newblock Universal approximation of an unknown mapping and its derivatives using multilayer feedforward networks.
\newblock {\em Neural networks}, 3(5):551--560, 1990.

\bibitem{izhikevich2004model}
Eugene~M Izhikevich.
\newblock Which model to use for cortical spiking neurons?
\newblock {\em IEEE transactions on neural networks}, 15(5):1063--1070, 2004.

\bibitem{jumper2021highly}
John Jumper, Richard Evans, Alexander Pritzel, Tim Green, Michael Figurnov, Olaf Ronneberger, Kathryn Tunyasuvunakool, Russ Bates, Augustin {\v{Z}}{\'\i}dek, Anna Potapenko, et~al.
\newblock Highly accurate protein structure prediction with alphafold.
\newblock {\em Nature}, 596(7873):583--589, 2021.

\bibitem{kheradpisheh2020temporal}
Saeed~Reza Kheradpisheh and Timoth{\'e}e Masquelier.
\newblock Temporal backpropagation for spiking neural networks with one spike per neuron.
\newblock {\em International journal of neural systems}, 30(06):2050027, 2020.

\bibitem{kidger2020universal}
Patrick Kidger and Terry Lyons.
\newblock Universal approximation with deep narrow networks.
\newblock In {\em Conference on learning theory}, pages 2306--2327. PMLR, 2020.

\bibitem{klos2023smooth}
Christian Klos and Raoul-Martin Memmesheimer.
\newblock Smooth exact gradient descent learning in spiking neural networks.
\newblock {\em arXiv preprint arXiv:2309.14523}, 2023.

\bibitem{kriener2022yin}
Laura Kriener, Julian G{\"o}ltz, and Mihai~A Petrovici.
\newblock The yin-yang dataset.
\newblock In {\em Proceedings of the 2022 Annual Neuro-Inspired Computational Elements Conference}, pages 107--111, 2022.

\bibitem{krizhevsky2012imagenet}
Alex Krizhevsky, Ilya Sutskever, and Geoffrey~E Hinton.
\newblock Imagenet classification with deep convolutional neural networks.
\newblock {\em Advances in neural information processing systems}, 25, 2012.

\bibitem{lapicque1907recherches}
L~Lapicque.
\newblock Recherches quantitatives sur l’excitation electrique des nerfs.
\newblock {\em J. Physiol. Paris}, 9:620--635, 1907.

\bibitem{larkum2022guide}
Matthew~E Larkum, Jiameng Wu, Sarah~A Duverdin, and Albert Gidon.
\newblock The guide to dendritic spikes of the mammalian cortex in vitro and in vivo.
\newblock {\em Neuroscience}, 489:15--33, 2022.

\bibitem{lecun2015deep}
Yann LeCun, Yoshua Bengio, and Geoffrey Hinton.
\newblock Deep learning.
\newblock {\em nature}, 521(7553):436--444, 2015.

\bibitem{lerma2024dimension}
Andres~Felipe Lerma-Pineda, Philipp Petersen, Simon Frieder, and Thomas Lukasiewicz.
\newblock Dimension-independent learning rates for high-dimensional classification problems.
\newblock {\em arXiv preprint arXiv:2409.17991}, 2024.

\bibitem{leshno1993multilayer}
Moshe Leshno, Vladimir~Ya Lin, Allan Pinkus, and Shimon Schocken.
\newblock Multilayer feedforward networks with a nonpolynomial activation function can approximate any function.
\newblock {\em Neural networks}, 6(6):861--867, 1993.

\bibitem{luccioni2024light}
Sasha Luccioni, Boris Gamazaychikov, Sara Hooker, R{\'e}gis Pierrard, Emma Strubell, Yacine Jernite, and Carole-Jean Wu.
\newblock Light bulbs have energy ratings—so why can’t ai chatbots?
\newblock {\em Nature}, 632(8026):736--738, 2024.

\bibitem{luccioni2024power}
Sasha Luccioni, Yacine Jernite, and Emma Strubell.
\newblock Power hungry processing: Watts driving the cost of ai deployment?
\newblock In {\em The 2024 ACM Conference on Fairness, Accountability, and Transparency}, pages 85--99, 2024.

\bibitem{lunghi2024investigation}
P~Lunghi, S~Silvestrini, G~Meoni, D~Dold, A~Hadjiivanov, D~Izzo, et~al.
\newblock Investigation of low-energy spiking neural networks based on temporal coding for scene classification.
\newblock In {\em 75th International Astronautical Congress (IAC 2024)}, pages 1--13, 2024.

\bibitem{maass1996lower}
Wolfgang Maass.
\newblock Lower bounds for the computational power of networks of spiking neurons.
\newblock {\em Neural computation}, 8(1):1--40, 1996.

\bibitem{maass1997networks}
Wolfgang Maass.
\newblock Networks of spiking neurons: the third generation of neural network models.
\newblock {\em Neural networks}, 10(9):1659--1671, 1997.

\bibitem{maas1997noisy}
Wolfgang Maass.
\newblock Noisy spiking neurons with temporal coding have more computational power than sigmoidal neurons.
\newblock {\em Advances in Neural Information Processing Systems}, 9:211--217, 1997.

\bibitem{maass2015spike}
Wolfgang Maass.
\newblock To spike or not to spike: that is the question.
\newblock {\em Proceedings of the IEEE}, 103(12):2219--2224, 2015.

\bibitem{maass1997complexity}
Wolfgang Maass and Michael Schmitt.
\newblock On the complexity of learning for a spiking neuron.
\newblock In {\em Proceedings of the tenth annual conference on Computational learning theory}, pages 54--61, 1997.

\bibitem{maass1999complexity}
Wolfgang Maass and Michael Schmitt.
\newblock On the complexity of learning for spiking neurons with temporal coding.
\newblock {\em Information and Computation}, 153(1):26--46, 1999.

\bibitem{mohri2018foundations}
Mehryar Mohri, Afshin Rostamizadeh, and Ameet Talwalkar.
\newblock {\em Foundations of machine learning}.
\newblock MIT press, 2018.

\bibitem{mostafa2017supervised}
Hesham Mostafa.
\newblock Supervised learning based on temporal coding in spiking neural networks.
\newblock {\em IEEE transactions on neural networks and learning systems}, 29(7):3227--3235, 2017.

\bibitem{nagumo1962active}
Jinichi Nagumo, Suguru Arimoto, and Shuji Yoshizawa.
\newblock An active pulse transmission line simulating nerve axon.
\newblock {\em Proceedings of the IRE}, 50(10):2061--2070, 1962.

\bibitem{neftci2019surrogate}
Emre~O Neftci, Hesham Mostafa, and Friedemann Zenke.
\newblock Surrogate gradient learning in spiking neural networks: Bringing the power of gradient-based optimization to spiking neural networks.
\newblock {\em IEEE Signal Processing Magazine}, 36(6):51--63, 2019.

\bibitem{openai2023gpt}
OpenAI.
\newblock {GPT}-4 technical report.
\newblock {\em arXiv preprint 2303.0877}, 2023.

\bibitem{opschoor2020deep}
Joost~AA Opschoor, Philipp~C Petersen, and Christoph Schwab.
\newblock Deep {ReLU} networks and high-order finite element methods.
\newblock {\em Analysis and Applications}, 18(05):715--770, 2020.

\bibitem{parhi2022near}
Rahul Parhi and Robert~D Nowak.
\newblock Near-minimax optimal estimation with shallow relu neural networks.
\newblock {\em IEEE Transactions on Information Theory}, 69(2):1125--1140, 2022.

\bibitem{paugam2006spiking}
H{\'e}lene Paugam-Moisy.
\newblock Spiking neuron networks a survey.
\newblock 2006.

\bibitem{payeur2021burst}
Alexandre Payeur, Jordan Guerguiev, Friedemann Zenke, Blake~A Richards, and Richard Naud.
\newblock Burst-dependent synaptic plasticity can coordinate learning in hierarchical circuits.
\newblock {\em Nature neuroscience}, 24(7):1010--1019, 2021.

\bibitem{petersen2021topological}
Philipp Petersen, Mones Raslan, and Felix Voigtlaender.
\newblock Topological properties of the set of functions generated by neural networks of fixed size.
\newblock {\em Foundations of computational mathematics}, 21:375--444, 2021.

\bibitem{Petersen2024Deep}
Philipp Petersen and Jakob Zech.
\newblock Mathematical theory of deep learning, 2024.

\bibitem{rossbroich2022fluctuation}
Julian Rossbroich, Julia Gygax, and Friedemann Zenke.
\newblock Fluctuation-driven initialization for spiking neural network training.
\newblock {\em Neuromorphic Computing and Engineering}, 2(4):044016, 2022.

\bibitem{safran2024many}
Itay Safran, Daniel Reichman, and Paul Valiant.
\newblock How many neurons does it take to approximate the maximum?
\newblock In {\em Proceedings of the 2024 Annual ACM-SIAM Symposium on Discrete Algorithms (SODA)}, pages 3156--3183. SIAM, 2024.

\bibitem{schmidt2020nonparametric}
Johannes Schmidt-Hieber.
\newblock Nonparametric regression using deep neural networks with relu activation function.
\newblock {\em Annals of statistics}, 48(4):1875--1897, 2020.

\bibitem{schmitt1999vc}
Michael Schmitt.
\newblock {VC} dimension bounds for networks of spiking neurons.
\newblock In {\em ESANN}, pages 429--434, 1999.

\bibitem{schuman2017survey}
Catherine~D Schuman, Thomas~E Potok, Robert~M Patton, J~Douglas Birdwell, Mark~E Dean, Garrett~S Rose, and James~S Plank.
\newblock A survey of neuromorphic computing and neural networks in hardware.
\newblock {\em arXiv preprint arXiv:1705.06963}, 2017.

\bibitem{shalev2014understanding}
Shai Shalev-Shwartz and Shai Ben-David.
\newblock {\em Understanding machine learning: From theory to algorithms}.
\newblock Cambridge university press, 2014.

\bibitem{shen2022optimal}
Zuowei Shen, Haizhao Yang, and Shijun Zhang.
\newblock Optimal approximation rate of {ReLU} networks in terms of width and depth.
\newblock {\em Journal de Math{\'e}matiques Pures et Appliqu{\'e}es}, 157:101--135, 2022.

\bibitem{siegel2022high}
Jonathan~W Siegel and Jinchao Xu.
\newblock High-order approximation rates for shallow neural networks with cosine and {ReLU}$^k$ activation functions.
\newblock {\em Applied and Computational Harmonic Analysis}, 58:1--26, 2022.

\bibitem{singh2023expressivity}
Manjot Singh, Adalbert Fono, and Gitta Kutyniok.
\newblock Expressivity of spiking neural networks.
\newblock {\em arXiv preprint arXiv:2308.08218}, 2023.

\bibitem{stanojevic2023exact}
Ana Stanojevic, Stanis{\l}aw Wo{\'z}niak, Guillaume Bellec, Giovanni Cherubini, Angeliki Pantazi, and Wulfram Gerstner.
\newblock An exact mapping from relu networks to spiking neural networks.
\newblock {\em Neural Networks}, 168:74--88, 2023.

\bibitem{stanojevic2024high}
Ana Stanojevic, Stanis{\l}aw Wo{\'z}niak, Guillaume Bellec, Giovanni Cherubini, Angeliki Pantazi, and Wulfram Gerstner.
\newblock High-performance deep spiking neural networks with 0.3 spikes per neuron.
\newblock {\em Nature Communications}, 15(1):6793, 2024.

\bibitem{thompson2021deep}
Neil~C Thompson, Kristjan Greenewald, Keeheon Lee, and Gabriel~F Manso.
\newblock Deep learning's diminishing returns: The cost of improvement is becoming unsustainable.
\newblock {\em ieee Spectrum}, 58(10):50--55, 2021.

\bibitem{thorpe1996speed}
Simon Thorpe, Denis Fize, and Catherine Marlot.
\newblock Speed of processing in the human visual system.
\newblock {\em nature}, 381(6582):520--522, 1996.

\bibitem{virmaux2018lipschitz}
Aladin Virmaux and Kevin Scaman.
\newblock Lipschitz regularity of deep neural networks: analysis and efficient estimation.
\newblock {\em Advances in Neural Information Processing Systems}, 31, 2018.

\bibitem{xie2024neuronal}
Weizhen Xie, John~H Wittig~Jr, Julio~I Chapeton, Mostafa El-Kalliny, Samantha~N Jackson, Sara~K Inati, and Kareem~A Zaghloul.
\newblock Neuronal sequences in population bursts encode information in human cortex.
\newblock {\em Nature}, pages 1--8, 2024.

\bibitem{yarotsky2017error}
Dmitry Yarotsky.
\newblock Error bounds for approximations with deep {ReLU} networks.
\newblock {\em Neural Networks}, 94:103--114, 2017.

\bibitem{yin2021accurate}
Bojian Yin, Federico Corradi, and Sander~M Boht{\'e}.
\newblock Accurate and efficient time-domain classification with adaptive spiking recurrent neural networks.
\newblock {\em Nature Machine Intelligence}, 3(10):905--913, 2021.

\bibitem{zenke2021visualizing}
Friedemann Zenke, Sander~M Boht{\'e}, Claudia Clopath, Iulia~M Com{\c{s}}a, Julian G{\"o}ltz, Wolfgang Maass, Timoth{\'e}e Masquelier, Richard Naud, Emre~O Neftci, Mihai~A Petrovici, et~al.
\newblock Visualizing a joint future of neuroscience and neuromorphic engineering.
\newblock {\em Neuron}, 109(4):571--575, 2021.

\bibitem{zhang2021understanding}
Chiyuan Zhang, Samy Bengio, Moritz Hardt, Benjamin Recht, and Oriol Vinyals.
\newblock Understanding deep learning (still) requires rethinking generalization.
\newblock {\em Communications of the ACM}, 64(3):107--115, 2021.

\end{thebibliography}

\newpage
\appendix

\section{Proofs}


\subsection{Proof of Lemma~\ref{lem:well-definedFiringTimes}} \label{appx:well-definedFiringTimes}

Fixing $v\in V\setminus V_{\rm in}$, we first assume $t_u$ exists whenever $(u,v)\in E$.
Then the potential at $v$ is given by \eqref{eqdef:accumulation}
\begin{align*}
    P_v(t) = \sum_{(u,v)\in E} \mathsf{w}_{(u,v)}\varrho(t-t_u-d_{(u,v)}),
\end{align*}
for $t\in\R$, where, by Definition~\ref{def:SNN}, at least one $\mathsf{w}_{(u,v)}>0$.
It follows that $P_v$ crosses the threshold value $1$ at a unique time $t_v$. 
The proof can now be concluded via an induction pattern. 
Namely, since $G$ is directed acyclic by assumption, there exists a topological ordering of the nodes, $(v_1, \dots, v_{\#V})$ such that $(v_i, v_j) \not \in E$ for all $j \leq i$, see \cite[Section 20.4]{cormen2022introduction}. 
Since the spike time is given for each input neuron, we deduce that $P_v$ is well-defined for all $v \in V \setminus V_{\rm in }$ by induction over the nodes in the topological order. \qed

\subsection{Proof of Theorem~\ref{thm:LipaffineSNN}} \label{sec:LipaffineSNNpf}

We start with the following auxiliary lemma.

\begin{lemma} \label{lem:LipLem1} 
Let $\Phi=(G,\mathsf{W},D)$ be a positive SNN. 
Let $v\in V\setminus V_{\rm in}$. 
Let $(t_u)_{(u,v)\in E}$, $(\tilde{t}_u)_{(u,v)\in E}$ denote two different vectors of spike times from all the presynaptic neurons $u$ to $v$, and let $t_v\big((t_u)_{(u,v)\in E}\big)$, $t_v\big((\tilde{t}_u)_{(u,v)\in E}\big)$ denote the respective, corresponding spike times at $v$. 
Then it holds that
\begin{align} \label{eq:LipLem1conc}
	|t_v\big((t_u)_{(u,v)\in E}\big) - t_v\big((\tilde{t}_u)_{(u,v)\in E}\big)| \leq \|(t_u)_{(u,v)\in E} - (\tilde{t}_u)_{(u,v)\in E}\|_{\ell^{\infty}}.
\end{align}
\end{lemma}

\begin{proof}
We explicitly indicate the dependence of the potential $P_v$ on the received spike times $(t_u)_{(u,v)\in E}$ by writing
\begin{align*}
	P_v(t; (t_u)_{(u,v)\in E}) = \sum_{(u,v)\in E} \mathsf{w}_{(u,v)}\varrho(t-t_u-d_{(u,v)}).
\end{align*}
Suppose $\tilde{t}_u\geq t_u$ for all $(u,v)\in E$. 
Then it is evident from the construction that $P_v$ satisfies
\begin{align}\label{eq:strangeMonotonicity}
	P_v(t; (t_u)_{(u,v)\in E}) \geq P_v(t; (\tilde{t}_u)_{(u,v)\in E}).
\end{align}
Furthermore, by definition, $t_v\big((t_u)_{(u,v)\in E}\big)$, $t_v\big((\tilde{t}_u)_{(u,v)\in E}\big)$, are determined as the unique times at which
\begin{align} \label{uniquefiring}
	\begin{split}
	    P_v(t_v\big((t_u)_{(u,v)\in E}\big); (t_u)_{(u,v)\in E}) &= 1, \\
	    P_v(t_v\big((\tilde{t}_u)_{(u,v)\in E}\big); (\tilde{t}_u)_{(u,v)\in E}) &= 1.
	\end{split}
\end{align}
Set $c=\|(t_u)_{(u,v)\in E} - (\tilde{t}_u)_{(u,v)\in E}\|_{\ell^{\infty}}$. 
It follows from the monotonicity depicted in \eqref{eq:strangeMonotonicity} that, 
\begin{align} \label{monotonecons}
	\begin{split}
	    P_v(t + c; (\tilde{t}_u)_{(u,v)\in E}) &= P_v(t; (\tilde{t}_u - c)_{(u,v)\in E})\geq P_v(t; (t_u)_{(u,v)\in E}),\\
	    P_v(t - c; (\tilde{t}_u)_{(u,v)\in E}) &= P_v(t; (\tilde{t}_u + c)_{(u,v)\in E}) \leq P_v(t; (t_u)_{(u,v)\in E}).
	\end{split}
\end{align}
In particular, by substituting $t=t_v\big((t_u)_{(u,v)\in E}\big)$ in \eqref{monotonecons}, and applying \eqref{uniquefiring}, we obtain
\begin{align*}
	P_v(t_v\big((t_u)_{(u,v)\in E}\big) - c; (\tilde{t}_u)_{(u,v)\in E}) &\leq 1 \\
	&= P_v(t_v\big((\tilde{t}_u)_{(u,v)\in E}\big); (\tilde{t}_u)_{(u,v)\in E}) \\
	&\leq P_v(t_v\big((t_u)_{(u,v)\in E}\big) + c; (\tilde{t}_u)_{(u,v)\in E}).
\end{align*}
By invoking the continuity of $P_v$ with respect to time and the intermediate value theorem, along with the uniqueness of spike times, we obtain
\begin{align*}
	t_v\big((\tilde{t}_u)_{(u,v)\in E}\big) \in [t_v\big((t_u)_{(u,v)\in E}\big) - c, t_v\big((t_u)_{(u,v)\in E}\big) + c],
\end{align*}
which readily implies \eqref{eq:LipLem1conc}, when $\tilde{t}_u\geq t_u$ for all $(u,v)\in E$.
For the general case, we define
\begin{align*}
	(t'_u)_{(u,v)\in E} \coloneqq (t_u + c)_{(u,v)\in E} 
	\quad\text{ and }\quad 
	(t''_u)_{(u,v)\in E} \coloneqq (t_u - c)_{(u,v)\in E}.
\end{align*}
Evidently, $t''_u\leq t_u, \tilde{t}_u\leq t'_u$. 
Then by applying the previous reasoning, we deduce that 
\begin{align} \label{gencase1}
	 t_v\big((t''_u)_{(u,v)\in E}\big)\leq t_v\big((t_u)_{(u,v)\in E}\big),\, t_v\big((\tilde{t}_u)_{(u,v)\in E}\big) \leq t_v\big((t'_u)_{(u,v)\in E}\big),
\end{align}
and that 
\begin{align} \label{gencase2}
	\begin{split}
	    |t_v\big((t_u)_{(u,v)\in E}\big) - t_v\big((t'_u)_{(u,v)\in E}\big)| &\leq c, \\
	    |t_v\big((t_u)_{(u,v)\in E}\big) - t_v\big((t''_u)_{(u,v)\in E}\big)| &\leq c.
	\end{split}
\end{align}
Hence, together, \eqref{gencase1}, \eqref{gencase2} imply
\begin{align*}
	|t_v\big((t_u)_{(u,v)\in E}\big) - t_v\big((\tilde{t}_u)_{(u,v)\in E}\big)| \leq c = \|(t_u)_{(u,v)\in E} - (\tilde{t}_u)_{(u,v)\in E}\|_{\ell^{\infty}},
\end{align*}
as desired.
\end{proof}

The proof of Theorem~\ref{thm:LipaffineSNN} is a direct application of this lemma.

\begin{proof}[Proof of Theorem~\ref{thm:LipaffineSNN}]
Let $(t_u)_{u\in V_{\rm in}}$, $(\tilde{t}_u)_{u\in V_{\rm in}}$ be two tuples of input spike times for $\Phi$.
For each $v\in V\setminus V_{\rm in}$, we denote by $t_v$, $\tilde{t}_v$, the corresponding spike times in response to $(t_u)_{u\in V_{\rm in}}$, $(\tilde{t}_u)_{u\in V_{\rm in}}$, respectively.
Define $S(v)\coloneqq\{u\colon (u,v)\in E\}$. 
Let $v_1\in S(v)$ be such that $\|(t_u)_{u\in S(v)}-(\tilde{t}_u)_{u\in S(v)}\|_{\ell^{\infty}} = |t_{v_1}-\tilde{t}_{v_1}|$.
Then it follows from Lemma~\ref{lem:LipLem1} that
\begin{align*} 
	|t_v - \tilde{t}_v|\leq |t_{v_1}-\tilde{t}_{v_1}|.
\end{align*}
Iterating this process, and letting $v_0\coloneqq v$, we move to select $v_{i+1}\in S(v_i)$ satisfying $\|(t_u)_{u\in S(v_i)}-(\tilde{t}_u)_{u\in S(v_i)}\|_{\ell^{\infty}} = |t_{v_{i+1}}-\tilde{t}_{v_{i+1}}|$. 
Note that, as $G$ is a finite graph, this process yields a directed path of synaptic edges that starts at some $v_N \in V_{\rm in}$ and ends in $v_0$: $(v_N, \dots, v_0)$.
Furthermore, for each $i=0,\dots,N-1$, it holds that
\begin{align*} 
	|t_{v_i} - \tilde{t}_{v_i}|
	\leq |t_{v_{i+1}}-\tilde{t}_{v_{i+1}}|.
\end{align*}
Therefore,
\begin{align*}
	|t_v - \tilde{t}_v| 
	= |t_{v_0} - \tilde{t}_{v_0}| 
	\leq |t_{v_N} - t_{v_N}|\leq \|(t_u)_{u\in V_{\rm in}} -(\tilde{t}_u)_{u\in V_{\rm in}}\|_{\ell^{\infty}}.
\end{align*}
Choosing $v\in V_{\rm out}$ arbitrary, we conclude that
\begin{multline} \label{eq:innerSNN}
\|\Realization(\Phi)\big((t_u)_{u\in V_{\rm in}}\big) -\Realization(\Phi)\big((\tilde{t}_u)_{u\in V_{\rm in}}\big)\|_{\ell^{\infty}} \\
= \|(t_v)_{v\in V_{\rm out}} -(\tilde{t}_v)_{v\in V_{\rm out}}\|_{\ell^{\infty}} 
\leq \|(t_u)_{u\in V_{\rm in}} -(\tilde{t}_u)_{u\in V_{\rm in}}\|_{\ell^{\infty}}.
\end{multline}
Next, let $(t_u)_{u\in V_{\rm in}} = A_{\rm in}(x)$, $(\tilde{t}_u)_{u\in V_{\rm in}} = A_{\rm in}(\tilde{x})$, for some $x,\tilde{x}\in\R^{{\rm d}_0}$. 
Then
\begin{multline} \label{eq:inputend}
\|(t_u)_{u\in V_{\rm in}} -(\tilde{t}_u)_{u\in V_{\rm in}}\|_{\ell^{\infty}} 
\leq \|(t_u)_{u\in V_{\rm in}} -(\tilde{t}_u)_{u\in V_{\rm in}}\|_{\ell^2} \\
\leq \|W_{\rm in}\|_F\|x-\tilde{x}\|_{\ell^2}
\leq {\rm d}_0^{\frac{1}{2}}\|W_{\rm in}\|_F\|x-\tilde{x}\|_{\ell^{\infty}}.
\end{multline}
In a similar vein,
\begin{align} \label{eq:outputend}
	\begin{split}
	    \|A_{\rm out}\big((t_v)_{v\in V_{\rm out}}\big) - A_{\rm out} \big((\tilde{t}_v)_{v\in V_{\rm out}}\big) &\|_{\ell^{\infty}}\\
	    &\leq \|A_{\rm out} \big((t_v)_{v\in V_{\rm out}}\big) - A_{\rm out} \big((\tilde{t}_v)_{v\in V_{\rm out}}\big) \|_{\ell^2}\\
	    &\leq \|W_{\rm out}\|_F\|(t_v)_{v\in V_{\rm out}} -(\tilde{t}_v)_{v\in V_{\rm out}}\|_{\ell^2} \\
	    &\leq {\rm d}_{\rm out}^{\frac{1}{2}}\|W_{\rm out}\|_F\|(t_v)_{v\in V_{\rm out}} -(\tilde{t}_v)_{v\in V_{\rm out}}\|_{\ell^{\infty}}.
	\end{split} 
\end{align}
Combining \eqref{eq:innerSNN}, \eqref{eq:inputend}, \eqref{eq:outputend}, we arrive at the conclusion of the theorem. 
\end{proof}

\subsection{Proof of Proposition~\ref{prop:LipPropPhi}} \label{appx:LipPropPhi}

We begin by establishing a series of helper lemmas.
For the forthcoming analysis, we recall that all SNNs considered share the same network graph.

\begin{lemma} \label{lem:LipLemma2} 
Let $\Phi=(G,\mathsf{W},D)$, $\widetilde{\Phi}=(G,\mathsf{W},\widetilde{D})$ be two positive SNNs. 
Let $v\in V\setminus V_{\rm in}$. 
Suppose $D$, $\widetilde{D}$ only differ in the synaptic delays from all the presynaptic neurons $u$ to $v$, which are $(d_{(u,v)})_{(u,v)\in E}$, $(\tilde{d}_{(u,v)})_{(u,v)\in E}$, respectively. 
Let $t_v\big((d_{(u,v)})_{(u,v)\in E}\big)$, $t_v\big((\tilde{d}_{(u,v)})_{(u,v)\in E}\big)$ denote the respective, corresponding spike times at $v$ in $\Phi$, $\widetilde{\Phi}$.
Then
\begin{align} \label{eq:LipLem2conc}
	|t_v\big((d_{(u,v)})_{(u,v)\in E}\big) - t_v\big((\tilde{d}_{(u,v)})_{(u,v)\in E}\big)| \leq \|(d_{(u,v)})_{(u,v)\in E} - (\tilde{d}_{(u,v)})_{(u,v)\in E}\|_{\ell^{\infty}}.
\end{align}
\end{lemma}

\begin{proof}
An important observation to make is that, 
since $D$ and $\widetilde{D}$ only vary in the synaptic delays from all presynaptic neurons $u$ to $v$, the spike times at the neurons $u$ with $(u,v) \in E$ remain identical for a given set of network input spike times $(t_u)_{u\in V_{\rm in}}$.
We denote these spike times coming to $v$ as $(t_u)_{(u,v)\in E}$.
The corresponding spike times at $v$, $t_v\big((d_{(u,v)})_{(u,v)\in E}\big)$, $t_v\big((\tilde{d}_{(u,v)})_{(u,v)\in E}\big)$, in $\Phi$, $\widetilde{\Phi}$ are then the unique times at which 
\begin{align*}
	\sum_{(u,v)\in E} \mathsf{w}_{(u,v)}\varrho\Big(t_v\big((d_{(u,v)})_{(u,v)\in E}\big)-t_u-d_{(u,v)}\Big) &= 1,\\
	\sum_{(u,v)\in E} \mathsf{w}_{(u,v)}\varrho\Big(t_v\big((\tilde{d}_{(u,v)})_{(u,v)\in E}\big)-t_u-\tilde{d}_{(u,v)}\Big) &= 1,
\end{align*}
respectively. 
It remains to notice that, from this point onward, by employing a similar argument to that used in the proof of Lemma~\ref{lem:LipLem1}, we can establish \eqref{eq:LipLem2conc}, as desired. 
\end{proof}

\begin{lemma} \label{lem:LipLemma3}
Let $\Phi=(G,\mathsf{W},D)$, $\widetilde{\Phi}=(G,\widetilde{\mathsf{W}},D)$ be two positive SNNs. 
Let $v\in V\setminus V_{\rm in}$. 
Suppose that there exists $\mathsf{b}>0$, such that for every $\mathsf{w}_{(u,v)}\in\mathsf{W}$, $\tilde{\mathsf{w}}_{(u,v)}\in\widetilde{\mathsf{W}}$, 
\begin{align} \label{eq:weightlowerbd}
	\min\big\{\mathsf{w}_{(u,v)}, \tilde{\mathsf{w}}_{(u,v)}\big\} \geq\mathsf{b}.
\end{align}
Furthermore, suppose $\mathsf{W}$, $\widetilde{\mathsf{W}}$ only differ in the synaptic weights from all the presynaptic neurons $u$ to $v$, which are $(\mathsf{w}_{(u,v)})_{(u,v)\in E}$, $(\tilde{\mathsf{w}}_{(u,v)})_{(u,v)\in E}$, respectively. 
Let $t_v\big((\mathsf{w}_{(u,v)})_{(u,v)\in E}\big)$, $t_v\big((\tilde{\mathsf{w}}_{(u,v)})_{(u,v)\in E}\big)$ denote the respective, corresponding spike times at $v$ in $\Phi$, $\widetilde{\Phi}$. 
Then
\begin{align} \label{eq:LipLem3conc}
	|t_v\big((\mathsf{w}_{(u,v)})_{(u,v)\in E}\big) - t_v\big((\tilde{\mathsf{w}}_{(u,v)})_{(u,v)\in E}\big)| \leq \frac{\|(\mathsf{w}_{(u,v)})_{(u,v)\in E} - (\tilde{\mathsf{w}}_{(u,v)})_{(u,v)\in E}\|_{\ell^{\infty}}}{\mathsf{b}^2}.
\end{align}
\end{lemma}

\begin{proof}
As with the proof of Lemma~\ref{lem:LipLemma2}, considering that 
$\mathsf{W}$, $\widetilde{\mathsf{W}}$ only differ in the synaptic weights from all the neurons $u$ presynaptic to $v$, we conclude that the spike times at these neurons, denoted by $(t_u)_{(u,v)\in E}$, are the same for a given set of network input spike times $(t_u)_{u\in V_{\rm in}}$.
Without loss of generality, we assume $d_{(u,v)}=0$ for all $(u,v)\in E$; otherwise, we can redefine $t_{u}$ to $t_{u}+d_{(u,v)}$. 
Next, we denote $t^{\star} \coloneqq t_v\big((\mathsf{w}_{(u,v)})_{(u,v)\in E}\big)$, $\tilde{t}^{\star} \coloneqq t_v\big((\tilde{\mathsf{w}}_{(u,v)})_{(u,v)\in E}\big)$, the unique times at which 
\begin{align} \label{eq:Qrecall}
    \sum_{(u,v)\in E} \mathsf{w}_{(u,v)}\varrho(t^{\star}-t_u) = 1 \quad\text{ and }\quad
	\sum_{(u,v)\in E} \tilde{\mathsf{w}}_{(u,v)}\varrho(\tilde{t}^{\star}-t_u) = 1,
\end{align}
respectively.
We order the presynaptic neurons according to their associated arrival times $t_u$; namely, $\{u\colon (u,v)\in E\} = \{u_1,\dots, u_N\}$, where $t_{u_1}\leq\cdots\leq t_{u_N}$, and rewrite \eqref{eq:Qrecall} as
\begin{align*} 
	\sum_{i=1}^N \mathsf{w}_{(u_i,v)}\varrho(t^{\star}-t_{u_i}) = 1 \quad\text{ and }\quad
	\sum_{i=1}^N \tilde{\mathsf{w}}_{(u_i,v)}\varrho(\tilde{t}^{\star}-t_{u_i}) = 1.
\end{align*}
Suppose $\tilde{\mathsf{w}}_{(u_i,v)}\geq \mathsf{w}_{(u_i,v)}$, for all $i=1,\dots,N$. 
It follows that $\tilde{t}^{\star}\leq t^{\star}$.
Let $1\leq J\leq N$ be the largest index for which $t^{\star}\geq t_{u_J}$, and $1\leq \tilde{J}\leq N$ the largest index for which $\tilde{t}^{\star}\geq t_{u_{\tilde{J}}}$. 
Then $\tilde{J}\leq J$. 
Respectively, the times $t^{\star}$, $\tilde{t}^{\star}$ are dictated by
\begin{align} \label{eq:Qrewrite}
	\sum_{i=1}^J \mathsf{w}_{(u_i,v)} (t^{\star}-t_{u_i}) = 1 \quad\text{ and }\quad
	\sum_{i=1}^{\tilde{J}} \tilde{\mathsf{w}}_{(u_i,v)}(\tilde{t}^{\star}-t_{u_i}) = 1.
\end{align}
Hence, 
\begin{align*}
    0\leq \sum_{i=1}^{\tilde{J}} \tilde{\mathsf{w}}_{(u_i,v)}(t^{\star}- \tilde{t}^{\star}) = \sum_{i=1}^{\tilde{J}} (\tilde{\mathsf{w}}_{(u_i,v)} - \mathsf{w}_{(u_i,v)})(t^{\star}-t_{u_i}) - \sum_{\tilde{J}+1}^J \mathsf{w}_{(u_i,v)}(t^{\star} -t_{u_i}),
\end{align*}
which implies
\begin{align} \label{eq:step1}
	\sum_{i=1}^{\tilde{J}} \tilde{\mathsf{w}}_{(u_i,v)}|t^{\star}- \tilde{t}^{\star}|\leq \sum_{i=1}^{\tilde{J}} |\tilde{\mathsf{w}}_{(u_i,v)} - \mathsf{w}_{(u_i,v)}||t^{\star}-t_{u_i}|.
\end{align}
On the other hand, it follows from \eqref{eq:Qrewrite} that $\max\{|t^{\star}-t_{u_i}|\colon i=1,\dots,\tilde{J}\}\leq\mathsf{b}^{-1}$, where $\mathsf{b}$ is as in \eqref{eq:weightlowerbd}. 
Putting this back in \eqref{eq:step1}, we deduce that
\begin{align*}
	|t^{\star}- \tilde{t}^{\star}| \leq \frac{\sum_{i=1}^{\tilde{J}} |\tilde{\mathsf{w}}_{(u_i,v)} - \mathsf{w}_{(u_i,v)}|}{\mathsf{b}\sum_{i=1}^{\tilde{J}} \tilde{\mathsf{w}}_{(u_i,v)}} \leq \frac{\|(\mathsf{w}_{(u,v)})_{(u,v)\in E} - (\tilde{\mathsf{w}}_{(u,v)})_{(u,v)\in E}\|_{\ell^{\infty}}}{\mathsf{b}^2},
\end{align*}
in the case $\tilde{\mathsf{w}}_{(u,v)}\geq\mathsf{w}_{(u,v)}$, for all $(u,v)\in E$, as wanted.
For the general case, we proceed similarly to the proof of Lemma~\ref{lem:LipLem1}. 
Setting $c=\|(\mathsf{w}_{(u,v)})_{(u,v)\in E} - (\tilde{\mathsf{w}}_{(u,v)})_{(u,v)\in E}\|_{\ell^{\infty}}$, we define $\mathsf{w}'_{(u,v)} \coloneqq \mathsf{w}_{(u,v)} + c$, $\mathsf{w}''_{(u,v)} \coloneqq \mathsf{w}_{(u,v)} - c$. 
Then
\begin{align} \label{eq:nesting1}
	\mathsf{w}''_{(u,v)} \leq \mathsf{w}_{(u,v)}, \, \tilde{\mathsf{w}}_{(u,v)}\leq \mathsf{w}'_{(u,v)},
\end{align}
for all $(u,v)\in E$. 
Therefore
\begin{align} \label{eq:nesting2}
	\begin{split}
	    \big| t_v\big((\mathsf{w}_{(u,v)})_{(u,v)\in E}\big) - t_v\big((\mathsf{w}'_{(u,v)})_{(u,v)\in E}\big) \big| &\leq \frac{c}{\mathsf{b}^2}, \\
	    \big| t_v\big((\mathsf{w}_{(u,v)})_{(u,v)\in E}\big) - t_v\big(\mathsf{w}''_{(u,v)})_{(u,v)\in E}\big) \big| &\leq \frac{c}{\mathsf{b}^2}.
	\end{split}
\end{align}
For the final step, we combine \eqref{eq:nesting1}, \eqref{eq:nesting2} to obtain \eqref{eq:LipLem3conc}.
\end{proof}

We are approaching the proof of Proposition~\ref{prop:LipPropPhi}.
Observe, although the inputs of $\Realization(\Phi)$ and $\Realization(\widetilde{\Phi})$ are identical, represented by $t\in\R^{{\rm d}_{\rm in}}$, the spike times at $v\in V\setminus V_{\rm in}$ will vary in both $\Phi$ and $\widetilde{\Phi}$ due to their different network parameters, respectively $(\mathsf{W},D)$ and $(\widetilde{\mathsf{W}},\widetilde{D})$. 
To effectively monitor the changes in spike times as they propagate through $\Phi$ and $\widetilde{\Phi}$, we utilize a graph splitting algorithm that enables us to partition a network graph into a finite sequence of disjoint subgraphs, each having depth $1$. 
The precise statement is provided below.

\begin{lemma} \label{lem:graphsplit}
Let $G=(V,E)$ be a network graph with depth $\mathsf{L}$. 
Then there exists a finite sequence of network subgraphs $(G^i)_{i=1}^{\mathsf{L}}$ of $G$, such that the following hold:
\begin{enumerate}
    \item[(i)] $E(G) = \bigcup_{i=1}^{\mathsf{L}} E(G^i)$;
    \item[(ii)] $E(G^i)\cap E(G^j)=\emptyset$ if $i\neq j$;
    \item[(iii)] the depth of $G^i$ is $1$;
    \item[(iv)] if $v$ is an output node in $G^i$, then all the synapses $(u,v)\in E(G)$ are included in $E(G^i)$;
    \item[(v)] every input node $u$ in $G^i$ is an input node of the network subgraph of $G$ comprising all the edges in $\bigcup_{j=i}^{\mathsf{L}} E(G^j)$.
\end{enumerate}
\end{lemma}

\begin{proof}
First, we partition the vertex set $V=\bigcup_{i=0}^{\mathsf{L}} V^i$ following an algorithm akin to the \textit{longest} path layering algorithm, which we now describe. 
Define $V^i$ to be the collection of vertices whose longest path from an input node has length $i$.
Since $G$ is directed acyclic, $V^i$ is well-defined. 
Evidently, $V^0=V_{\rm in}$, and $i$ can reach up to $\mathsf{L}$, where $V^{\mathsf{L}}$ comprises output nodes whose longest directed path to an input node equals the graph depth $\mathsf{L}$.
Moreover, there are only directed edges from vertices in $V^i$ to vertices in $V^j$, if $i<j$.
With $(V^i)_{i=0}^{\mathsf{L}}$ established, we construct $G^i$, $i=1,\dots,\mathsf{L}$, to be the subgraph of $G$ comprising all incoming edges to vertices in $V^i$. 
We refer to Figures \ref{fig:spliting1}, \ref{fig:spliting2} for a visualization of this graph splitting process.
Then the conditions (i), (ii), (iii), (iv) follow directly from the construction. 
To see (v), we fix an index $i$ and an input node $v$ in $G^i$. 
By way of construction, $v\in V^k$, for some $k<i$. 
Let $\mathcal{G}^i$ denote the network subgraph comprising all the edges in $\bigcup_{j=i}^{\mathsf{L}} E(G^j)$. 
Supposing that $v$ is not an input node of $\mathcal{G}^i$, we can obtain an edge $(u,v)$ in $\mathcal{G}^i$. Then it must be that $(u,v)\in E(G^k)$, posing a contradiction to (ii). 
\end{proof}

With Lemma~\ref{lem:graphsplit} established, we advance to the proof of Proposition~\ref{prop:LipPropPhi}.

\begin{figure}
\begin{minipage}{0.5\textwidth}
    \centering
    \caption{A network graph $G$ with four input nodes in burgundy and three output nodes in violet}
    \label{fig:spliting1}
    \includegraphics[width=0.8\linewidth]{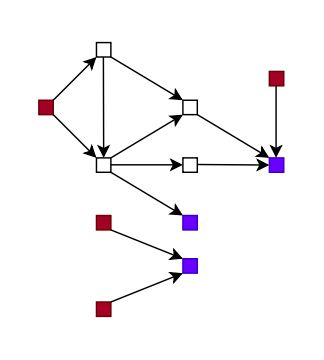}
\end{minipage} \hfill
\begin{minipage}{0.5\textwidth}
    \centering
    \includegraphics[width=1.0\linewidth]{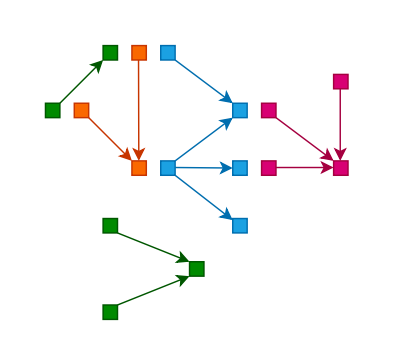}
    \caption{The network graph $G$ split into four network subgraphs with depth of $1$: $G^1$ in green, $G^2$ in orange, $G^3$ in light blue and $G^4$ in pink}
    \label{fig:spliting2}
\end{minipage}
\end{figure}

\begin{proof}[Proof of Proposition~\ref{prop:LipPropPhi}]
Consider the decomposition of $G$ into $G^1,G^2,\dots, G^{\mathsf{L}}$ according to Lemma~\ref{lem:graphsplit}.
Let $\mathsf{W}^i$, $\widetilde{\mathsf{W}}^i$ be the respective restrictions of $\mathsf{W}$, $\widetilde{\mathsf{W}}$ onto $E(G^i)$, and $D^i$, $\widetilde{D}^i$ be the respective restrictions of $D$, $\widetilde{D}$ onto $E(G^i)$.
We construct the following $2\mathsf{L}$ positive SNNs,
\begin{align*} 
    \Phi^i = (G^i, \mathsf{W}^i, D^i) 
    \quad\text{ and }\quad 
    \widetilde{\Phi}^i = (G^i, \widetilde{\mathsf{W}}^i, \widetilde{D}^i),
\end{align*}
for $i=1,\dots,\mathsf{L}$. 
Let $V^i_{\rm in}$ be the set of input nodes of $G^i$ and $V^i_{\rm out}$ be the corresponding set of output nodes.
Let $t\in\R^{{\rm d}_{\rm in}}$ be a tuple of input spike times for the positive SNNs $\Phi$, $\widetilde{\Phi}$. 
For each $i=1,\dots, \mathsf{L}$, we write $t_{V^i_{\rm in}}$, $\tilde{t}_{V^i_{\rm in}}$, to denote the input tuples to $\Phi^i$, $\widetilde{\Phi}^i$, respectively, and $t_{V^i_{\rm out}}$, $\tilde{t}_{V^i_{\rm out}}$ to denote the corresponding output tuples.  
Note that these tuples depend on the given input tuple $t\in\R^{{\rm d}_{\rm in}}$. 
Further,
\begin{multline} \label{eq:secondlayer1}
    \Big\| \Realization(\Phi^i)\big(t_{V^i_{\rm in}}\big) - \Realization(\widetilde{\Phi}^i) \big(\tilde{t}_{V^i_{\rm in}}\big) \Big\|_{\ell^{\infty}} 
    \\
    \leq \Big\|\Realization(\Phi^i)\big(t_{V^i_{\rm in}}\big) - \Realization(\Phi^i)\big(\tilde{t}_{V^i_{\rm in}}\big)\Big\|_{\ell^{\infty}} 
    + \Big\|\Realization(\Phi^i)\big(\tilde{t}_{V^i_{\rm in}}\big) - \Realization(\widetilde{\Phi}^i)\big(\tilde{t}_{V^i_{\rm in}}\big)\Big\|_{\ell^{\infty}}.
\end{multline}
With the graph depth of $G^i$ being $1$ (by Lemma~\ref{lem:graphsplit}), we can apply Lemmas~\ref{lem:LipLemma2},~\ref{lem:LipLemma3} to the second term on the RHS of \eqref{eq:secondlayer1}, resulting in
\begin{align} \label{eq:secondlayer2}
    \Big\| \Realization(\Phi^i)\big(\tilde{t}_{V^i_{\rm in}}\big) - \Realization(\widetilde{\Phi}^i)\big(\tilde{t}_{V^i_{\rm in}}\big) \Big\|_{\ell^{\infty}} 
    \leq 
    \Big(1+\frac{1}{\mathsf{b}^2}\Big)\|\Phi^i - \widetilde{\Phi}^i\|_{\ell^{\infty}}.
\end{align}
For the first term on the RHS of \eqref{eq:secondlayer1}, we use Lemma~\ref{lem:LipLem1} to obtain
\begin{align} \label{eq:secondlayer3}
    \Big\|\Realization(\Phi^i)\big(t_{V^i_{\rm in}}\big) - \Realization(\Phi^i)\big(\tilde{t}_{V^i_{\rm in}}\big)\Big\|_{\ell^{\infty}} 
    \leq 
    \|t_{V^i_{\rm in}} - \tilde{t}_{V^i_{\rm in}}\|_{\ell^{\infty}}.
\end{align}
By construction, an input node of $G^i$ either originates from an output node of $G^j$, for some $j<i$, or an input node of $G$ (see for instance, the graph $G^4$ depicted in Figure~\ref{fig:spliting2}). Subsequently, each component in $t_{V^i_{\rm in}}$ is either a component in $t_{V^j_{\rm out}}$, or a component in the network input $t$. A similar assertion applies to the corresponding components in $\tilde{t}_{V^i_{\rm in}}$.
We can then deduce that the RHS of \eqref{eq:secondlayer3} is bounded above by
\begin{equation*}
    \Big\|t_{V^j_{\rm out}} - \tilde{t}_{V^j_{\rm out}}\Big\|_{\ell^{\infty}} 
    = \Big\|\Realization(\Phi^{i-1})\big(t_{V^j_{\rm in}}\big) - \Realization(\widetilde{\Phi}^{i-1})\big(\tilde{t}_{V^j_{\rm in}}\big)\Big\|_{\ell^{\infty}},
\end{equation*}
for some $j <i$.
However, this quantity resembles the LHS of \eqref{eq:secondlayer1}.
Hence, by replicating the provided reasoning inductively, and combining \eqref{eq:secondlayer1}, \eqref{eq:secondlayer2}, we derive
\begin{equation} \label{eq:thirdlayer}
    \begin{split}
        \Big\| t_{V^i_{\rm out}} - \tilde{t}_{V^i_{\rm out}} \Big\|_{\ell^{\infty}}
        &= \Big\| \Realization(\Phi^i)\big(t_{V^i_{\rm in}}\big) - \Realization(\widetilde{\Phi}^i)\big(\tilde{t}_{V^i_{\rm in}}\big) \Big\|_{\ell^{\infty}} \\
        &\leq \sum_{k=1}^i \Big(1+\frac{1}{\mathsf{b}^2}\Big)\|\Phi^k - \widetilde{\Phi}^k\|_{\ell^{\infty}}.
    \end{split}
\end{equation}
Let $v\in V_{\rm out}$ be an output node of $G$. Then $v$ is an output node in one of the $G^i$, $i=1,\dots, \mathsf{L}$.
By Lemma~\ref{lem:graphsplit}, the edge sets 
$E(G^i)$ are mutually disjoint, and collectively, they exhaust the entire edge set $E(G)$. 
Therefore, from \eqref{eq:thirdlayer} we can infer
\begin{equation} \label{eq:fourthlayer}
    |t_v - \tilde{t}_v| \leq \sum_{i=1}^{\mathsf{L}} \Big(1+\frac{1}{\mathsf{b}^2}\Big)\|\Phi^i - \widetilde{\Phi}^i\|_{\ell^{\infty}} \leq \mathsf{L} \cdot \Big(1+\frac{1}{\mathsf{b}^2}\Big)\|\Phi - \widetilde{\Phi}\|_{\ell^{\infty}},
\end{equation}
where $t_v = t_v(t)$, $\tilde{t}_v = \tilde{t}_v(t)$ denote the output spike times at $v$ in $\Phi$, $\tilde{\Phi}$, respectively, given $t\in\R^{{\rm d}_{\rm in}}$. 
The conclusion \eqref{eq:LipPropconc} now follows from \eqref{eq:fourthlayer}, completing the proof.
\end{proof}

\subsection{Proof of Theorem~\ref{thm:LipThm}
} \label{appx:LipThms}

In the subsequent discussion, the dimension of the zero vector $\vec{0}$ can be inferred from the context.
As usual, we start with a necessary lemma. 

\begin{lemma} \label{lem:SNNoutputmag}
Let $\Phi = (G,\mathsf{W},D)$ be a positive SNN.
Let $\mathsf{b}, \mathsf{B}\in (0,\infty)$ be such that for all $(u,v)\in E$, $\mathsf{w}_{(u,v)}\geq\mathsf{b}$ and $d_{(u,v)}\leq\mathsf{B}$.
Let $\mathsf{L}$ be the graph depth of $G$.
Then 
\begin{align} \label{eq:SNNoutputmagconc}
	\|\Realization(\Phi)(\vec{0})\|_{\ell^{\infty}}\leq \mathsf{L} \cdot \Big(\frac{1}{\mathsf{b}} + \mathsf{B}\Big).
\end{align}
\end{lemma}

\begin{proof} 
To begin, we consider the case where $d_{(u,v)}=0$ for all $(u,v)\in E$, and we refer to this specific SNN as $\Phi_0$. 
We make the following observation. 
Let $v\in V\setminus V_{\rm in}$. 
Suppose $t_u=0$ for all $(u,v)\in E$, and let $t_v$ be the corresponding spike time. 
Then by definition,
\begin{align*}
	P_v(t_v)=\sum_{(u,v)\in E} \mathsf{w}_{(u,v)}\varrho(t_v) = 1,
\end{align*}
from which we obtain
\begin{align} \label{eq:basecase}
	t_v = \frac{1}{\sum_{(u,v)\in E} \mathsf{w}_{(u,v)}}.
\end{align}
Using this, we proceed to derive \eqref{eq:SNNoutputmagconc} for $\Phi=\Phi_0$. 
For each $u\in V\setminus V_{\rm in}$, let $t_u$ be the corresponding spike time at $u$ when the input spike times of $\Phi_0$ are set to $\vec{0}$.
Let $v\in V_{\rm out}$. 
Then we can also interpret $t_v=t_v\big((t_u)_{(u,v)\in E}\big)$.
Therefore, an application of Lemma~\ref{lem:LipLem1} and the base case result \eqref{eq:basecase} grants us
\begin{align} 
	|t_v\big((t_u)_{(u,v)\in E}\big)|& \leq |t_v(\vec{0})| + |t_v\big((t_u)_{(u,v)\in E}\big) - t_v(\vec{0})|\nonumber \\ &\leq \frac{1}{\sum_{(u,v)\in E} \mathsf{w}_{(u,v)}} + \|(t_u)_{(u,v)\in E}\|_{\ell^{\infty}}.\label{eq:inductive}
\end{align}
Continuing as in the proof of Theorem~\ref{thm:LipaffineSNN}, we can find a directed path of synaptic edges, $(v_N,\dots, v_0)$, such that $v_0=v$, $v_N\in V_{\rm in}$, and for each $i=0,\dots,N-1$, it holds that
\begin{align*} 
	\|(t_u)_{(u,v_i)\in E}\|_{\ell^{\infty}} = |t_{v_{i+1}}|.
\end{align*}
By inductively employing the reasoning in \eqref{eq:inductive} along this path, bearing in mind that the input spike times of $\Phi_0$ are $\vec{0}$, we acquire
\begin{align*}
	|t_{v_0}| = |t_v|\leq \sum_{i=0}^{N-1} \frac{1}{\sum_{(u,v_i)\in E} \mathsf{w}_{(u,v_i)}} \leq \frac{\mathsf{L}}{\mathsf{b}},
\end{align*}
which in turn implies
\begin{align} \label{eq:casezeroconc}
	\|\Realization(\Phi_0)(\vec{0})\|_{\ell^{\infty}}\leq \frac{\mathsf{L}}{\mathsf{b}}.
\end{align}
Subsequently, in the case of general $\Phi$, it can be deduced from \eqref{eq:casezeroconc} and the argument presented in the proof of Proposition~\ref{prop:LipPropPhi} (with the synaptic weights coinciding) that
\begin{align*}
	\|\Realization(\Phi)(\vec{0})\|_{\ell^{\infty}} &\leq \|\Realization(\Phi_0)(\vec{0})\|_{\ell^{\infty}} + \|\Realization(\Phi)(\vec{0}) - \Realization(\Phi_0)(\vec{0})\|_{\ell^{\infty}}\\
 &\leq \mathsf{L} \cdot \Big(\frac{1}{\mathsf{b}} + \|\Phi-\Phi_0\|_{\ell^{\infty}}\Big) \leq \mathsf{L} \cdot \Big(\frac{1}{\mathsf{b}} + \mathsf{B}\Big),
\end{align*}
as wanted.
\end{proof}

\begin{proof}[Proof of Theorem~\ref{thm:LipThm}]
Let $x\in\R^{{\rm d}_0}$. 
Then the inputs of $\Phi$ and $\widetilde{\Phi}$ are $t=(t_u)_{u\in V_{\rm in}} = A_{\rm in}(x)$ and $\tilde{t} = (\tilde{t}_u)_{u\in V_{\rm in}} = \tilde{A}_{\rm in}(x)$, respectively, where
\begin{align} \label{eq:encoders}
	A_{\rm in}(x) = W_{\rm in}x + b_{\rm in} \quad\text{ and }\quad \tilde{A}_{\rm in}(x) = \widetilde{W}_{\rm in}x + \tilde{b}_{\rm in}. 
\end{align}
It follows from \eqref{eq:encoders} and the estimate \eqref{eq:innerSNN}, as well as Proposition \ref{prop:LipPropPhi}, that
\begin{align} \label{eq:ClippedInsert}
	\|\Realization(\Phi)(t) -\Realization(\widetilde{\Phi})(\tilde{t})\|_{\ell^{\infty}} 
	&\leq \|\Realization(\Phi)(t) - \Realization(\Phi)(\tilde{t})\|_{\ell^{\infty}} + \|\Realization(\Phi)(\tilde{t}) -\Realization(\widetilde{\Phi})(\tilde{t})\|_{\ell^{\infty}} \\
	\nonumber &\leq \|t-\tilde{t}\|_{\ell^{\infty}} + \mathsf{L} \cdot \Big(1+\frac{1}{\mathsf{b}^2}\Big) \|\Phi - \widetilde{\Phi}\|_{\ell^{\infty}}\\
	\nonumber &\leq \|W_{\rm in} - \widetilde{W}_{\rm in}\|_F\|x\|_{\ell^2} + \|b_{\rm in} - \tilde{b}_{\rm in}\|_{\ell^{\infty}} + \mathsf{L} \cdot \Big(1+\frac{1}{\mathsf{b}^2}\Big) \|\Phi - \widetilde{\Phi}\|_{\ell^{\infty}}\\
	\nonumber &\leq {\rm d}_0^{\frac{1}{2}} \|W_{\rm in} - \widetilde{W}_{\rm in}\|_F\|x\|_{\ell^{\infty}} + \mathsf{L} \cdot \Big(1+\frac{1}{\mathsf{b}^2}\Big) \|\Phi - \widetilde{\Phi}\|_{\ell^{\infty}} + \|b_{\rm in} - \tilde{b}_{\rm in}\|_{\ell^{\infty}}
\end{align}
Denote $z=\Realization(\Phi)(t)$, $\tilde{z} = \Realization(\widetilde{\Phi})(\tilde{t})$. 
We can write
\begin{align} \label{eq:Ts}
	\begin{split}
	    \|\Realization(\Psi)(x) -\Realization(\widetilde{\Psi})(x)\|_{\ell^{\infty}} &= \|A_{\rm out}(z) - \tilde{A}_{\rm out}(\tilde{z})\|_{\ell^{\infty}} \\
	    &\leq \|A_{\rm out}(z) - A_{\rm out}(\tilde{z})\|_{\ell^{\infty}} + \|A_{\rm out}(\tilde{z}) - \tilde{A}_{\rm out}(\tilde{z})\|_{\ell^{\infty}}\\
	    &= T_1 + T_2,
	\end{split}  
\end{align}
where $T_1\coloneqq\|A_{\rm out}(z) - A_{\rm out}(\tilde{z})\|_{\ell^{\infty}}$, $T_2\coloneqq\|A_{\rm out}(\tilde{z}) - \tilde{A}_{\rm out}(\tilde{z})\|_{\ell^{\infty}}$.
Applying \eqref{eq:ClippedInsert} to $T_1$, we get
\begin{align*}
	T_1 \leq 
	{\rm d}_{\rm out}^{\frac{1}{2}}\|W_{\rm out}\|_F 
	\bigg({\rm d}_0^{\frac{1}{2}} \|W_{\rm in} - \widetilde{W}_{\rm in}\|_F\|x\|_{\ell^{\infty}}  + \mathsf{L} \cdot \Big(1+\frac{1}{\mathsf{b}^2}\Big) \|\Phi - \widetilde{\Phi}\|_{\ell^{\infty}} + \|b_{\rm in} - \tilde{b}_{\rm in}\|_{\ell^{\infty}} \bigg).
\end{align*}
Further, given
\begin{align*} 
	A_{\rm out}(x) = W_{\rm out}x + b_{\rm out} \quad\text{ and }\quad \tilde{A}_{\rm out}(x) = \widetilde{W}_{\rm out}x + \tilde{b}_{\rm out},
\end{align*}
we can also similarly obtain
\begin{align*}
	T_2 &\leq {\rm d}_{\rm out}^{\frac{1}{2}}\|W_{\rm out} - \widetilde{W}_{\rm out}\|_F\|\tilde{z}\|_{\ell^{\infty}} + \|b_{\rm out} - \tilde{b}_{\rm out}\|_{\ell^{\infty}}\\
	&\leq {\rm d}_{\rm out}^{\frac{1}{2}}\|W_{\rm out} - \widetilde{W}_{\rm out}\|_F\Big(\|\Realization(\widetilde{\Phi})(\tilde{t}) - \Realization(\widetilde{\Phi})(\vec{0})\|_{\ell^{\infty}} + \|\Realization(\widetilde{\Phi})(\vec{0})\|_{\ell^{\infty}} \Big) + \|b_{\rm out} - \tilde{b}_{\rm out}\|_{\ell^{\infty}}\\
	&\leq {\rm d}_{\rm out}^{\frac{1}{2}}\|W_{\rm out} - \widetilde{W}_{\rm out}\|_F\Big(\|\tilde{t}\|_{\ell^{\infty}} + \mathsf{L} \cdot \Big(\frac{1}{\mathsf{b}} + \mathsf{B}\Big) \Big) + \|b_{\rm out} - \tilde{b}_{\rm out}\|_{\ell^{\infty}}\\
	&\leq {\rm d}_{\rm out}^{\frac{1}{2}}\|W_{\rm out} - \widetilde{W}_{\rm out}\|_F\Big({\rm d}_0^{\frac{1}{2}}\|\widetilde{W}_{\rm in}\|_F\|x\|_{\ell^{\infty}} + \|\tilde{b}_{\rm in}\|_{\ell^{\infty}} + \mathsf{L} \cdot \Big(\frac{1}{\mathsf{b}} + \mathsf{B}\Big) \Big) + \|b_{\rm out} - \tilde{b}_{\rm out}\|_{\ell^{\infty}},
\end{align*}
where we have used \eqref{eq:innerSNN} and Lemma~\ref{lem:SNNoutputmag} in the third inequality above. 
Substituting these estimates back into \eqref{eq:Ts}, we acquire the desired conclusion for the theorem.
\end{proof}

\subsection{Proof of Theorem~\ref{thm:generalizationGapTheorem}} \label{appx:generalization}

A standard learning result employing covering numbers is as follows.

\begin{theorem}[{\cite[Exercise 3.31]{mohri2018foundations}}]\label{thm:mainLearningResultLit}
Let ${\rm d}_0,m\in \N$.
Let $\mathcal{H}$ be a hypothesis set of functions from $[0,1]^{{\rm d}_0}$ to $[c_{\rm min}, c_{\rm max}]$, where $c_{\rm min},  c_{\rm max} \in \R$. 
Let $\mathcal{D}$ be a distribution on $[0,1]^{{\rm d}_0} \times [c_{\rm min}, c_{\rm max}]$. 
Let, for $m \in \N$, $S \sim \mathcal{D}^m$ be a sample. 
Then, for all $\varepsilon >0$
\begin{align*}
	\mathbb{P}\Big( \sup_{h\in \mathcal{H}} | \risk(h) - \widehat{\mathcal{R}}_S(h)| \geq \varepsilon \Big) \leq 2\mathcal{N}\Big(\mathcal{H}, \frac{\varepsilon}{8 (c_{\rm max}-c_{\rm min})}, L^{\infty}([0,1]^{{\rm d}_0}) \Big)
	\mathrm{exp}\Big(\frac{-m\varepsilon^2}{ 2 (c_{\rm max}-c_{\rm min})^4}\Big).
\end{align*}
\end{theorem}

We show below that Theorem~\ref{thm:mainLearningResultLit} and \eqref{eq:coveringNumberEstimate} imply Theorem~\ref{thm:generalizationGapTheorem}.

\begin{proof}[Proof of Theorem~\ref{thm:generalizationGapTheorem}]
Set, for $m\in \N$ satisfying \eqref{eq:assumptionOnM}, 
\begin{align*}
	\varepsilon(\delta) \coloneqq \sqrt{\frac{2(\mathsf{M} \log (m \lceil 16 \mathsf{B} C^{\star}_{\rm Lip} \rceil) +  \log(2/\delta))}{m}} \leq 1,
\end{align*}
then by Theorem \ref{thm:mainLearningResultLit} the statement of this theorem follows if 
\begin{align*}
	2\mathcal{N} \Big(\Realization_{[0,1]} (\mathcal{P}_{\rm SNN}^{\star}(G,\vec{{\rm d}}; \mathsf{b}, \mathsf{B})), \frac{\varepsilon(\delta)}{8}, L^{\infty}([0,1]^{{\rm d}_0})\Big) \mathrm{exp}\Big(\frac{-m\varepsilon(\delta)^2}{ 2}\Big) \leq \delta.
\end{align*}
Since, by definition
\begin{align*}
	\mathrm{exp}\Big(\frac{-m\varepsilon(\delta)^2}{2}\Big) 
	= \frac{\delta}{2} \mathrm{exp}\Big(-\mathsf{M} \log \Big(m \Big \lceil 16 \mathsf{B} C^{\star}_{\rm Lip}  \Big\rceil\Big) \Big),
\end{align*}
we can conclude the theorem if 
\begin{align} \label{eq:ShowingThiswillCompleteTheProof}
	\mathcal{N}\Big(\Realization_{[0,1]} (\mathcal{P}_{\rm SNN}^{\star}(G,\vec{{\rm d}}; \mathsf{b}, \mathsf{B})), \frac{\varepsilon(\delta)}{8}, L^{\infty}([0,1]^{{\rm d}_0}) \Big)
	\leq \mathrm{exp}\Big(\mathsf{M} \log \Big(m \Big \lceil 16 \mathsf{B} C^{\star}_{\rm Lip}  \Big\rceil\Big) \Big).
\end{align}
Let us have a look at the left-hand side of \eqref{eq:ShowingThiswillCompleteTheProof}. 
Building upon the rationale given in \eqref{eq:coveringNumberEstimate}, it holds that
\begin{align} \label{eq:logstep}
	\begin{split}
	    \log\Big(\mathcal{N}\Big(\Realization_{[0,1]} (\mathcal{P}_{\rm SNN}^{\star}(G,\vec{{\rm d}}; \mathsf{b}, \mathsf{B})), \frac{\varepsilon(\delta)}{8}, L^{\infty}([0,1]^{{\rm d}_0})\Big) \Big)
	    &\leq  \mathsf{M} \log \Big(\Big \lceil \frac{16 \mathsf{B} C^{\star}_{\rm Lip}}{\varepsilon(\delta)} \Big \rceil\Big)\\
	    &\leq \mathsf{M} \log \Big(m \Big \lceil 16 \mathsf{B} C^{\star}_{\rm Lip} \Big \rceil \Big),  
	\end{split}      
\end{align}
where we have used that $\lceil 1/(\varepsilon(\delta))\rceil \leq \lceil\sqrt{m}\rceil \leq m$.
Exponentiating \eqref{eq:logstep} yields \eqref{eq:ShowingThiswillCompleteTheProof} and thus completes the proof.
\end{proof}

\subsection{Proof of Theorem \ref{thm:generalizationGapTheoremRealizable}} \label{app:proofThmRealizable}

Let $\Psi_m$ satisfy \eqref{eq:almostERM}.
By setting $\delta=\varepsilon$ in the statement of \cite[Lemma 4]{schmidt2020nonparametric}, it follows directly that, for $\varepsilon\in (0,1]$,
\begin{align*}
    &\mathbb{E}(\hat{\risk}_S(\Realization_{[0,1]}(\Psi_m))) \\
    &\leq 4 \bigg( \inf_{\Psi \in \mathcal{P}_{\rm SNN}^{\star}(G,\vec{{\rm d}}; \mathsf{b}, \mathsf{B})} \risk(\Realization_{[0,1]}(\Psi)) 
    + \frac{18 \mathcal{N}(\Realization_{[0,1]}( \mathcal{P}_{\rm SNN}^{\star}(G,\vec{{\rm d}}; \mathsf{b}, \mathsf{B})), \varepsilon, L^\infty(\R^{{\rm d}_0})) + 73}{m}  +  32 \varepsilon \bigg).
\end{align*}
Subsequently, invoking \eqref{eq:coveringNumberEstimate} yields
\begin{multline} \label{eq:EMcDiarmid}
    \mathbb{E}(\hat{\risk}_S(\Realization_{[0,1]}(\Psi_m))) \\ \leq 4 \bigg( \inf_{\Psi \in \mathcal{P}_{\rm SNN}^{\star}(G,\vec{{\rm d}}; \mathsf{b}, \mathsf{B})} \risk(\Realization_{[0,1]}(\Psi))   + \frac{18 \mathsf{M} \log \Big(\Big \lceil \frac{2 \mathsf{B} C^{\star}_{\rm Lip} }{\varepsilon} \Big\rceil\Big)  + 73}{m}  +  32 \varepsilon  \bigg).  
\end{multline}
Now, it is straightforward to verify that the function
\begin{equation*}
    g(z_1,\dots,z_m)\coloneqq\frac{1}{m}\sum_{i=1}^m  |\Realization_{[0,1]}(\Psi_m)(z_i) - f_0(z_i)|^2
\end{equation*}
is of bounded variation, with 
\begin{equation} \label{eq:BV}
    |g(z_1,\dots,z_i,\dots,z_m) - g(z_1,\dots,z_i',\dots,z_m)| \leq \frac{\sqrt{2}}{m}.
\end{equation}
Thus, applying McDiarmid’s inequality (\cite[Theorem D.3]{mohri2018foundations}) to $g$, and combining with \eqref{eq:EMcDiarmid}, \eqref{eq:BV}, we obtain the bound in \eqref{eq:fastRateLearningForERM}. \qed

\subsection{Proof of Lemma~\ref{lem:minapprox}} \label{appx:minapproxpf}

We begin by constructing $\Phi^{\rm min}_{\varepsilon}=(G,\mathsf{W}^{\rm min}_{\varepsilon},D^{\rm min}_{\varepsilon})$ as follows.
We take $G$ to be a graph with ${\rm d}_0$ input nodes $u_1, \dots, u_{{\rm d}_0}$ and one output node $v$.
Let $\mathsf{W}^{\rm min}_{\varepsilon}$ consist of $\mathsf{w}_{(u_i,v)}=\varepsilon^{-1}$ for $i = 1, \dots, {\rm d}_0$. 
Let all the delays be zero, i.e.
$d_{(u_i,v)}=0$, for $i = 1, \dots, {\rm d}_0$.
Finally, we formalize the structure of $\Psi^{\rm min}_{\varepsilon}$ by defining $A_{\rm in}:\R^{{\rm d}_0}\to\R^{{\rm d}_0}$, $A_{\rm out}:\R\to\R$ as $A_{\rm in} = {\rm Id}^{{\rm d}_0 \times {\rm d}_0}$, $A_{\rm out} = {\rm Id}^{1\times 1}$, respectively.

To show that $\Psi^{\rm min}_{\varepsilon}$ fulfills \eqref{minapproxconc}, we first calculate the output spike time $t_v$ by $\Phi^{\rm min}_{\varepsilon}$ from the input spike times $t_{u_i}$. 
Suppose, without loss of generality, that 
\begin{align*}
	t_{u_1} = \min \{ t_{u_i}\colon i = 1, \dots, {\rm d}_0\}.
\end{align*}
The potential at $v$ then takes the form \eqref{eqdef:accumulation}
\begin{align} \label{chargesv}
	P_v(t) = \sum_{i=1}^{{\rm d}_0}\varepsilon^{-1}\varrho(t-t_{u_i}).
\end{align}
Since all the terms on the right-hand side of \eqref{chargesv} are non-negative, we deduce 
\begin{align*}
	P_v(t_{u_1} + \varepsilon) \geq \varepsilon^{-1}\varrho(\varepsilon) = 1.
\end{align*}
It follows that the spike time $t_v$ of $v$ is not larger than $t_{u_1} + \varepsilon$.
On the other hand, $P_v(t_{u_1}) = 0$. 
Therefore, $t_v \in (t_{u_1}, t_{u_1}+\varepsilon]$, by the continuity of $P_v$ and the intermediate value theorem.
We conclude that, for $x_1,\dots,x_{{\rm d}_0}\in\R$,
\begin{multline*}
|\Realization(\Psi^{\rm min}_{\varepsilon})(x_1,\dots, x_{{\rm d}_0}) - \min\{x_1, \dots, x_{{\rm d}_0}\}| \\
= |(A_{\rm out}\circ\Phi^{\rm min}_{\varepsilon}\circ A_{\rm in})(x_1,\dots, x_{{\rm d}_0}) - \min\{x_1, \dots, x_{{\rm d}_0}\}|\leq\varepsilon,
\end{multline*}
which shows \eqref{minapproxconc}. 
It remains to observe that ${\rm Size}(\Psi^{\rm min}_{\varepsilon})=2{\rm d}_0+1$, as per the construction. 
\qed

\subsection{Proof of Lemma~\ref{lem:ReLUapprox}} \label{appx:ReLUapproxpf}

We create $\Psi_{a,b,c,d,\varepsilon} = (A_{\rm in}, \Phi, A_{\rm out})$ as follows. 
We define $A_{\rm in}\colon \R^{{\rm d}_0}\to \R^2$ to be such that
\begin{align} \label{eq:Ain}
	A_{\rm in}(x) = 
	\begin{pmatrix} -a^{\top}x - b \\ 0 \end{pmatrix},
\end{align}
and $A_{\rm out}\colon \R\to \R$ such that $A_{\rm out}(x) = -cx + d$.
In constructing $\Phi=(G,\mathsf{W}, D)$, we follow the structure of $\Phi^{\rm min}_{\varepsilon}$ in the proof of Lemma~\ref{lem:minapprox}.
Namely, we let $G$ consist of two synapses, $(u_1,v)$, $(u_2,v)$, with $\mathsf{w}_{(u_i,v)}=\varepsilon^{-1}$ and $d_{(u_i,v)}=0$.
To see that $\Psi_{a,b,c,d,\varepsilon}$ satisfies \eqref{eq:approximationOfReLU}, we note, as argued in the proof of Lemma~\ref{lem:minapprox},
\begin{align*}
	|\Realization(\Phi)(t_1,t_2) - \min\{t_1,t_2\}|\leq\varepsilon;
\end{align*}
therefore, 
\begin{multline} \label{eq:outend3}
    \Big| A_{\rm out} \circ \Realization(\Phi)(t_1,t_2) - \Big( -c \cdot \big(\min\{t_1,t_2\}\big) + d\Big)\Big| \\
    = \Big| A_{\rm out} \circ \Realization(\Phi)(t_1,t_2) - \Big(c\cdot \big(\max\{-t_1,-t_2\}\big) + d\Big)\Big| \leq |c|\varepsilon.
\end{multline}
Combining \eqref{eq:Ain}, \eqref{eq:outend3}, we obtain \eqref{eq:approximationOfReLU}.

Finally, it is straightforward to see that ${\rm Size}(\Psi_{a,b,c,d, \varepsilon})\leq {\rm d}_0+5$.
Moreover, the weights of $\Psi_{a,b,c,d, \varepsilon}$ are bounded above in absolute value by $\max\{1/\varepsilon, \|a\|_{\ell^{\infty}},|b|,|c|,|d|\}$, and the synaptic weights are bounded below by ${1}/{\varepsilon}$.

\subsection{Proof of Theorem~\ref{thm:universality}} \label{appx:universalitypf}

For $M\in\N$, recall the set $H_M$ defined in \eqref{eq:thisStatementCanBeUsedForBarronToo}.
Then by the well-known result \cite[Theorem 1]{leshno1993multilayer}, $\bigcup_{M \in \N } H_M$ is dense in $\mathcal{C}(\Omega)$. 
On the other hand, it is evident from Lemma~\ref{lem:ReLUapprox} and Definition~\ref{def:Addition} that each $H_M$ is contained in the closure of the set of all realizations of affine SNNs.
The result now follows.
\qed

\subsection{Proof of Theorem~\ref{thm:reapproximationofFEMSpaces}} \label{appx:FEMpf}

We first prove that affine SNNs with relatively moderate size can effectively approximate the basis elements $\phi_{\eta}$ of $V_{\mathcal{T}}$, in the following lemma.

\begin{lemma} \label{lem:basis}
Let ${\rm d}_0\in\N$, and let $\Omega \subset \R^{{\rm d}_0}$ be compact.
Let $\mathcal{T}$ be a regular triangulation of $\Omega$ with node set $\mathcal{N}$. 
Let $\varepsilon>0$.
Then for each node $\eta \in \mathcal{N}$ such that $G(\eta)$ is convex, there exists an SNN $\Psi_\varepsilon^\eta$ such that, for all $x\in\Omega$
\begin{align*}
	|\Realization(\Psi_\varepsilon^\eta)(x) - \phi_\eta(x)|\leq \varepsilon.
\end{align*}
Moreover, 
\begin{align*}
	{\rm Size}(\Psi^{\eta}_{\varepsilon}) \leq ({\rm d}_0 +2)\#T(\eta) + 6,
\end{align*}
all the weights in $\Psi_\varepsilon^\eta$ are bounded above in absolute value by 
\begin{align*}
	\max\big\{1,1/h_{\min}(\mathcal{T}), C{\rm d}_0/h_{\min}(\mathcal{T}),3/\varepsilon\big\}
\end{align*}
for some $C = C(\Omega)>0$, and all the synaptic weights are bounded below by $\min\{1,{3}/{\varepsilon}\}$.
\end{lemma}

\begin{proof}
Since for each $\tau\in T(\eta)$, $g_{\tau}$ is globally affine, it takes the form, $g_{\tau}(x) = a_{\tau}^{\top} x + b_{\tau}$, where $a_{\tau}\in\R^{{\rm d}_0}$, $b_{\tau}\in\R$. 
We build an affine SNN $\Psi^{\eta}_{3\varepsilon} = (A_{\rm in}, \Phi, A_{\rm out})$ as follows. 
Let $A_{\rm out}\colon \R^2\to\R$ be such that, $A_{\rm out}(x,y) = x-1-y$, and $A_{\rm in}\colon \R^{{\rm d}_0}\to \R^{\#T(\eta) + 1}$ such that
\begin{align*} 
	A_{\rm in}(x) = \begin{pmatrix} \big(a_{\tau}^{\top} x + b_{\tau}\big)_{\tau\in T(\eta)} \\ 0 \end{pmatrix}.
\end{align*}
Next, we consider $\Phi = (G,\mathsf{W},D)$, a positive SNN where $G$ comprises $\#T(\eta)+1$ input nodes, labeled $u_0, u_{\tau}$, for $\tau\in T(\eta)$, $2$ output nodes, labeled $v_1$, $v_2$, and $1$ intermediate node $w$, along with $\#T(\eta) + 3$ synaptic edges.
The synaptic edges and their weights are,
\begin{align} \label{eq:wdesigned}
    \mathsf{w}_{(u_{\tau},w)} = \mathsf{w}_{(w,v_2)} = \mathsf{w}_{(u_0,v_2)} = \varepsilon^{-1}, \quad \text{ and }\quad \mathsf{w}_{(w,v_1)} = 1.
\end{align}
Let all the corresponding delays be zero. 
Then, given input spike times $((t_{u_{\tau}})_{\tau\in T(\eta)}, t_{u_0})$, we can deduce from \eqref{eq:wdesigned} and the reasoning provided in the proof of Lemma~\ref{lem:minapprox} that
\begin{align*}
	|t_{w} - \min\{t_{u_{\tau}}\colon \tau\in T(\eta)\}|\leq \varepsilon,
\end{align*}
and $|t_{v_2} - \min\{t_w, t_{u_0}\}|\leq \varepsilon$. 
Therefore,
\begin{equation} \label{eq:finding}
    \begin{split}
        |t_{v_1} - (\min\{t_{u_{\tau}}\colon \tau\in T(\eta)\}+1)| &\leq \varepsilon,\\
	|t_{v_2} - \min\{ t_{u_0}, \min\{t_{u_{\tau}}\colon \tau\in T(\eta)\}\}| &\leq 2\varepsilon.
    \end{split}
\end{equation}
Integrating \eqref{eq:finding} with the specifics given to $A_{\rm in}$, $A_{\rm out}$, we conclude for $x\in\Omega$, 
\begin{align*}
	|\Realization(\Psi^{\eta}_{3\varepsilon})(x) - \big(\min_{\tau \in T(\eta)} (a_{\tau}^{\top} x + b_{\tau}) - \min\{ 0, \min_{\tau \in T(\eta)} (a_{\tau}^{\top} x + b_{\tau})\} \big)| \leq 3\varepsilon.
\end{align*}
It should be now routine to check that ${\rm Size}(\Psi^{\eta}_{3\varepsilon}) \leq ({\rm d}_0 +2)\#T(\eta) + 6$ and that every synaptic weight in $\Psi^{\eta}_{3\varepsilon}$ is bounded below by $\min\{1,1/\varepsilon\}$. 
As for the remaining weights, it can be inferred from \eqref{eq:delta} and the fact that $g_{\tau}=\phi_{\eta}$ on $\tau$ that
\begin{align} \label{eq:weighta}
	\|a_{\tau}\|_{\ell^{\infty}}\leq 1/h_{\rm min}(\mathcal{T}).
\end{align}
Additionally, it also follows from \eqref{eq:delta} that $g_{\tau}(x_{\tau}^{\star}) = 0$ for some $x_{\tau}^{\star}\in \Omega$.
Thus
\begin{align} \label{eq:weightb}
	|b_{\tau}| \leq |a_{\tau}^{\top} x_{\tau}^{\star}| \leq {\rm d}_0\,\sup \{|x|\colon x \in \Omega\}/h_{\min}(\mathcal{T}).
\end{align}
Therefore, by combining \eqref{eq:weighta}, \eqref{eq:weightb}, we conclude the final assertion of the lemma with $1/\varepsilon$ instead of $3/\varepsilon$.
Finally, the result follows by substituting $\varepsilon$ by $\varepsilon/3$.  
\end{proof}

Theorem~\ref{thm:reapproximationofFEMSpaces} now follows as an immediate consequence of Lemma~\ref{lem:basis}.

\begin{proof}[Proof of Theorem~\ref{thm:reapproximationofFEMSpaces}]
Let $f \in V_{\mathcal{T}}$. 
Since $(\phi_{\eta})_{\eta\in\mathcal{N}}$ is a basis of $V_{\mathcal{T}}$ satisfying \eqref{eq:delta}, we can write
\begin{align*}
	f(x) = \sum_{\eta \in \mathcal{N}} f(\eta) \phi_\eta(x).
\end{align*}
For each $\eta\in\mathcal{N}$, following an argument presented in the proof of Lemma~\ref{lem:basis}, we construct an affine SNN $\Psi^{\eta}_{\varepsilon}$ that guarantees for all $x\in\Omega$,
\begin{align*}
	|\Realization(\Psi^{\eta}_{\varepsilon})(x) - f(\eta)\phi_{\eta}(x)|\leq |f(\eta)|\varepsilon.
\end{align*}
Let $\Psi^f_{\varepsilon} \coloneqq \bigoplus_{\eta\in\mathcal{N}} \Psi^{\eta}_{\varepsilon}$.
The theorem now follows directly from an application of Lemma~\ref{lem:addition} and the triangle inequality. 
\end{proof}

\subsection{Proof of Theorem~\ref{thm:approximationWsinfty}} \label{appx:approximationWsinfinity}

We cite the following known result, which we will use to demonstrate Theorem~\ref{thm:approximationWsinfty}.

\begin{proposition}[{\cite[Proposition 1]{bertoluzza2012primer}}]\label{prop:FEMResult}
Let ${\rm d}_0\in \N$ and $\Omega \subset \R^{{\rm d}_0}$ be a compact domain. 
Let $s\in \{1,2\}$.
Let $\mathcal{T}$ be a regular triangulation of $\Omega$.
Then for every $f\in W^{s, \infty}(\Omega)$, there exist $g = \sum_{\eta \in \mathcal{N}} c_\eta \phi_\eta$ and constants $C_1, C_2 = C_2(\mathcal{T}, {\rm d}_0) >0$, such that $|c_\eta|\leq C_1\|f\|_{L^{\infty}(\Omega)}$, and that
\begin{align*}
	\|f-g\|_{L^{\infty}(\Omega)} \leq C_2\cdot (h_{\max}(\mathcal{T}))^s \|f\|_{W^{s, \infty}(\Omega)}.
\end{align*}
In addition, the constant $C_2(\mathcal{T},{\rm d}_0)$ depends on $\mathcal{T}$ only through the so-called shape coefficient $h_{\rm max}(\mathcal{T})/h_{\rm min}(\mathcal{T})$, i.e., $C_2(\mathcal{T},{\rm d}_0) = C_2(h_{\rm max}(\mathcal{T})/h_{\rm min}(\mathcal{T}),{\rm d}_0)$.
\end{proposition}

\begin{proof}[Proof of Theorem~\ref{thm:approximationWsinfty}]
In what follows, the majorant constant $C$ is subject to change meaning from one instance to the next, and its parametric dependence, if present, will be explicitly specified.
Since $\Omega$ is admissible, we can choose for $N\in\N$ a triangulation $\mathcal{T}_N$  for which \eqref{eq:comparability} holds. 
Therefore, by Proposition~\ref{prop:FEMResult}, there exists $g = \sum_{\eta \in \mathcal{N}} c_\eta \phi_\eta$ satisfying
\begin{align} \label{eq:part1}
	\|f-g\|_{L^{\infty}(\Omega)} \leq C N^{-s/{\rm d}_0} \|f\|_{W^{s, \infty}(\Omega)},
\end{align}
where $C=C({\rm d}_0)$ only.
In turn, since $g\in V_{\mathcal{T}_N}$, Theorem~\ref{thm:reapproximationofFEMSpaces} implies the existence of an affine SNN $\Psi^g_N$ such that
\begin{align} \label{eq:part2}
	\|\Realization(\Psi^g_N) - g\|_{L^{\infty}(\Omega)} \leq CN\max_{\eta\in\mathcal{N}(\mathcal{T}_N)} |c_{\eta}| N^{-s/{\rm d}_0 - 1}\leq CN^{-s/{\rm d}_0}\|f\|_{L^{\infty}(\Omega)}.
\end{align}
By renaming $\Psi^g_N$ to $\Psi^f_N$ and combining \eqref{eq:part1}, \eqref{eq:part2}, we obtain \eqref{eq:Sobolevthmconc}.
The size of $\Psi^f_N$ and its associated weights can now be determined by applying the conclusions of Theorem~\ref{thm:reapproximationofFEMSpaces} with $\varepsilon$ in place of $N^{-s/{\rm d}_0 - 1}$.
\end{proof}

\subsection{Proof of Theorem~\ref{thm:COD}} \label{appx:CODpf}

Before proceeding, we introduce a key result derived from \cite[Theorem 1]{barron1992neural} and \cite[Proposition 2.2]{caragea2023neural}, which sets the stage for our analysis.

\begin{theorem}\label{thm:Barron} 
Let ${\rm d}_0\in \N$. 
There is a universal constant $\kappa > 0$ such that the following holds. 
For every $K > 0$, every $f \in \Gamma_K$, and every $M \in \N$, there exists an element $g \in H_{8M}$ (as in \eqref{eq:thisStatementCanBeUsedForBarronToo}) such that
\begin{align*}
	\sup_{x \in \overline{B(0,1)}} |f(x) - g(x)| \leq  \frac{\kappa \, {\rm d}_0^{\frac{1}{2}} K}{\sqrt{M}},
\end{align*}
where $\overline{B(0,1)}$ denotes the closed unit ball in $\R^{{\rm d}_0}$.
Furthermore, for all $i = 1, \dots, 8M$, $\|a_i\|_{\ell^{\infty}}, |b_i|, |c_i|, |d_i|$ in the definition of $g$ via \eqref{eq:thisStatementCanBeUsedForBarronToo} can be chosen so that
\begin{align*}
	\|a_i\|_{\ell^{\infty}}, |b_i|, |c_i|, |d_i|\leq C\sqrt{K},
\end{align*}
for some $C>0$.
\end{theorem}

In the forthcoming discussion, all majorant constants are universal, and the meaning of the analytic constants $C,c>0$ may vary between different instances. 

\begin{proof}[Proof of Theorem~\ref{thm:COD}]
By Theorem~\ref{thm:Barron}, there exist constant $\kappa>0$ and $a_i\in\R^{{\rm d}_0}, b_i,c_i,d_i\in\R$ for $i = 1, \dots, 8M$ such that $\|a_i\|_{\ell^{\infty}}, |b_i|, |c_i|, |d_i| \leq C\sqrt{K}$, for some $C>0$, and
\begin{align}\label{eq:firstPartOfTriangleIneq}
	\sup_{x \in \overline{B(0,1)}} \Big|f(x) -  \sum_{i=1}^{8M} c_i \max\{a_i^\top x + b_i,0\} + d_i \Big| \leq  \frac{\kappa\,{\rm d}_0^{\frac{1}{2}} K}{\sqrt{M}}.
\end{align}
On the other hand, by using Lemma~\ref{lem:ReLUapprox} and performing repeated additions of affine SNNs, we can construct an affine SNN $\Psi$ such that
\begin{align}\label{eq:secondPartOfTriangleIneq}
	\sup_{x \in \overline{B(0,1)}} \Big|\Psi(x) -  \sum_{i=1}^{8M} c_i \max\{a_i^\top x + b_i,0\} + d_i \Big| \leq  \frac{\kappa K}{\sqrt{M}}.
\end{align}
Letting $\Psi^f_M=\Psi$, we infer from \eqref{eq:firstPartOfTriangleIneq}, \eqref{eq:secondPartOfTriangleIneq} that
\begin{align*} 
	\sup_{x \in \overline{B(0,1)}} \Big|\Psi^f_M(x) -  f(x) \Big| \leq  \frac{\kappa K \cdot ({\rm d}_0^{\frac{1}{2}} +1)}{\sqrt{M}} \leq \frac{2\kappa \, {\rm d}_0^{\frac{1}{2}} K}{\sqrt{M}}.
\end{align*}
Hence, we can set $\nu=2\kappa$.
In addition, it is readily seen from the conclusions of Lemmas~\ref{lem:addition},~\ref{lem:ReLUapprox} and \eqref{eq:secondPartOfTriangleIneq} that ${\rm Size}(\Psi^f_M)\leq C{\rm d}_0M$, that all the weights in $\Psi^f_M$ can be bounded above by $C\cdot (M^{3/2}/\sqrt{K} + \sqrt{K})$, and all of its synaptic weights can be bounded below by $cM^{3/2}/\sqrt{K}$, for some $C,c>0$.
This completes the proof. 
\end{proof}


\subsection{Proof of Theorem~\ref{thm:fullErrorBound}} \label{appx:fullErrorBound}

It follows directly from \eqref{eq:squarerisk}, \eqref{eq:uniformapprox} that
\begin{equation*}
    \inf_{g \in \mathcal{H}(m^{-1/(\kappa_{\mathsf{M}} + 2)})} \risk(g) \leq m^{-2/(\kappa_{\mathsf{M}} + 2)}.
\end{equation*}
Therefore, since an empirical risk minimizer is also an empirical risk minimizer, Theorem \ref{thm:generalizationGapTheoremRealizable} and its proof apply to $g_m$; that is,
\begin{align} \label{eq:riskPsi}
    \risk(g_m) &\leq  4 \bigg( \varepsilon^2   + \frac{18 \mathsf{M}(\varepsilon) \log \Big(\Big \lceil \frac{2 \mathsf{B}(\varepsilon) C^{\star}_{\rm Lip}(\varepsilon) }{\varepsilon^2 } \Big\rceil\Big)  + 72}{m}  +  32 \varepsilon^2  \bigg) + \sqrt{\frac{\log(1/\delta)}{m}},
\end{align}
with $\varepsilon=m^{{-1}/(\kappa_{\mathsf{M}} + 2)}$, which happens with probability at least $1-\delta$.
To evaluate this expression and obtain \eqref{eq:risksimp}, we proceed to estimate 
\begin{equation*}
    18\mathsf{M}(\varepsilon) \log \Big(\Big \lceil \frac{2 \mathsf{B}(\varepsilon) C^{\star}_{\rm Lip}(\varepsilon) }{\varepsilon^2 } \Big\rceil\Big) + 72.
\end{equation*}
Recall from definition~\eqref{eqdef:Lstar} that 
\begin{equation*}
    C^{\star}_{\rm Lip}(\varepsilon)
    = {\rm d}_{\rm out}(\varepsilon) 
    \Big( 2\mathsf{B}(\varepsilon)\Big({\rm d}_{\rm in}^{\frac{1}{2}}(\varepsilon){\rm d}_0 + 1\Big) + \mathsf{B}(\varepsilon)\mathsf{L}(\varepsilon)\Big(2+\frac{1}{\mathsf{b}^2(\varepsilon)}\Big) + \frac{\mathsf{L}(\varepsilon)}{\mathsf{b}(\varepsilon)}\Big) + 1,
\end{equation*}
where $\mathsf{L}(\varepsilon)$ represents the graph depth $G(\varepsilon)$.
Thus, using
\begin{align*}
    \max\{{\rm d}_{\rm out}(\varepsilon), {\rm d}_{\rm in}^{\frac{1}{2}}(\varepsilon){\rm d}_0, \mathsf{L}(\varepsilon)\}\leq \mathsf{M}(\varepsilon) \leq \varepsilon^{-\kappa_{\mathsf{M}}},
\end{align*}
and $\varepsilon^{\kappa_{\mathsf{B}}} \leq \mathsf{b}(\varepsilon) \leq 1\leq \mathsf{B}(\varepsilon) \leq \varepsilon^{-\kappa_{\mathsf{B}}}$, we deduce
\begin{align*}
	C^{\star}_{\rm Lip}(\varepsilon) \leq C m^{(2\kappa_{\mathsf{M}}+3\kappa_{\mathsf{B}})/(\kappa_{\mathsf{M}}+ 2)},
\end{align*}
which in turn implies that
\begin{align} \label{eq:Mbound}
    18\mathsf{M}(\varepsilon) \log \Big(\Big \lceil \frac{2 \mathsf{B}(\varepsilon) C^{\star}_{\rm Lip}(\varepsilon) }{\varepsilon^2 } \Big\rceil\Big) + 72 \leq c \varepsilon^{-\kappa_{\mathsf{M}}} \max\{1, \kappa_{\mathsf{B}}/ \kappa_{\mathsf{M}}\} \log(m) 
\end{align}
for a universal constant $c >0$. Lastly, for $\varepsilon=m^{{-1}/(\kappa_{\mathsf{M}} + 2)}$, we observe
\begin{align} \label{eq:ratbound}
    \frac{\varepsilon^{-\kappa_{\mathsf{M}}}}{m} = m^{ \frac{\kappa_{\mathsf{M}}}{\kappa_{\mathsf{M}} + 2} - 1} = m^{- \frac{2}{\kappa_{\mathsf{M}} + 2}} = \varepsilon^2.
\end{align}
The proof is now completed by substituting \eqref{eq:Mbound}, \eqref{eq:ratbound}, along with the expression of $\varepsilon$ into \eqref{eq:riskPsi}.
\qed

\section{SNNs with general synaptic weights} \label{appx:discontinuity}

In the following, we consider a modified version of the model \eqref{eqdef:accumulation} where synaptic weights are allowed to take real values. 
The first example illustrates discontinuity in terms of input spike times, while the second illustrates discontinuity in terms of network parameters. 

\begin{exmp}
Consider an SNN with three input neurons $u_1, u_2, u_3$ presynaptic to one output neuron $v$. Let $\mathsf{w}_{(u_1,v)}=\mathsf{w}_{(u_3,v)}=1$ and $\mathsf{w}_{(u_2,v)}=-1$.
Further, let $d_{(u_1,v)}=d_{(u_2,v)}=d_{(u_3,v)}=0$.
Then for every $t\in\R$
\begin{align*}
    P_v(t) = \varrho (t - t_{u_1}) - \varrho (t - t_{u_2}) + \varrho (t - t_{u_3}).
\end{align*}
Observe that
\begin{align*}
    \sum_{i=1}^3 \mathsf{w}_{(u_i,v)} >0,
\end{align*}
which is a condition that guarantees the existence of network output spike times, first identified in \cite{singh2023expressivity}.
Thus, $t_v$ exists, as a function of $(t_{u_i})_{i=1}^3$. 
However, $t_v(t_{u_1}, t_{u_2}, t_{u_3})$ is discontinuous at $(t_{u_1},t_{u_2}, t_{u_3}) = (0,1,2)$.
Indeed, for $\varepsilon >0$, the spike time at $(t_{u_1}, t_{u_2}, t_{u_3}) = (0,1+\varepsilon, 2)$ is $t_v = 1$, while the spike time at $(t_{u_1}, t_{u_2}, t_{u_3}) = (0,1-\varepsilon, 2)$ is $2+\varepsilon$.
\end{exmp}

\begin{exmp}
Consider an SNN with three input neurons $u_1, u_2, u_3$ presynaptic to one output neuron $v$. Let $\mathsf{w}_{(u_1,v)}=\mathsf{w}_{(u_3,v)}=1$ and $\mathsf{w}_{(u_2,v)}=-1$.
Further, for $s\in\R$, let $d_{(u_1,v)}=0$, $d_{(u_2,v)}=1+s$, and $d_{(u_3,v)}=2$.
Let $t_v(t_{u_1},t_{u_2},t_{u_3};s)$ denote the output spike time at $v$ given the parameter $s$. 
Then it is not hard to see that for every $t \in \R$
\begin{align*}
	t_v(t,t,t;s) = \left\{ \begin{array}{ll}
	 t + 2 - s & \text{ if } s < 0 \\
	 t + 1 & \text{ if } s \geq 0.
	\end{array}\right.
\end{align*}
We observe that 
\begin{equation*}
    \|t_v(t,t,t;s)-t_v(t,t,t;-s)\|_{L^\infty} \geq 1
\end{equation*}
for all $s\in \R$.
Letting $s \to 0$ demonstrates that $t_v(t_{u_1},t_{u_2},t_{u_3};s)$ does not depend continuously on its parameters.
Hence, we conclude that, contrary to feedforward neural networks, gradient-based training is generally not well-defined for SNNs with negative weights.
\end{exmp}


\end{document}